\documentclass[preprint,12pt,authoryear]{elsarticle}

\usepackage{amssymb}

\usepackage{enumitem}
\usepackage{amsmath}
\usepackage{amsthm}
\usepackage{makecell}
\usepackage{mathtools}
\usepackage{adjustbox}
\usepackage{booktabs}
\usepackage{graphicx}
\usepackage{float}
\usepackage{xcolor}
\usepackage{etoolbox}
\usepackage{setspace}
\usepackage{comment}
\usepackage{subcaption}
\usepackage[hidelinks]{hyperref}
\usepackage[ruled,vlined]{algorithm2e}

\DeclareMathOperator*{\argmax}{arg\,max}

\newtheorem{Lemma}{Lemma}
\newtheorem{Theorem}{Theorem}
\newtheorem*{Sketch}{Sketch of Proof}
\newtheorem{Corollary}{Corollary}

\makeatletter
\@addtoreset{Theorem}{section}
\makeatother

\journal{}

\definecolor{red(ncs)}{rgb}{0.0, 0.0, 0.0}
\newcommand{\rd}[1]{\textcolor{red(ncs)}{#1}}

\definecolor{red(ncss)}{rgb}{0.0, 0.0, 0.0}
\newcommand{\rdd}[1]{\textcolor{red(ncss)}{#1}}

\begin{document}

\begin{frontmatter}



\title{\rd{Addressing Maximization Bias in Reinforcement Learning with Two-Sample Testing}}


\author[1]{Martin Waltz\corref{CorrespondingAuthor}}
\ead{martin.waltz@tu-dresden.de}

\author[1,2]{Ostap Okhrin}

\affiliation[1]{organization={Technische Universität Dresden, Chair of Econometrics and Statistics, esp. in the Transport Sector},
            city={Dresden},
            postcode={01062}, 
            country={Germany}}

\affiliation[2]{organization={\rd{Center for Scalable Data Analytics and Artificial Intelligence (ScaDS.AI)}}, city={\rd{Dresden/Leipzig}}, country={\rd{Germany}}}

\cortext[CorrespondingAuthor]{Corresponding author}
\begin{abstract}
Value-based reinforcement-learning algorithms have shown strong results in games, robotics, and other real-world applications. Overestimation bias is a known threat to those algorithms and can \rdd{sometimes} lead to dramatic performance decreases or even complete algorithmic failure. We frame the bias problem statistically and consider it an instance of estimating the maximum expected value (MEV) of a set of random variables. We propose the $T$-Estimator (TE) based on two-sample testing for the mean, that flexibly interpolates between over- and underestimation by adjusting the significance level of the underlying hypothesis tests. \rdd{We also introduce a} generalization, termed $K$-Estimator (KE), \rdd{that} obeys the same bias and variance bounds as the TE \rdd{and relies} on a nearly arbitrary kernel function. We introduce modifications of $Q$-Learning and the Bootstrapped Deep $Q$-Network (BDQN) using the TE and the KE\rd{, and prove convergence in the tabular setting.} Furthermore, we propose an adaptive variant of the TE-based BDQN that dynamically adjusts the significance level to minimize the absolute estimation bias. All proposed estimators and algorithms are thoroughly tested and validated on diverse tasks and environments, illustrating the bias control and performance potential of the TE and KE.
\end{abstract}



\begin{keyword}
maximum expected value \sep two-sample testing \sep reinforcement learning \sep $Q$-learning \sep estimation bias



\end{keyword}

\end{frontmatter}



\section{Introduction}
Estimating the maximum expected value (MEV) of a set of random variables is a long-standing statistical problem, including early contributions of \cite{blumenthal1968estimation}, \cite{dudewicz1971maximum}, and \cite{ishwaei1985non}. These works show that for various underlying distributions\rdd{, for example, Gaussian and Binomial distributions,} an unbiased estimator does not exist. The problem has recently attained increased attention since it also arises in reinforcement learning (RL), \rd{where an agent interacts with an environment while optimizing for a policy - a mapping from states to actions - that maximizes a numerical reward signal \citep{sutton2018reinforcement}.~\rdd{Many of the frequently} used RL} algorithms define a policy-dependent action-value, also called $Q$-value, for each state-action pair. This value represents the expected sum of discounted rewards when executing the given action in the given state and following a specific policy afterward. In particular, the update rule of the $Q$-Learning algorithm \rdd{\citep{watkins1992q}} is based on adjusting the $Q$-estimate for a given state-action pair towards the observed reward and the maximum of the estimated $Q$-values of the next state. However, this use of the Maximum Estimator (ME) of the MEV leads to overestimations of action-values, which are transmitted throughout the update routine \citep{hasselt2010double}. These can damage the learning performance or even lead to failure of the algorithm \citep{thrun1993issues}, especially when function approximation is used \citep{van2016deep}.

\cite{hasselt2010double} proposed the Double Estimator (DE), which splits the data into independent sets, thereby separating the selection and evaluation of a maximizing value. The corresponding Double $Q$-Learning is a popular choice among practitioners \rdd{\citep{yuan2019double, he2021variational}}. Although the DE introduces underestimation bias, Double $Q$-Learning offers improved robustness and strong performances, especially in highly stochastic environments. Another crucial contribution is \rdd{the work by} \cite{d2016estimating}, in which the Weighted Estimator (WE) alongside Weighted $Q$-Learning is proposed. From a bias perspective, the estimator builds a compromise between the overestimating ME and the underestimating DE. However, the WE does not offer additional flexibility in selecting the level of bias and is computationally demanding since it requires numerical integration or Monte Carlo approximation. \rd{\cite{lan2020maxmin} presented MaxMin $Q$-Learning, which builds on learning multiple approximators for the same action-value. Whether the algorithm over- or underestimates $Q$-values depends on the discrete choice of the number of approximators. We refer to the corresponding estimator of the MEV as the MaxMin Estimator (MME).} \rdd{Notably, the DE, WE, and the MME also} led to modifications of the Deep $Q$-Networks (DQN, \citealt{mnih2015human}). \rdd{The DQN} expanded $Q$-Learning to the deep neural network (DNN) setting and paved the path for the striking success of RL in recent years \citep{silver2017mastering, vinyals2019grandmaster}.

\cite{d2016estimating}, \cite{lan2020maxmin}, among others, have shown that both over- and underestimation of $Q$-values might not always be harmful, depending on, e.g., the stochasticity of the environment, the difference of the action-values, the size of the action space, or the time horizon. We argue that a competitive estimator of the MEV should thus be able to interpolate between over- and underestimation via an interpretable hyperparameter, enabling it to deal with a diverse set of environments. Furthermore, the estimator should obey a variance bound \rdd{similar to competitors like the ME and DE to avoid trading a smaller absolute bias for a significantly increased variance bound. In addition,} for practical application, \rdd{the estimator} should be fast and stable to compute. Fulfilling these criteria, we propose an estimator based on two-sample testing for the mean, named $T$-Estimator (TE). The idea is to get a statistically significant statement of whether one mean is truly larger than others. Consequently, the hyperparameter is the level of significance $\alpha$. The ME is shown to be a special case of the TE with $\alpha = 0.5$. Building on the two-sample test statistic, we further consider a generalization termed $K$-Estimator (KE), which is characterized by a suitable kernel function and can smooth the discontinuities around testing decisions of the TE. We theoretically and empirically analyze the TE and KE regarding their biases and variances, for which general bounds are derived. Using \rdd{these} newly defined estimators, we propose RL algorithms for \rdd{both} the \rd{tabular} case and with DNNs as function approximators. Since the two-sample testing procedure incorporates variance estimates of the involved variables, we employ an online variance update routine \citep{d2019exploiting} in the \rd{tabular} scenario, and the framework of the Boostrapped DQN (BDQN, \citealt{osband2016deep}) in the DNN setting. \rd{Furthermore, we prove the convergence of the algorithm for the tabular case.}

The empirical evidence that over- and underestimation of action-values is not necessarily detrimental to learning performance might be explained by the connection of $Q$-estimates to the exploration procedure of algorithms \citep{fox2016taming}. However, \cite{fox2019toward} and \cite{liang2021temporal} argue that these topics should be addressed separately by focusing firstly on unbiased value-estimation and secondly on improved exploration schemes. We acknowledge this perspective by additionally proposing an adaptive tuning mechanism for the significance level $\alpha$ of the TE in the DNN setting with the objective of minimizing the absolute estimation bias. The approach complements recent proposals of \cite{dorka2021adaptively} and \cite{wang2021adaptive}. The dynamic adjustment of $\alpha$ is realized by running partial greedy episodes and comparing $n$-step returns \citep{sutton2018reinforcement} with the action-value estimates for the visited state-action pairs. Furthermore, through learning $\alpha$, we avoid the \rd{computationally} demanding tuning process of this environment-specific hyperparameter. Finally, we demonstrate the performance potential of all \rdd{the} newly proposed estimators and algorithms by extensively testing them in various tasks and environments, with and without function approximation.

The paper is organized as follows: Section \ref{sec:Estimating_max_mu} formalizes the problem of estimating the MEV. Section \ref{sec:two_sample_testing_MEV} details the proposed estimators\rd{, and Section \ref{sec:deps} analyzes them with and without fulfillment of the underlying independence assumptions.} Section \ref{sec:TE_KE_based_RL} introduces the RL setup and presents the new temporal-difference algorithms, while Section \ref{subsec:adaptive_bias} details the measurement of estimation bias and introduces the adaptive update mechanism of $\alpha$. The experiments are shown and thoroughly discussed in Section \ref{sec:Experiments}, with the code being available at: \url{https://github.com/MarWaltz/TUD_RL}. Section \ref{sec:related_work} provides further literature on the state-of-the-art\rd{,} and Section \ref{sec:conclusion} concludes \rd{this article}.

\sloppy
\section{Estimating the Maximum Expected Value}\label{sec:Estimating_max_mu}
\subsection{Problem Definition}
Let us consider $M \geq 2$ independent \rdd{real-valued} random variables $X_1, \ldots, X_M$ with finite expectations $\mu_1 = \operatorname{E}(X_1), \ldots, \mu_M = \operatorname{E}(X_M)$ and \rdd{finite} variances $\sigma_{1}^2 = \operatorname{Var}(X_1), \ldots, \sigma_{M}^2 = \operatorname{Var}(X_M)$. The corresponding probability density functions (pdfs) and cumulative distribution functions (cdfs) are denoted $f_{X_1}, \ldots, f_{X_M}$ and $F_{X_1}, \ldots, F_{X_M}$, respectively. The quantity of interest is the \emph{maximum expected value}: $\mu_{*} = \max_i \mu_i$. Estimation is performed based on samples $S = \rdd{(}S_1, \ldots, S_M\rdd{)}$\rd{, where $S_i$ is \rdd{a set containing samples from $X_i$} for $i=1,\ldots,M$,} without knowing moments or imposing distributional assumptions. \rdd{For simplicity, we refer to the set of samples $S_i$ simply as sample $S_i$. In addition, we assume to have at least two different elements inside each sample $S_i$ to ensure a non-zero sample variance. Moreover, the} realizations in a sample $S_i$ are assumed to be \rd{independent and identically distributed}. \rd{Consequently, we assume that there are no \emph{in-sample} dependencies inside a sample $S_i$ and no \emph{cross-sample dependencies} between samples $S_i$ and $S_j$, for $i,j = 1, \ldots, M$ and $i \neq j$.} The unbiased sample mean of $S_i$ is denoted $\hat{\mu}_i(S_i)$, while an estimator of the MEV is referred to as $\hat{\mu}_*(S)$. Throughout the paper, we \rdd{use abbreviated notations for conciseness, such as} $\hat{\mu}_i = \hat{\mu}_i(S_i)$, $\hat{\mu}_{*} = \hat{\mu}_{*}(S)$, and similar \rdd{expressions}. Primary evaluation criteria of an estimator are its bias $\operatorname{Bias}(\hat{\mu}_*) = \operatorname{E}(\hat{\mu}_*) - \mu_*$, and variance $\operatorname{Var}(\hat{\mu}_*) = \operatorname{E}\left\{ [\hat{\mu}_* - \operatorname{E}(\hat{\mu}_*)]^2 \right\}$. These can be aggregated to the mean squared error $\operatorname{MSE}(\hat{\mu}_*) = \rdd{\operatorname{E}\left[ (\hat{\mu}_* -\mu_*)^2 \right] = } \operatorname{Bias}(\hat{\mu}_*)^2 + \operatorname{Var}(\hat{\mu}_*)$.

\subsection{Maximum Estimator}\label{subsec:ME}
The ME $\hat{\mu}^{\rdd{\textrm{ME}}}_{*}$ is the classic approach and takes the maximum of unbiased mean estimates:
\begin{equation*}
    \hat{\mu}^{\rdd{\textrm{ME}}}_{*} = \max_i \hat{\mu}_i.
\end{equation*}
Denoting the pdf of $\hat{\mu}_i$ as $\hat{f}_i$ and the corresponding cdf as $\hat{F}_i$, it holds:
\begin{equation}\label{eq:ME_exp}
    \operatorname{E}\left(\hat{\mu}^{\rdd{\textrm{ME}}}_{*}\right) = \sum_{i=1}^{M} \int_{-\infty}^{\infty} x \hat{f}_i(x) \prod\limits_{\substack{j=1 \\ j\neq i}}^M \hat{F}_j(x) dx.
\end{equation}

The ME is positively biased: $\operatorname{E}\left(\hat{\mu}^{\rdd{\textrm{ME}}}_{*}\right) \geq \mu_{*}$\rdd{; see \cite{hasselt2010double}}. \rdd{This bias occurs because $x$ inside the integral positively correlates with the monotonically increasing product $\prod\limits_{\substack{j=1 \\ j\neq i}}^M \hat{F}_j(x)$.} \rdd{F}ollowing \cite{aven1985upper}, a general upper bound for the bias can be given:
\begin{equation}\label{eq:ME_bias_bounds}
    0 \leq \operatorname{Bias}\left(\hat{\mu}^{\rdd{\textrm{ME}}}_{*}\right) \leq \sqrt{\frac{M-1}{M} \sum_{i=1}^{M} \operatorname{Var}\left(\hat{\mu}_i\right)}.
\end{equation}
The bias is particularly large when $\mu_1 \approx \ldots \approx \mu_M$. Furthermore, it can be shown that the variance of $\hat{\mu}^{\rdd{\textrm{ME}}}_{*}$ is bounded from above: $\operatorname{Var}\left(\hat{\mu}^{\rdd{\textrm{ME}}}_{*}\right) \leq \sum_{i=1}^{M} \frac{\sigma_{i}^{2}}{\vert S_i\vert }$, where $|S_i|$ is the size of $S_i$\rdd{;} see \cite{van2013estimating}.

\subsection{Double Estimator}\label{subsec:DE}
\cite{hasselt2010double} introduced the DE, which is thoroughly analyzed in \cite{van2013estimating}. The key idea is to separate the selection of the maximizing random variable and the evaluation of its sample mean, which is performed simultaneously in the ME. The DE splits $S$ randomly into disjoint subsets $S^A = \rdd{(}S^{A}_1, \ldots, S^{A}_M\rdd{)}$ and $S^B = \rdd{(}S^{B}_1, \ldots, S^{B}_M\rdd{)}$, guaranteeing that means based on the subsets are still unbiased. Afterwards, one selects an index which maximizes the sample mean in $S^{A}$: \rd{$a^* \in \{i \mid \hat{\mu}_{i}(S_{i}^{A}) = \max_j \hat{\mu}_{j}(S_{j}^{A})\}$.} The DE is defined by evaluating $a^*$ on $S^{B}$: \rd{$\hat{\mu}^{\rdd{\textrm{DE}}}_{*}(S) = \hat{\mu}_{a^*}(S^{B}_{a^*})$. Similarly}, one can perform the same procedure with $S^{A}$ and $S^{B}$ switched to get a second DE estimate. Averaging both DE estimates yields the 2-fold Cross-Validation estimator (CVE) $\hat{\mu}^{\rdd{\textrm{CVE}}}_{*}$, which has a reduced variance in comparison to a single DE estimate. The expectations of the DE and the CVE are equal since both DE estimates (for $S^A$ and $S^B$) have identical expectations:
\begin{align}
    \operatorname{E}\left(\hat{\mu}^{\rdd{\textrm{CVE}}}_{*}\right) = \operatorname{E}\left(\hat{\mu}^{\rdd{\textrm{DE}}}_{*}\right) &= \sum_{i=1}^{M} \operatorname{E}\left[\rd{\hat{\mu}_{i}(S_{i}^B)}\right] \rdd{\mathbb{P}}(i=a^*) \nonumber\\
    &= \sum_{i=1}^{M} \operatorname{E}\left[\rd{\hat{\mu}_{i}(S_{i}^B)}\right] \int_{-\infty}^{\infty} \hat{f}_i^{A}(x) \prod\limits_{\substack{j=1 \\ j\neq i}}^M \hat{F}_j^{A}(x) dx, \label{eq:DE_exp}
\end{align}
where $\hat{f}_i^{A}$ and $\hat{F}_i^{A}$ are the cdf and pdf of \rd{$\hat{\mu}_i(S_{i}^A)$}, respectively. \cite{hasselt2010double} showed that the DE is prone to underestimation: $\operatorname{E}\left(\hat{\mu}^{\rdd{\textrm{DE}}}_{*}\right) \leq \mu_{*}$, because it might attribute non-zero selection probability to non-maximum variables. Furthermore, \cite{van2013estimating} conjectures the following lower bound for the bias:
\begin{equation*}
    -\frac{1}{2} \left(\sqrt{\sum_{i=1}^{M} \frac{\sigma_{i}^{2}}{|S_{i}^{A}|}} + \sqrt{\sum_{i=1}^{M} \frac{\sigma_{i}^{2}}{|S_{i}^{B}|}} \right) < \operatorname{Bias}(\hat{\mu}^{\rdd{\textrm{DE}}}_{*}) \leq 0,
\end{equation*}
while the variance of the CVE is shown to be bounded as the ME: $\operatorname{Var}\left(\hat{\mu}^{\rdd{\textrm{CVE}}}_{*}\right) \leq \sum_{i=1}^{M} \frac{\sigma_{i}^{2}}{|S_i|}$. Worth mentioning is that the variance of the CVE is not necessarily half the variance of the DE, as there is non-zero covariance between the two DE estimates, see the example in \ref{appendix:analytic_forms}. Throughout the experiments, we follow \cite{d2021gaussian} and use the CVE instead of the DE whenever possible.

\subsection{Weighted Estimator}\label{subsec:WE}
\cite{d2016estimating} introduced the Weighted Estimator (WE) for the MEV, which is a weighted mean of all sample averages. Each weight corresponds to the probability of $\hat{\mu}_i$ being larger than all other means:
\begin{equation*}
    \hat{\mu}^{\rdd{\textrm{WE}}}_{*} = \sum_{i=1}^{M} w_i \hat{\mu}_i = \sum_{i=1}^{M} \rdd{\mathbb{P}} \left(\hat{\mu}_i = \max_j \hat{\mu}_j\right) \hat{\mu}_i.
\end{equation*}
Since the probabilities depend on the unknown mean distributions $\hat{f_i}$, the authors propose a Gaussian approximation based on the central limit theorem:
\begin{equation}\label{eq:WE}
    \hat{\mu}^{\rdd{\textrm{WE}}}_{*} = \sum_{i=1}^{M}\hat{\mu}_i \int_{-\infty}^{\infty}\Tilde{f}_i(x) \prod\limits_{\substack{j=1 \\ j\neq i}}^M \Tilde{F}_j(x) dx,
\end{equation}
where $\Tilde{f}_i$ \rdd{and $\Tilde{F}_i$ are the Gaussian pdf and cdf, respectively,} with mean $\hat{\mu}_i$ and variance $\frac{\hat{\sigma}_{i}^2}{|S_i|}$. The unbiased estimate of ${\sigma}_{i}^2$ is denoted $\hat{\sigma}_{i}^2$ and $|S_i|$ refers to the sample size. Crucially, the bias of the WE is bounded by the ME and DE:
\begin{equation*}
    \operatorname{Bias}(\hat{\mu}^{\rdd{\textrm{DE}}}_*) \leq \operatorname{Bias}(\hat{\mu}^{\rdd{\textrm{WE}}}_*) \leq \operatorname{Bias}(\hat{\mu}^{\rdd{\textrm{ME}}}_*),
\end{equation*}
while it exhibits the same variance bound: $\operatorname{Var}\left(\hat{\mu}^{\rdd{\textrm{WE}}}_{*}\right) \leq \sum_{i=1}^{M} \frac{\sigma_{i}^{2}}{|S_i|}$\rdd{; see \cite{d2016estimating}}. Thus, the bias of the WE might be positive or negative, depending on the distribution of the random variables. A drawback of this estimator lies in increased computation time since calculating the integrals in (\ref{eq:WE}) is a demanding process. Tackling this issue, \cite{d2021gaussian} propose to use Monte Carlo approximations instead, and we follow this approach when computing the WE in the experiments. The Monte Carlo sample sizes in those cases are set to 100.

\subsection{\rd{MaxMin Estimator}}
\rd{\cite{lan2020maxmin} proposed MaxMin $Q$-Learning with the corresponding MaxMin Estimator (MME) for the MEV. Similar to the DE, the MME splits each of the $M$ samples in $S = \rdd{(}S_1, \ldots, S_M\rdd{)}$ into $N$ disjoint, equally-sized subsamples, which we denote as $S_{i}^{j}$ for $i=1,\ldots, M$ and $j=1,\ldots,N$. The estimator is then constructed as follows:
\begin{equation}\label{eq:MME}
    \hat{\mu}^{\rdd{\textrm{MME}}}_{*} = \max_i \min_j \hat{\mu}(S_{i}^{j}),
\end{equation}
where $\hat{\mu}\left(S_{i}^{j}\right)$ is the sample mean of $S_{i}^{j}$. The underlying rationale of (\ref{eq:MME}) is that the underestimation introduced via the minimum operator mitigates the overestimation due to the max operator. \cite{lan2020maxmin} provide analytical results for the expectation and variance of the MME in the particular case that the sample means are uniformly distributed. Following \cite{d2021gaussian}, a general expression for the expectation of the MME is:
$$\operatorname{E}\left(\hat{\mu}^{\rdd{\textrm{MME}}}_{*}\right) = \sum_{i=1}^{M} \sum_{j=1}^{N} \rdd{\mathbb{P}} \left[i = \argmax_i \min_j \hat{\mu}(S_{i}^{j})\right] \mu_i.$$
The MME can control the estimation bias via the number of subsamples $N$. Consequently, this approach necessitates learning $N$ distinct approximators for the same quantity, which in turn limits the tuning capabilities of the method due to the discrete nature of $N$. Moreover, as emphasized by \cite{lan2020maxmin}, a critical subtlety is that the splitting procedure drastically reduces the available sample size for each of the $M \cdot N$ estimators.}

\section{Two-Sample Testing-based Estimators}\label{sec:two_sample_testing_MEV}
\subsection{T-Estimator}
To create a flexible estimator which:
\begin{enumerate}[label=(\alph*)]
    \item is able to interpolate between over- and underestimation,
    \item obeys a variance bound similar to the ME,
    \item has a \rd{continuously tunable and} interpretable hyperparameter,
    \item is fast and easy to compute,
\end{enumerate}
we propose a procedure based on one-sided two-sample testing for the mean. Generally, for two random variables $X_1, X_2$, we consider the hypothesis $H_0: \mu_1 \geq \mu_2$. The test statistic is constructed as follows \citep{MSAwackerly}:
\begin{equation}\label{eq:TST}
    T = \frac{\hat{\mu}_1 - \hat{\mu}_{2}}{\sqrt{ \frac{\hat{\sigma}^2_1}{|S_1|} + \frac{\hat{\sigma}^2_{2}}{|S_{2}|}}},
\end{equation}
\rd{where $\hat{\sigma}^2_i = \frac{1}{\vert S_i \vert -1}\sum_{j=1}^{\vert S_i \vert} \left(S_{i,j} - \hat{\mu}_i \right)^2$ is the unbiased estimator of the variance $\sigma^2_i$ for $i=1,2$. We denote with $S_{i,j}$ the $j$-th observation of sample $S_i$, while $|S_i|$ is the size of $S_i$.} If the realization of $T$ is smaller than the \rdd{respective} $\alpha$-quantile, $z_{\alpha} = \Phi^{-1}(\alpha)$, \rdd{where $\Phi$ is the standard Gaussian cdf,} hypothesis $H_0$ is rejected. The use of the normal distribution as the asymptotic distribution of the test statistics for $H_0$ can be justified via the \rd{Lindeberg-Lévy} central limit theorem, since it holds $\sqrt{|S_i|} \frac{\hat{\mu}_i - \mu_i}{\sigma_i} \xrightarrow[|S_i|\rightarrow \infty]{d} \mathcal{N}(0, 1)$ for $i = 1,2$, and using Slutsky's theorem, as $\hat{\sigma}_i$ converges almost surely to $\sigma_i$ \rdd{\citep{serfling1980approximation}}.

Based on this test, the following procedure for estimating the MEV is proposed: First, we consider the complete set of indices $\mathcal{L} = \{1, \ldots, M\}$ and select an index that corresponds to a variable with the value of the ME: $i^* \in \{i \mid \hat{\mu}_{i} = \max_j \hat{\mu}_{j}\}$. Second, we test for all $i \in \mathcal{L}$ the $H_0$: $\mu_i \geq \mu_{i^*}$. If $H_0$ is rejected for some $i'$, we assert $\mu_{i'} < \mu_{i^*}$ and remove variable index $i'$ from the index set: $\mathcal{L} \leftarrow \mathcal{L} \char`\\ \{i'\}$. Third, we average the remaining $\{\hat{\mu}_i \mid i \in \mathcal{L}\}$. Compactly written:
\begin{equation}\label{eq:T-Estimator}
    \rdd{\hat{\mu}^{\rdd{\textrm{TE}}}_*(\alpha) = \left[\sum_{i=1}^{M}\mathcal{I}\left(T_i\geq z_{\alpha}\right)\right]^{-1}
    \sum_{i=1}^{M}\mathcal{I}\left(T_i\geq z_{\alpha}\right)\hat{\mu}_i, \hspace{0.2cm} \text{where} \hspace{0.2cm} T_i = \frac{\hat{\mu}_i - \hat{\mu}^{\rdd{\textrm{ME}}}_{*}}{\sqrt{ \frac{\hat{\sigma}^2_i}{|S_i|} + \frac{\hat{\sigma}^2_{i^*}}{|S_{i^*}|} }},}
\end{equation}
\rdd{and} $\mathcal{I}(\cdot)$ is the indicator function. We refer to (\ref{eq:T-Estimator}) as \emph{T-Estimator} (TE). In simple words: TE averages the means of all variables, which are statistically not smaller than the one of the ME. Consequently, the selection is a binary decision of rejection or non-rejection of the underlying hypothesis. A key aspect of the TE is the consideration of the values of the sample means together with their uncertainties, expressed by variances. Asymptotically, \rd{when} $|S_i| \rightarrow \infty$ for $i = 1,\ldots,M$, the TE follows a normal distribution since it is an average of asymptotically \rd{normally} distributed variables.

The hyperparameter is the significance level $\alpha$, which is an interpretable quantity for practitioners and researchers, and is naturally restricted to $\alpha \in (0, 0.5]$. One can directly determine the extreme case on the upper domain limit: $\hat{\mu}^{\rdd{\textrm{TE}}}_*(\alpha = 0.5) = \hat{\mu}^{\rdd{\textrm{ME}}}_*$. \rd{This follows from the fact that $z_{0.5} = 0$, meaning that the test statistics, which is either zero or negative, has to be zero to fulfill the condition of the indicator function. Therefore, only the means reaching the maximum value will be averaged. Consequently, the ME is a special case of the TE, being prone to overestimation bias with the bounds given in Section \ref{subsec:ME}. Intuitively,} by reducing $\alpha$, we reduce the bias since we tend to non-reject $H_0$ for smaller sample means. If one would consider a significance level of zero, the TE would collapse into the Average Estimator (AE): $\hat{\mu}^{\rdd{\textrm{AVG}}}_* = M^{-1} \sum_{i=1}^{M} \hat{\mu}_i$. \cite{imagaw2017estimating} provide a similar definition in a multi-armed bandit context. The AE has low variance: $\operatorname{Var}(\hat{\mu}^{\rdd{\textrm{AVG}}}_*) = M^{-2} \sum_{i=1}^{M} \frac{\sigma_{i}^{2}}{|S_i|}$, but severe negative bias: $\operatorname{Bias}(\hat{\mu}^{\rdd{\textrm{AVG}}}_*) = - M^{-1} \sum_{i=1}^{M} (\max_j \mu_j - \mu_i)$. However, we do not include $\alpha = 0$ in our definition domain for the TE since 1) it is statistically not reasonable to consider such hypothesis tests and 2) the uncertainties of the sample means, quantified through variances, are not present anymore.

The following lemma contains the bias bounds of the TE.

\begin{Lemma}\label{lemma:Bias_bound_TE}
    For $\alpha \in (0, 0.5]$, it holds:
    $$\rdd{\min_i \mu_i - \max_i \mu_i - \sqrt{\frac{M-1}{M} \sum_{i=1}^{M} \mathrm{Var}\left(\hat{\mu}_i\right)}} \leq \mathrm{Bias}\left[\hat{\mu}^{\rdd{\textnormal{TE}}}_*(\alpha)\right] \leq \mathrm{Bias}(\hat{\mu}^{\rdd{\textnormal{ME}}}_*).$$
    \rd{Further, if $\mathrm{Var}\left(\hat{\mu}_i\right) = V$ for $i = 1, \ldots, M$ and some $V > 0$, then $\mathrm{Bias}\left[\hat{\mu}^{\rdd{\textnormal{TE}}}_*(\alpha)\right]$ is a monotonically increasing function of $\alpha$.}
\end{Lemma}
\begin{proof}\label{proof:Bias_bound_TE}
The upper bound is straightforward since the TE is a weighted average of sample means, while the ME is the extreme case of weighting the maximum sample mean with one. Regarding the lower bound, we use that per construction:
\begin{equation}\label{eq:lower_bound_TE}
    \hat{\mu}^{\rdd{\textrm{TE}}}_*(\alpha) \geq \rdd{\min_i \hat{\mu}_i}.
\end{equation}
To see this, we first note that the numerator of the test statistics \rdd{$T_i$} in (\ref{eq:T-Estimator}) is always zero for the ME, \rd{leading to a value of one for the corresponding indicator functions for} all $\alpha \in (0, 0.5]$. However, since the test statistics for index $i$ positively correlates with $\hat{\mu}_i$ and $\hat{\sigma}_i^2$, extreme variance scenarios are possible in which \rdd{indices with much smaller means, including the one for the minimum mean, are non-rejected. Considering $M$ might be very large, $\hat{\mu}^{\rdd{\textrm{TE}}}_*(\alpha)$ can thus be arbitrarily close to $\min_i \hat{\mu}_i$, yielding (\ref{eq:lower_bound_TE}).} Building expectations, we have:
\begin{align*}
    \operatorname{E}\left[\hat{\mu}^{\rdd{\textrm{TE}}}_*(\alpha)\right] &\geq \rdd{\operatorname{E}\left(\min_i \hat{\mu}_i\right) \geq \min_i \mu_i - \sqrt{\frac{M-1}{M} \sum_{i=1}^{M} \mathrm{Var}\left(\hat{\mu}_i\right)}},
\end{align*}
where the last \rdd{inequality} uses the bound for the minimum sample average of \cite{aven1985upper}. The bias follows immediately. \rd{The second part regarding the monotonicity in the case of equal variances of the mean estimates follows from the definition in (\ref{eq:T-Estimator}) since $z_{\alpha} \leq z_{\alpha'}$ for $\alpha \leq \alpha'$ and $\alpha, \alpha' \in (0, 0.5]$.}
\end{proof}

\rd{Lemma \ref{lemma:Bias_bound_TE} shows that the TE can achieve negative and positive bias depending on the level of significance $\alpha$. However, we can not make a general statement about the sign of the bias for a particular $\alpha$ since this depends on the number and the distribution of the underlying random variables.}

Furthermore, the TE is consistent for the MEV since, with increasing sample size, each sample mean approaches its population mean, the ME approaches the true MEV, and the variances of the means tend to zero. Consequently, all tests except for the true MEV variable will reject the $H_0$, and only the MEV will be left. 

\rd{The TE involves conducting multiple hypothesis tests, which can potentially lead to an increase in type I error. In such cases, common approaches for correction include the Bonferroni correction \citep{armstrong2014use}, the method proposed by \cite{holm1979simple}, or, for large values of $M$, the procedure developed by \cite{benjamini1995controlling}. However, we have chosen not to incorporate these techniques into the TE for two reasons. Firstly, we consider the significance level $\alpha$ as a freely adjustable hyperparameter. For instance, if we were to employ the Bonferroni correction, which divides the significance level by the number of tests, we could simply choose a different $\alpha$ to achieve a similar effect. Secondly, in the upcoming section, we will introduce a generalization of the TE that generally circumvents the necessity of binary testing decisions.}

Regarding variance, the TE shares the common overly pessimistic variance bound, while the proof relies on the TE being a weighted average of means and is similar to \cite{d2016estimating}:
\begin{Lemma}
    For $\alpha \in (0, 0.5]$, it holds: $\mathrm{Var}\left[\hat{\mu}^{\rdd{\textnormal{TE}}}(\alpha)\right] \leq \sum_{i=1}^{M} \frac{\sigma_{i}^{2}}{|S_i|}$.
\end{Lemma}

\subsection{K-Estimator}
\rd{By examining the structure of (\ref{eq:T-Estimator}), we can perceive the TE as a re-weighting scheme for sample means that is driven by statistical considerations. The pivotal aspect lies in the determination of weights through indicator functions. Notably, the derivatives of these indicator functions with respect to $\alpha$ are either zero or non-existent, resulting in either an exclusion or an inclusion of specific mean.} Avoiding this behaviour, we propose to apply the standard Gaussian cdf $\Phi$ directly to the test statistics and use the resulting values as a smoothed weighting:
\begin{equation}\label{eq:K-Estimator_Phi}
    \hat{\mu}^{\Phi}_* = \left[\sum_{i=1}^{M}\Phi\left(T_i\right)\right]^{-1} \sum_{i=1}^{M}\Phi\left(T_i\right)\hat{\mu}_i, \quad \text{where} \quad T_i = \frac{\hat{\mu}_i - \hat{\mu}^{\rdd{\textrm{ME}}}_{*}}{\sqrt{ \frac{\hat{\sigma}^2_i}{|S_i|} + \frac{\hat{\sigma}^2_{i^*}}{|S_{i^*}|} }}.
\end{equation}
In fact, it can be generalized even further by considering a  weighting kernel $\kappa(\cdot)$:
\begin{equation}\label{eq:K-Estimator}
    \hat{\mu}^{\rdd{\textrm{KE}}}_* = \left[\sum_{i=1}^{M}\kappa\left(T_i\right)\right]^{-1} \sum_{i=1}^{M}\kappa\left(T_i\right)\hat{\mu}_i.
\end{equation}
with \rd{$\kappa: (-\infty;0] \rightarrow [0; \infty)$}. We require that $\kappa(\cdot)$ is monotonically increasing to build a reasonable kernel since $T_i \leq 0$, $\forall i = 1, \ldots, M$, and that $\lim\limits_{T_i \rightarrow -\infty} \kappa(T_i) = 0$ for consistency. Similar kernel functions are considered in \cite{mammen1991estimating} for isotonic regressions. We refer to (\ref{eq:K-Estimator}) as the \emph{K-Estimator} (KE). Crucially, the hyperparameter of the KE is not a fixed scalar anymore, but the specification of $\kappa(\cdot)$. For example, \rd{the TE is a special case of the KE obtained by setting $\kappa(T_i;\alpha) = \mathcal{I}(T_i \geq z_{\alpha})$}. \rd{We emphasize that re-weighting schemes of sample means are widely known in the form of softmax operators in the RL literature \citep{asadi2017alternative, sutton2018reinforcement}. Although structurally similar, the KE uses the test statistics $T_i$ to determine the weight of mean $\hat{\mu}_i$. The weights in conventional softmax approaches are determined only by the mean itself, thereby neglecting the uncertainty of the mean estimates.}

Further options for $\kappa(\cdot)$ are listed in Table \ref{tbl:g_specifications}. Additionally, more flexible parametrized specifications are available. Consider, for example, the cdf of the beta distribution $\mathcal{B}_{\mathfrak{a}, \mathfrak{b}}$ with two shape parameters $\mathfrak{a}$, $\mathfrak{b}$. Although the latter is naturally defined on $[0,1]$, one could simply scale and shift it to, e.g., $[-5,0]$ to generate a more valid specification. This particular case is used in the example of Section \ref{sec:Analysis_Bias_Var_MSE}. Apart from that, we primarily apply the standard Gaussian cdf throughout the paper.

\renewcommand{\arraystretch}{2}
\begin{table}[h!]
\centering
\begin{tabular}{l|c} 
Kernel & $\kappa(T)$ for $T \leq 0$ \\  
 \hline
 cdf: Gaussian $\Phi_{\lambda}$ & $\int_{-\infty}^{T} \frac{1}{\sqrt{2\pi\lambda^2}} \exp{\left[-\frac{1}{2} \left(\frac{t}{\lambda}\right)^2\right]}dt$\\
 cdf: $t$-distribution $t_{\nu}$& $\int_{-\infty}^{T} \frac{\Gamma(\frac{\nu +1}{2})}{\sqrt{\nu \pi} \Gamma(\frac{\nu}{2})} \left(1 + \frac{t^2}{\nu}\right)^{-\frac{\nu+1}{2}}dt$\\
 Epanechnikov & $\frac{3}{4}(1-T^2) \mathcal{I}(|T|\leq 1)$ \\
 \rd{Softmax} & \rd{$\exp{(T)}$}\\
 Triangle & $(1-|T|) \mathcal{I}(|T|\leq 1)$
\end{tabular}
\caption{Exemplary kernel functions. \rd{The parameter} $\lambda$ is the standard deviation of the Gaussian cdf, \rd{$\nu$ is the degree} of freedom of the $t$-distribution, and $\Gamma(x) = \int_{0}^{\infty}t^{x-1}\exp(-t)dt$ denotes the gamma function. We abbreviate the standard Gaussian kernel $\Phi_{\lambda=1}$ with $\Phi$.}
\label{tbl:g_specifications}
\end{table}

Bias and variance of the KE depend on the chosen kernel specification, but the bounds of the TE are still valid as long as the kernel function fulfills the requirements stated above.

\begin{Corollary}\label{lemma:Bias_bound_KE}
    For the KE, it holds:
    $$\rdd{\min_i \mu_i - \max_i \mu_i - \sqrt{\frac{M-1}{M} \sum_{i=1}^{M} \mathrm{Var}\left(\hat{\mu}_i\right)}} \leq \mathrm{Bias}\left(\hat{\mu}^{\rdd{\textnormal{KE}}}_*\right) \leq \mathrm{Bias}(\hat{\mu}^{\rdd{\textnormal{ME}}}_*).$$
\end{Corollary}
\begin{proof}\label{proof:Bias_bound_KE}
The KE cannot exceed the ME, thus the upper bound holds. For the lower bound, we note the same relationship as for the TE:
\begin{equation*}
    \hat{\mu}^{\rdd{\textrm{KE}}}_* \geq \rdd{\min_i \hat{\mu}_i}.
\end{equation*}
\rdd{The} sample mean corresponding to the ME is per construction weighted with $\kappa(0)$ (before normalization). Simultaneously, for extreme variance scenarios, it is possible that the weight \rdd{of smaller means tends to $\kappa(0)$ as well, and, for sufficiently large $M$, the $\hat{\mu}^{\rdd{\textrm{KE}}}_*$ thus might be arbitrarily close to $\min_i \hat{\mu}_i$.} The remaining steps are identical to the proof of Lemma \ref{lemma:Bias_bound_TE}.
\end{proof}

\begin{Corollary}
    For the KE, it holds: $\mathrm{Var}\left(\hat{\mu}^{\rdd{\textnormal{KE}}}\right) \leq \sum_{i=1}^{M} \frac{\sigma_{i}^{2}}{|S_i|}$.
\end{Corollary}

The proof is again similar to \cite{d2016estimating}. \rd{Moreover, \rdd{in \ref{appendix:Exp_TEKE}}, we derive a general expression for the expectation of the KE for arbitrary $M$ in the case of known variances.}

Finally, we highlight three main differences between the KE (including the TE as a special case) and the WE of Section \ref{subsec:WE} since both are constructed as a weighted sum of sample means. First, the weights of the WE are probabilities, while the weights of the KE do not have a probabilistic interpretation. Second, while the KE allows for a multitude of specifications, the WE is not tunable and thus cannot be adjusted to a given scenario in a practical problem. Third, the computation of the WE requires integration or Monte Carlo approximation schemes, while the KE's computation is extremely fast.

\section{\rd{On the Role of Dependencies}}\label{sec:deps}
\rd{The previous exposition and the analytical properties of the estimators were predicated on the assumption of acquiring independent samples from independent random variables. However, in the realm of reinforcement learning, where we intend to apply these estimators, both of these independence assumptions often do not hold true. Specifically, the recursive nature of the algorithms and the bootstrapping of target values, which will be detailed in Section \ref{sec:TE_KE_based_RL}, can lead to the emergence of \emph{in-sample dependencies}. For instance, the observations $S_1$ of a random variable $X_1$ might display auto-correlations. Moreover, given that different states can share successor states, \emph{cross-sample dependencies} can manifest, implying correlations between the samples $S_1$ and $S_2$ of random variables $X_1$ and $X_2$, respectively.}

\rd{Deriving general analytical bounds for the bias of the estimators under the presence of these two distinct types of dependencies is a challenging task. Therefore, we will undertake an empirical investigation of bias and variance of the considered estimators by introducing in-sample and cross-sample dependencies into two Gaussian random variables. At first, however, we will analyze the example if all independence assumptions are fulfilled.}

\subsection{\rd{Independence: A Gaussian Example}}\label{sec:Analysis_Bias_Var_MSE}
\rd{We consider a similar setup to \cite{d2016estimating} with $M = 2$ Gaussian random variables $X_1 \sim \mathcal{N}(\mu_1, \sigma^2)$, $X_2 \sim \mathcal{N}(\mu_2, \sigma^2)$, where $\sigma^2 = 100$ is the common known variance, and we have $|S_1| = |S_2| = 100$ observations of each variable. To stress again: We assume to have independent samples from independent random variables.} We fix $\mu_2 = 0$ and compute bias, variance, and MSE for different $\mu_1 \in [0, 5]$. For completeness, we report the analytic forms for the expectation and variance of the estimators in this case in \ref{appendix:analytic_forms}. For the TE, we select significance levels $\alpha \in \{0.05, 0.10, 0.15\}$ and for the KE, we analyze the standard Gaussian kernel $\Phi$, the Epanechnikov kernel, and the shifted and scaled $\mathcal{B}_{\mathfrak{a}, \mathfrak{b}}$ cdf kernel with $\mathfrak{a} = 2$, $\mathfrak{b} = 0.5$ as described above. Results are displayed in Figures \ref{fig:MSE_TE} and \ref{fig:MSE_KE}.

\begin{figure}[h]
\centering
\includegraphics[width=\linewidth]{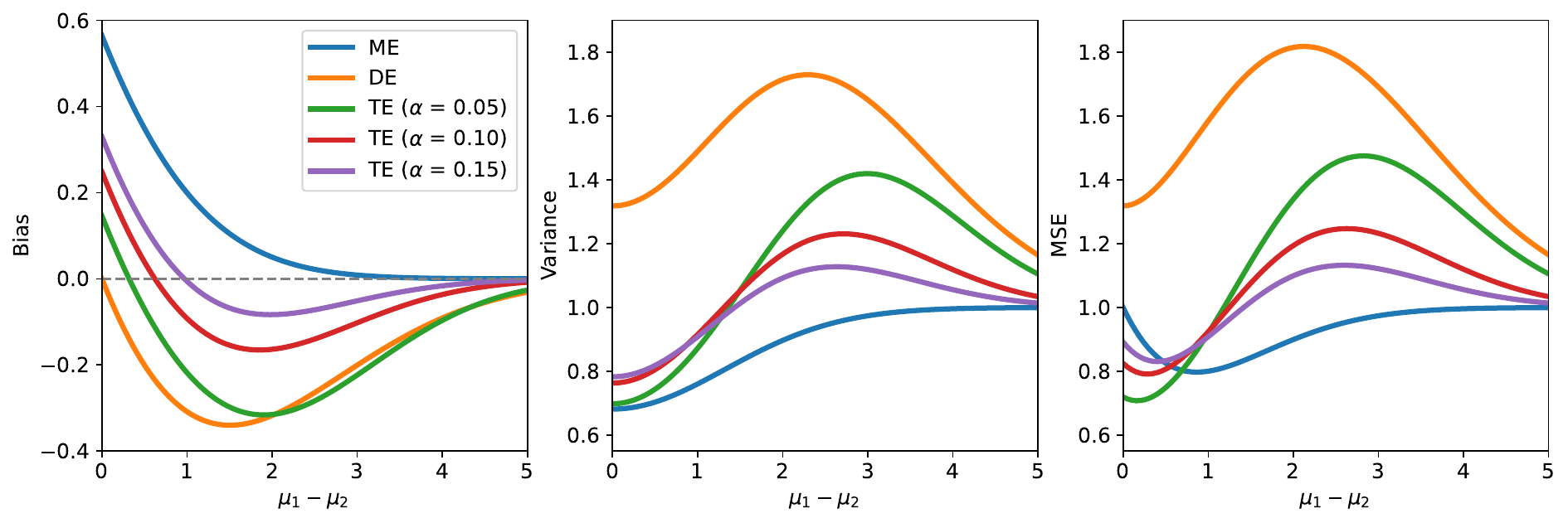}
\caption{Comparison of the ME, DE, and TE with level of significance in parentheses.}
\label{fig:MSE_TE}
\end{figure}
\begin{figure}[h]
\centering
\includegraphics[width=\linewidth]{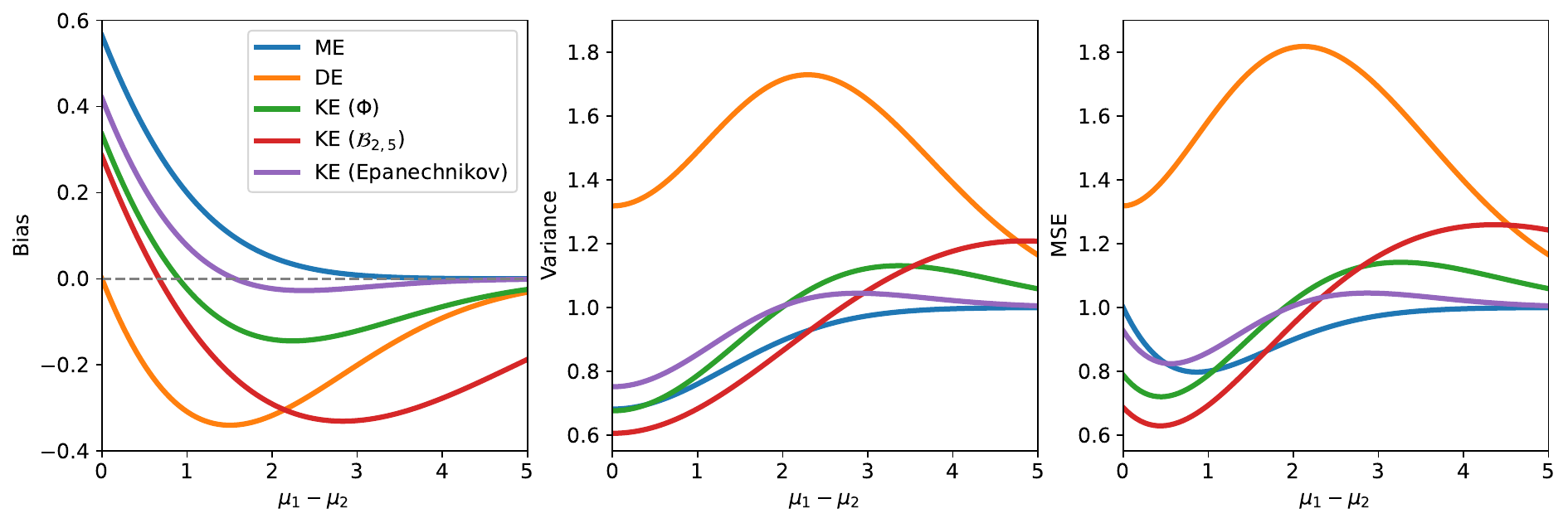}
\caption{Comparison of the ME, DE, and KE with kernel in parentheses.}
\label{fig:MSE_KE}
\end{figure}

Regarding the TE, we see how the bias decreases with a smaller significance level. In general, the estimator operates between the biases of ME and DE, although the bias of the DE is not necessarily a lower bound (see $\alpha = 0.05$). In the mean equality case, $\mu_1 = \mu_2$, the TE can avoid the significant overestimation of the ME while having only slightly increased variance. Consequently, the TE outperforms the conventional competitors for all considered significance levels under the MSE criterion for $\mu_1 = \mu_2$. However, if the difference of the true expectations of the random variables is large, all estimators become unbiased. In this scenario, the ME is the best choice due to its low variance. Regarding the KE in Figure \ref{fig:MSE_KE}, we see that the standard Gaussian and Epanechnikov kernels can achieve a desirable balance between under- and overestimation while maintaining a smaller variance than the considered TE. \rd{The chosen specification of the beta distribution appears sub-optimal for this particular problem. In \ref{appendix:opt_Gaussian}, we optimize for a better fitting parametrization to further illustrate the flexibility of the KE.}

Overall, we have seen through this investigation that both the TE and the KE can achieve flexible trade-offs between bias and variance in the estimation of the MEV. Considering typical levels of significance like $0.05$, $0.10$, or $0.15$ in the TE builds a robust estimator, further enhanced by accurately specifying a suited KE.

\subsection{\rd{In-sample dependencies}}\label{subsec:non_iid_MEV}
\rd{In this section, we introduce in-sample dependencies into our two Gaussian random variables. Specifically, we will examine the following auto-correlated process, motivated from the literature on time series econometrics \citep{tsay2005analysis}:
\vspace{-0.3cm}
\begin{align}\label{eq:AR_process}
    X_{i,t} &= (1-\rho)\mu_i + \rho X_{i, t-1} + \varepsilon_{i,t},\\
    \varepsilon_{i,t} &\sim \mathcal{N}\left[0, (1-\rho^2) \sigma^2\right],\nonumber
\end{align}
where $X_{i,0} = \mu_i + \varepsilon_{i,0}$. We set $\varepsilon_{i,0} \sim \mathcal{N}\left(0, \sigma^2\right)$, for $i=1,2$, and consider time steps $t = 1, \ldots, T$. Throughout the experiments, we set $T = 100$. This formulation ensures that the unconditional first two moments remain unchanged: $\operatorname{E}[X_{i,t}] = \mu_i$ and $\operatorname{Var}[X_{i,t}] = \sigma^2$. Similar to Section \ref{sec:Analysis_Bias_Var_MSE}, there are still no cross-sample dependencies between the two processes. In particular, we have $\operatorname{Cov}(X_{1,t}, X_{2,t}) = 0$ for all $t$. However, there is now auto-correlation within each process, $\operatorname{Cov}(X_{i,t-1}, X_{i,t}) \neq 0$ for $i=1,2$, which can be modulated by the parameter $\rho \in [0,1)$. Figure \ref{fig:Sampled_AR_processes} illustrates exemplary realizations of (\ref{eq:AR_process}) for $\rho \in \{0.0, 0.25, 0.5, 0.75\}$.}

\begin{figure}[h]
    \centering
    \includegraphics[width=\linewidth]{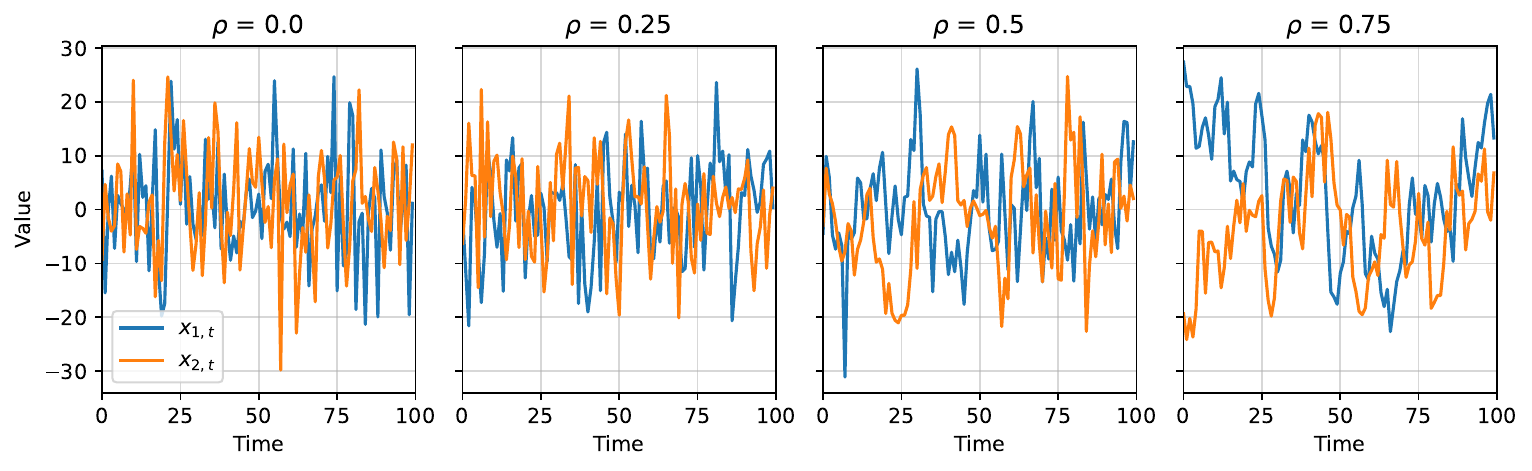}
    \caption{\rd{Realizations of the processes in (\ref{eq:AR_process}) for $\mu_1 = \mu_2 = 0$, $\sigma^2 = 100$, and varying $\rho$. While $\rho = 0.0$ represents independent random noise, the in-sample dependencies are visible for larger $\rho$ through the prolonged increasing or decreasing patterns of the time series.}}
    \label{fig:Sampled_AR_processes}
\end{figure}

\begin{figure}[h]
    \centering\includegraphics[width=\textwidth]{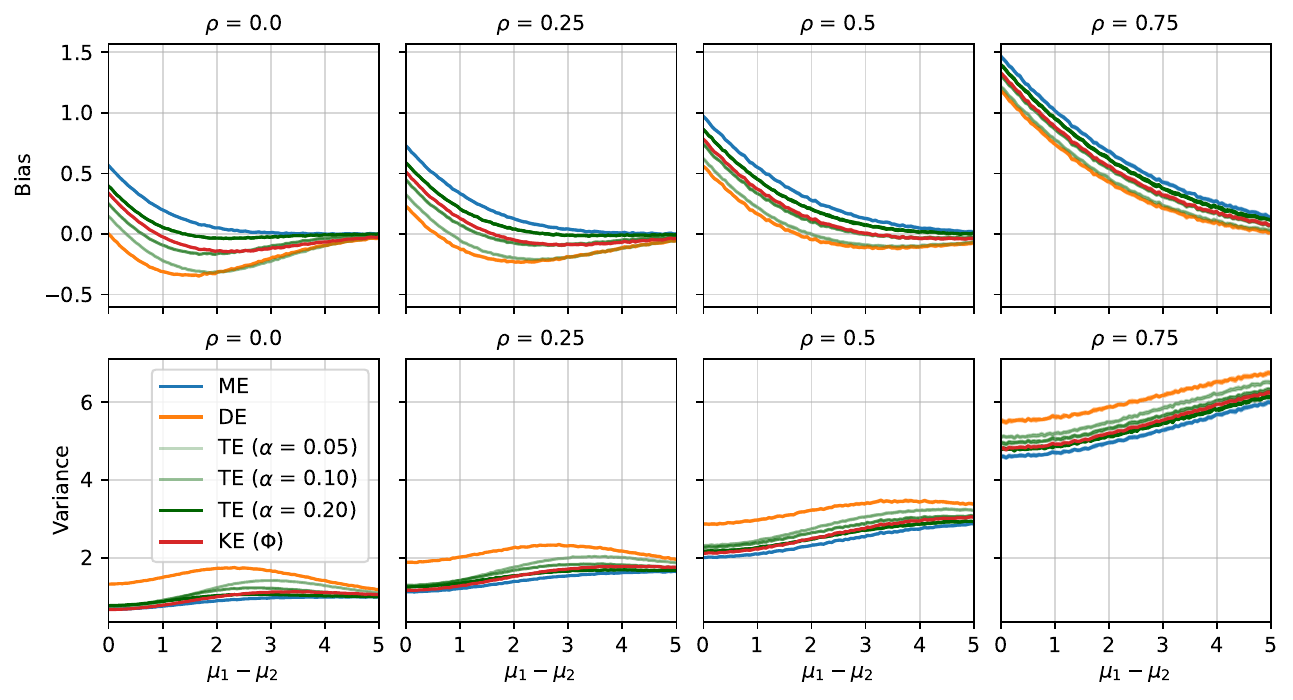}
    \caption{\rd{Bias and variance of the ME, DE, TE, and KE when in-sample dependencies are present. The latter are introduced via autocorrelations, which are modulated by the parameter $\rho \in [0,1)$. The larger $\rho$, the larger the in-sample dependence. Three different levels of significance ($\alpha$) are displayed in green for the TE, where the green gets darker with increasing $\alpha$. Point-wise 95\% confidence intervals are included based on repeating the experiment in 30 independent runs. However, the intervals are very narrow.}}
    \label{fig:stuy_autocorrelated}
\end{figure}

\rd{We compute the ME, DE, TE ($\alpha \in \{0.05, 0.10, 0.20\}$), and the KE (standard Gaussian kernel) after simulating from (\ref{eq:AR_process}) for different $\mu_1 \in [0,5]$, while $\mu_2 = 0$ and $\sigma^2 = 100$ are fixed. The \rdd{simulation and estimation} procedure is repeated 10\, 000 times for each $\mu_1$, and the results are used to estimate the bias and the variance of the respective estimator of the MEV. Figure \ref{fig:stuy_autocorrelated} displays the results.}

\rd{It is evident that the scenario with $\rho = 0.0$ aligns with the analytical outcomes depicted in Figures \ref{fig:MSE_TE} and \ref{fig:MSE_KE}, thus providing a validation of the derivations in \ref{appendix:analytic_forms} through simulation. However, all estimators exhibit increasing bias and variance when the in-sample dependencies increase. This observation can be explained by the fact that introducing autocorrelation has a similar effect as reducing the sample size, which can be seen as follows: If $\rho = 0$, the variance of the mean estimate of the process ($\ref{eq:AR_process}$) is $\operatorname{Var}\left(\frac{1}{T} \sum_{t=1}^{T}X_{i,t}\right) = \frac{1}{T^2} \sum_{t=1}^{T}\operatorname{Var}\left(X_{i,t}\right) = \frac{\sigma^2}{T}$, where the first step \rdd{holds} since the $X_{i,t}$ are independent. Now assume $\rho > 0$. This leads to: ${\operatorname{Var}\left(\frac{1}{T} \sum_{t=1}^{T}X_{i,t}\right) = \frac{1}{T^2} \left[\sum_{t=1}^{T}\operatorname{Var}\left(X_{i,t}\right) + \sum_{t=1}^{T} \sum_{\substack{\Tilde{t}=1\\ \Tilde{t} \neq t}}^{T} \operatorname{Cov}\left(X_{i,t}, X_{i, \Tilde{t}}\right) \right] > \frac{\sigma^2}{T}}$. Due to the positive covariance terms, the variance of the mean estimates is increased. Consequently, the estimates of the MEV are less accurate and have an increased variance. Intuitively, each new realization of the process (\ref{eq:AR_process}) contains less new information \rdd{than} the $\rho = 0$ case due to the dependence on the past.}

\rd{However, we observe that the relative ranking of the estimators in terms of both bias and variance remains unchanged. In general, the ME consistently exhibits the highest bias, while the bias of the TE decreases monotonically with smaller $\alpha$.}

\subsection{\rd{Cross-sample dependencies}}
\rd{Regarding cross-sample dependencies, we model the two random variables with a bivariate Gaussian distribution, instead of considering two independent univariate Gaussians. In particular, we set:
\begin{equation}\label{eq:biv_Gaussian}
    \begin{pmatrix}
        X_1 \\ X_2
    \end{pmatrix} \sim \mathcal{N}\left\{
    \begin{pmatrix} \mu_1 \\ \mu_2 \end{pmatrix}, \begin{pmatrix}
        \sigma^2 & \rho_G \sigma^2\\
        \rho_G \sigma^2 & \sigma^2
    \end{pmatrix}
    \right\},
\end{equation}
where $\rho_G \in [0,1]$ is the correlation parameter of the distribution. The unconditional first two moments are identical to Sections \ref{sec:Analysis_Bias_Var_MSE} and \ref{subsec:non_iid_MEV}, no in-sample dependencies are contained, and the simulation conditions are set to mimic the situation of the previous case. However, the cross-sample dependence can now be increased by increasing $\rho_G$. Figure \ref{fig:copula_sample} shows different realizations of (\ref{eq:biv_Gaussian}), while Figure \ref{fig:study_copula} displays the results.}

\begin{figure}[h]
    \centering
    \includegraphics[width=\linewidth]{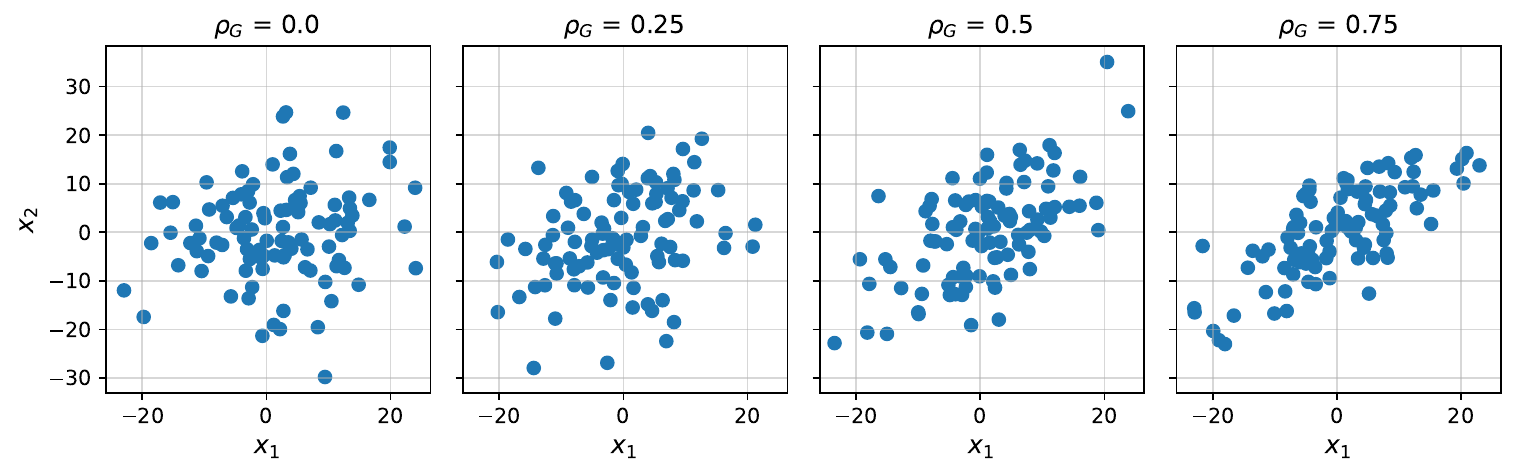}
     \caption{\rd{Realizations of the bivariate Gaussian distribution defined in (\ref{eq:biv_Gaussian}) for ${\mu_1 = \mu_2 = 0}$, $\sigma^2 = 100$, and varying $\rho$. Setting $\rho_G = 0$ means that the two samples are independent, while the cross-sample correlation gets stronger with larger $\rho_G$.}}
    \label{fig:copula_sample}
\end{figure}

\begin{figure}[h]
    \centering
    \includegraphics[width=\textwidth]{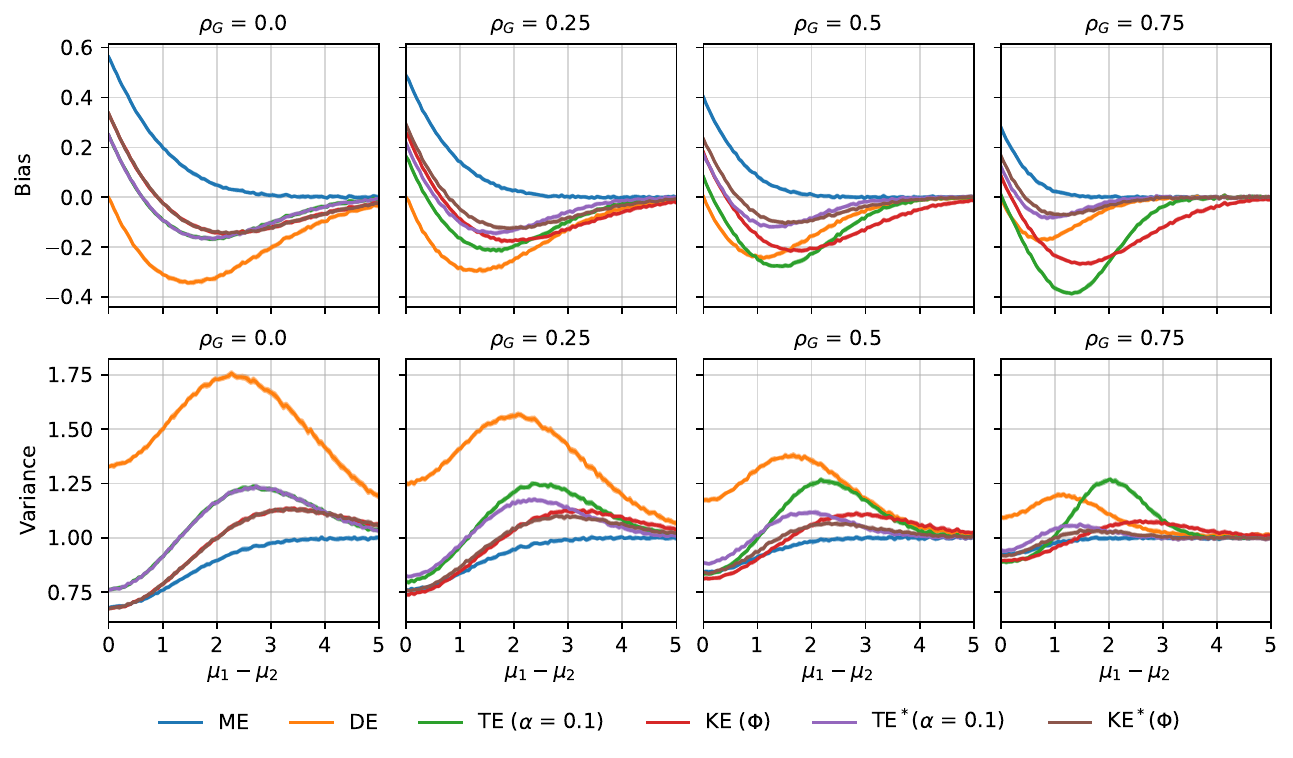}
    \caption{\rd{Bias and variance of the ME, DE, TE, and KE when cross-sample dependencies in the form of the correlation parameter $\rho_G \in [0,1]$ of a bivariate Gaussian are introduced. The larger $\rho_G$, the larger the cross-sample dependence. We include versions of the TE and KE which contain a covariance correction during the computation of the test statistics. These are denoted TE$^*$ and KE$^*$, respectively. \rdd{For the case $\rho_G = 0.0$, the curves of the TE$^*$ and KE$^*$ strongly overlap with those of the TE and KE, respectively.} Point-wise 95\% confidence intervals are included based on repeating the experiment in 30 independent runs. However, the intervals are very narrow.}}
    \label{fig:study_copula}
\end{figure}

\rd{Analyzing the ME and DE, we see that the increased correlation between the samples is generally beneficial in reducing the absolute value of the bias and the variance of the estimators. However, a different behavior is observed in the TE and KE, which are more sensitive to cross-sample dependencies. In particular, a positive correlation leads to an underestimation of the MEV and an increased variance, although the effects of cross-sample dependencies on bias and variance are generally not nearly as pronounced as the ones of in-sample dependencies.}

\rd{This behavior can be elucidated by examining the construction of the estimators. The underlying test statistics, as defined in (\ref{eq:TST}), assumes independence between the random variables, implying $\operatorname{Var}(\hat{\mu}_1 - \hat{\mu}_2) = \operatorname{Var}(\hat{\mu}_1) + \operatorname{Var}(\hat{\mu}_2)$. However, this assumption no longer holds true in scenarios with $\rho_G > 0$. More precisely, in these scenarios, (\ref{eq:TST}) overestimates $\operatorname{Var}(\hat{\mu}_1 - \hat{\mu}_2)$, leading to the consideration of smaller means and, consequently, the observed underestimation of the TE and KE.
To account for this circumstance, one could implement a covariance correction and replace the test statistics of TE and KE as given in (\ref{eq:TST}) with:
\begin{equation}\label{eq:TST_corrected}
    T^* = \frac{\hat{\mu}_1 - \hat{\mu}_{2}}{\sqrt{ \frac{\hat{\sigma}^2_1}{|S_1|} + \frac{\hat{\sigma}^2_{2}}{|S_{2}|} - 2\cdot \widehat{\operatorname{Cov}}(S_1, S_2)}},
\end{equation}
using the notation of Section \ref{sec:Estimating_max_mu} and with $\widehat{\operatorname{Cov}}(S_1, S_2)$ being the covariance estimate from samples $S_1$ and $S_2$. We include the resulting covariance-corrected TE and KE in Figure \ref{fig:study_copula}, denoted as TE$^*$ and KE$^*$. Due to the correction, the dependence sensitivity is alleviated, and the expected desirable behavior materializes.}

\subsection{\rd{Discussion}}
\rd{Both in-sample and cross-sample dependencies can impact the properties of the estimators of the MEV, although the effects of in-sample dependencies on bias and variance appear to be much stronger. Importantly, we observed that the \emph{relative ordering} of the estimators to each other remains the same for these crucial in-sample dependencies, implying that the choice of the estimator of the MEV is still critical in the presence of temporal dependencies. Similar to the works of \cite{hasselt2010double}, \cite{d2016estimating}, and \cite{zhu2021self}, we will transfer the estimators of the MEV to the domain of reinforcement learning in the next section.}

\section{Application to Reinforcement Learning}\label{sec:TE_KE_based_RL}
\subsection{Tabular Version}
Reinforcement learning describes a collection of learning techniques for sequential decision processes, in which an agent aims to maximize its reward while interacting with an environment\rdd{;} see \cite{sutton2018reinforcement} and \cite{bertsekas2019reinforcement}. The problem is modeled as a Markov Decision Process (MDP, \citealt{puterman1994markov}), consisting of a state space $\mathcal{S}$, a finite action space $\mathcal{A}$, a state transition probability distribution $\mathcal{P}: \mathcal{S} \times \mathcal{A} \times \mathcal{S} \rightarrow [0,1]$, a bounded reward \rdd{distribution} $\mathcal{R}: \mathcal{S} \times \mathcal{A} \rightarrow \rdd{\mathcal{P}_{\mathbb{R}}}$\rdd{, where $\mathcal{P}_{\mathbb{R}}$ is the set of probability distributions over $\mathbb{R}$,} and a discount factor $\gamma \in [0,1]$. If $\gamma = 1$, we assume there is a zero-reward absorbing state and that the probability of reaching this state converges to one as time tends to infinity\rdd{;} see \cite{lan2020maxmin}. At each time step $t$, the agent takes an action $a_t \in \mathcal{A}$ based on state information $s_t \in \mathcal{S}$, receives a reward \rd{$r_t \sim \mathcal{R}(s_t, a_t)$}, and transitions with probability $\mathcal{P}(s_{t+1} \mid s_t, a_t)$ to a new state $s_{t+1} \in \mathcal{S}$. Objective is to optimize for a policy $\pi: \mathcal{S} \times \mathcal{A} \rightarrow [0,1]$, a mapping from states to distributions over actions, that maximizes the expected return $\operatorname{E}_{\pi}\left[ \sum_{t=0}^{\infty} \gamma^t r_t\right]$. \rdd{Value-based methods, which are very common RL approaches,} define action-values $Q^{\pi}(s,a) = \operatorname{E}_{\pi}\left[ \sum_{t=0}^{\infty} \gamma^t r_t \mid s_0 = s, a_0 = a \right]$ for a certain policy. Thus, $Q^{\pi}(s,a)$ is the expected return when starting in state $s$, executing $a$, and following policy $\pi$ afterwards. There exists an optimal deterministic stationary policy $\pi^*(s) = \argmax_{a \in \mathcal{A}} Q^*(s,a)$, that is connected with optimal action-values $Q^*(s,a) = \max_{\pi} Q^{\pi}(s,a)$ for all $s \in \mathcal{S}$ and $a \in \mathcal{A}$ if the state space is finite or countably infinite \citep[Theorem 6.2.10]{puterman1994markov}. To optimize for $Q^*(s,a)$, one uses a recursive relationship known as \cite{bellman1954theory} optimality equation:
\begin{equation}\label{eq:Bellman_opt_eq}
    Q^*(s,a) = \mathcal{R}(s,a) + \gamma \sum_{s' \in \mathcal{S}} \mathcal{P}(s' \mid s, a) \max_{a' \in \mathcal{A}} Q^*(s',a'),
\end{equation}
where $s'$ is a successor state after performing action $a$ in state $s$. Since $Q^*(s',a')$ is the expected return from executing $a'$ in $s'$ and following the optimal policy afterward, the problem immediately appears as an instance of the estimation of the MEV of a set of random variables, namely the stochastic returns. Consequently, the methodology of Section \ref{sec:Estimating_max_mu} applies. $Q$-Learning, from \cite{watkins1992q}, translates (\ref{eq:Bellman_opt_eq}) into a sample-based algorithm by using the ME:
\begin{equation*}
\rd{
    \hat{Q}^{*}_{t+1}(s_t, a_t) \leftarrow \hat{Q}^{*}_t(s_t, a_t) + \tau_t(s_t, a_t) \left[y_{t}^{Q} - \hat{Q}^{*}_t(s_{t}, a_t) \right],
}
\end{equation*}
\rd{with target $y_{t}^Q = r_t + \gamma \max_{a \in \mathcal{A}} \hat{Q}^{*}_t(s_{t+1}, a)$ and learning rate $\tau_t(s_t, a_t)$. An estimate of $Q^{*}(s_t,a_t)$ at time $t$ is denoted $\hat{Q}^{*}_t(s_t,a_t)$.} The algorithm is known to converge to the optimal action-values if the conditions of \cite{robbins1951stochastic} on the learning rate are fulfilled, and each state-action pair is visited infinitely often \citep{tsitsiklis1994asynchronous}. However, especially in early stages of the training when the $Q$-estimates are imprecise, the algorithm tends to transmit overestimated values. Overcoming this issue, \cite{hasselt2010double} uses the DE and stores two separate $Q$-tables with estimates $\hat{Q}^{*}_{A,t}$ and $\hat{Q}^{*}_{B,t}$, leading to the target:  $y^{DQ}_t = r_t + \gamma \hat{Q}^{*}_{B,t}\left[s_{t+1}, \argmax_{a \in \mathcal{A}} \hat{Q}^{*}_{A,t}(s_{t+1}, a)\right]$. To apply the TE and KE in a $Q$-Learning setup, we propose to replace the target with $y^{KQ}_t = r_t + \gamma \operatorname{KE_a} \hat{Q}^{*}_{t}(s_{t+1}, a),$ where:
\begin{align}
    \operatorname{KE_a} \hat{Q}^{*}_{t}(s_{t+1}, a) &= \left\{\sum_{a\in \mathcal{A}} \kappa\left[T_{\hat{Q}^{*}_t}(s_{t+1},a)\right]\right\}^{-1} \sum_{a\in \mathcal{A}} \kappa\left[T_{\hat{Q}^{*}_t}(s_{t+1},a)\right] \hat{Q}^{*}_t(s_{t+1},a),\nonumber\\
    T_{\hat{Q}^{*}_t}(s_{t+1},a) &= \frac{\hat{Q}^{*}_t(s_{t+1},a) - \max_{a' \in \mathcal{A}}\hat{Q}^{*}_t(s_{t+1},a')}{\sqrt{\widehat{\operatorname{Var}}_t\left[\hat{Q}^{*}_t(s_{t+1},a)\right] + \widehat{\operatorname{Var}}_t\left[\hat{Q}^{*}_t(s_{t+1}, a^*) \right]}},\label{eq:KQ_target}
\end{align}
for a maximizing action $a^* \in \{a \in \mathcal{A} \mid \hat{Q}^{*}_t(s_{t+1}, a) = \max_{a' \in \mathcal{A}}\hat{Q}^{*}_t(s_{t+1},a')\}$. The variance estimate of $\hat{Q}^{*}_t(s_{t+1}, a)$ for some action $a \in \mathcal{A}$ at time $t$ is denoted $\widehat{\operatorname{Var}}_t\left[\hat{Q}^{*}_t(s_{t+1}, a)\right]$ and will be generated following the proposal of \cite{d2019exploiting}. First, the process variance of the underlying return of the visited state-action pair $(s_t, a_t)$ is estimated via an exponentially-weighted online update:
\begin{equation}\label{eq:proc_sigmaSQ}
\resizebox{1.0 \textwidth}{!} 
{
    $
    \widehat{\sigma^{2}}_{\text{pro}, t+1}(s_t, a_t) \leftarrow  \left[1-\tau_t(s_t,a_t)\right] \left\{\widehat{\sigma^{2}}_{\text{pro}, t}(s_t, a_t) + \tau_t(s_t,a_t)\left[y^{KQ}_t - \hat{Q}^{*}_t(s_t, a_t) \right]^2 \right\}.
    $
}
\end{equation}
Second, to get the variance estimate in (\ref{eq:KQ_target}) for some $a \in \mathcal{A}$, a normalization by \cite{kish1965survey} effective sample size $n_{\text{eff},t}(s_{t+1},a)$ is performed:
\begin{equation*}
    \widehat{\operatorname{Var}}_t\left[\hat{Q}^{*}_t(s_{t+1},a)\right] = \frac{ \widehat{\sigma^{2}}_{\text{pro}, t}(s_{t+1}, a)}{n_{\text{eff},t}(s_{t+1},a)}.
\end{equation*}
The effective sample size weights each sample depending on the learning rate and is computed via $n_{\text{eff},t}(s_{t+1},a) = \frac{\left[\omega_t(s_{t+1},a)\right]^2}{\omega_{t}^2(s_{t+1},a)}$, in which numerator and denominator are incrementally updated for each visited state-action pair $(s_t, a_t)$ :
\begin{align}
\omega_{t+1}(s_{t},a_t) &\leftarrow [1-\tau_t(s_t,a_t)]\omega_t(s_{t},a_t) + \tau_t(s_t,a_t),\nonumber \\
\omega_{t+1}^2(s_{t},a_t) &\leftarrow [1-\tau_t(s_t,a_t)]^2\omega_{t}^2(s_{t},a_t) + [\tau_{t}(s_t,a_t)]^2.\label{eq:proc_2_main}
\end{align}
With this approach, we can introduce TE-$Q$-Learning (TE-$Q$) and KE-$Q$-Learning (KE-$Q$), respectively, being summarized in Algorithm \ref{algo:TE_KE_based_Q_Learning}. The following \rdd{theorem} states the convergence of the algorithm to the optimal action-values with probability 1.
\begin{Theorem}\label{pro:KE_Q_convergence}
    \rdd{Let the following regularity conditions be fulfilled:}
    \begin{enumerate}
    \item \rdd{The MDP is finite.}
    \item \rdd{$\gamma \in [0,1).$}
    \item \rdd{The learning rates satisfy $\tau_t(s,a) \in [0,1]$, $\sum_t \tau_t(s,a) = \infty$, $\sum_t \tau_t^{2}(s,a) < \infty$ all with probability 1 for all $s \in \mathcal{S}, a \in \mathcal{A}$.}
    \item \rdd{The reward function is bounded.}
    \item \rdd{Each state-action pair is visited infinitely often.}
    \end{enumerate}
    \rdd{Then the following holds} for the random sequence of action-value estimates $\hat{Q}^{*}_{t}$ generated by TE/KE-$Q$-Learning:
    $$\mathbb{P}\left[\lim_{t\rightarrow \infty}\hat{Q}^{*}_{t}(s,a)=Q^{*}(s,a)\right]=1 \quad \forall s \in \mathcal{S}, a \in \mathcal{A}.$$
\end{Theorem}
\begin{proof}
    Please refer to \ref{appendix:convergence_proof}.
\end{proof}
\rd{We emphasize that the proof scheme can be similarly applied to Weighted $Q$-Learning \citep{d2016estimating}, for which a proof of convergence has been missing so far.}

\rd{As a note, we mention the possibility of incorporating a covariance correction, as shown in (\ref{eq:TST_corrected}), within the definition of the target in (\ref{eq:KQ_target}). Implementing such an approach would entail updating the covariances between the return estimates for all actions of each state and can possibly constructed analogous to (\ref{eq:proc_sigmaSQ}). However, assume that only two actions, $A$ and $B$, are available in some state. To estimate the return covariance between these actions, one needs to execute $A$, compute its target, and do the same for $B$. Although a parallelized solution is straightforward in a simulation environment, for the sake of simplicity and taking into account the insights from Section \ref{sec:Analysis_Bias_Var_MSE} that suggest cross-sample dependencies are not as critical, we opt to retain the test statistics as defined in (\ref{eq:KQ_target}).}

\begin{algorithm}[ht]
\begin{small}
\setstretch{1.10}
\SetAlgoLined
 \rd{\textbf{initialize} $\forall s \in \mathcal{S}, a \in \mathcal{A}:$ $\hat{Q}^{*}(s,a) = 0, \widehat{\sigma^{2}}_{\rm pro}(s, a) > 0 , \rdd{\omega}(s,a) \in (0,1], \rdd{\omega}^2(s,a) \in (0,1]$}\\
 \DontPrintSemicolon
\Repeat{}{
    Initialize $s$\\
    \Repeat{$s$ is terminal}{
      Choose action $a$ from state $s$ with policy derived from \rd{$\hat{Q}^*$} (e.g., $\epsilon$-greedy)\\
      Take action $a$, observe reward $r$ and next state $s'$\\[0.75ex]
      \emph{Update effective sample size:}\\
      \quad $\omega(s,a) \leftarrow (1-\tau)\omega(s,a) + \tau$\\
      \quad $\omega^2(s,a) \leftarrow (1-\tau)^2\omega^2(s,a) + \tau^2$\\
      \quad $n_{\rm eff}(s,a) \leftarrow \frac{\omega(s,a)^2}{\omega^2(s,a)}$\\[0.75ex]
      Calculate target $y_{KQ}$\\[0.75ex]
      \emph{Update process variance:}\\
      \quad $\rd{\widehat{\sigma^{2}}_{\rm pro}}(s, a) \leftarrow  (1-\tau) \left\{\rd{\widehat{\sigma^{2}}_{\rm pro}}(s, a) + \tau\left[y^{KQ} - \hat{Q}^{*}(s, a) \right]^2 \right\}$\\[0.75ex]
      \emph{Update $Q$-estimate:}\\
      \quad $\hat{Q}^{*}(s, a) \leftarrow \hat{Q}^{*}(s, a) + \tau \left[y^{KQ} - \hat{Q}^{*}(s, a) \right]$\\
      $s \leftarrow s'$\\
    }
}
\caption{TE-$Q$-Learning/KE-$Q$-Learning}
\label{algo:TE_KE_based_Q_Learning}
\end{small}
\end{algorithm}

\subsection{Deep Version}
\rd{Real-world control tasks frequently entail continuous state spaces, which necessitates to employ strategies such as state aggregation \citep{singh1994reinforcement} to make tabular algorithms applicable. In recent research, there has been a growing emphasis on leveraging DNNs as function approximators to parameterize the function \rd{$\hat{Q}^{*}_t(s_t,a_t; \theta_t)$, with $\theta_t$ being the parameter set of the neural network at time $t$}.} The resulting DQN \citep{mnih2015human} and their extensions \citep{van2016deep, dabney2018distributional, hessel2018rainbow} have shown breakthrough performances on various challenging tasks \rd{\citep{bellemare2020autonomous, barata2023reinforcement}}, thereby revolutionizing the capabilities of modern artificial intelligence.

The optimization procedure of these algorithms is still based on the Bellman optimality equation (\ref{eq:Bellman_opt_eq}), but uses gradient \rd{descent} to update \rd{$\theta_t$:
\begin{equation*}
    \theta_{t+1} \leftarrow \theta_t + \tau_t \left[y^{DQN}_t -\hat{Q}^{*}_t(s_t,a_t;\theta_t) \right] \nabla_{\theta_t} \hat{Q}^{*}_t(s_t,a_t;\theta_t), 
\end{equation*}
where $y^{DQN}_t = r_t + \gamma \max_{a \in \mathcal{A}} \hat{Q}^{*}_t(s_{t+1},a;\theta^{-}_t)$. The set $\theta^{-}_t$ refers to the parameters of the target network, which is a time-delayed copy of the main network with parameter $\theta_t$.} Instead of updating fully online, DQN samples minibatches of past experiences from a replay buffer $D$ to stabilize training. \cite{van2016deep} proposed the Double DQN (DDQN) and uses the DE to compute the target: \rd{$y^{DDQN}_t = r_t + \gamma\,\hat{Q}^{*}_t[s_{t+1}, \argmax_{a \in \mathcal{A}} \hat{Q}^{*}_t(s_{t+1},a;\theta_t);\theta^{-}_t]$.} The action selection is performed via the main network, while the evaluation uses the target network like the regular DQN.

To translate the TE and KE to the function approximation case, we require a variance estimate of the $Q$-estimates. We follow \cite{d2019exploiting} and use the framework of the BDQN \citep{osband2016deep} to accomplish this task. Generally, the bootstrap is a method for computing measures of accuracy for statistical estimates \citep{efron1994introduction}. The method trains different regressors of the target function based on bootstrap samples generated by sampling with replacement from the original dataset. The BDQN transfers this idea to the DQN algorithm by maintaining $K \in \mathbb{N}$ differently initialized \rd{$Q$-networks $\hat{Q}^{*}_{k,t}(s_t,a_t;\theta_{k,t})$ with parameter $\theta_{k,t}$ at time $t$ and $k = 1,\ldots,K$, each equipped with its own target network $\hat{Q}^{*}_{k,t}(s_t,a_t;\theta_{k,t}^-)$.} This DQN modification was proposed to improve over the usual $\epsilon$-greedy exploration strategy. At the beginning of each episode, one \rd{$\hat{Q}^{*}_{k,t}$} is selected randomly, and the agent acts \rd{greedily} with relation to this \rd{$\hat{Q}^{*}_{k,t}$}. At test time, the majority vote of the function approximators is used. 

Generally, the BDQN can be implemented using $K$ different networks or maintaining a common network body and specifying $K$ different heads. We pursue the latter approach. Diversification across the heads is achieved by different random initialization of the parameters and random generation of binary masks \rd{$m^{i}_{k} \in \{0,1\}$, where $i$ refers to the $i$-th transition tuple of the replay buffer, and $k = 1,\ldots,K$. The masks indicate which head should be trained on which sample. More precisely, if some tuple $i$ is sampled from the buffer, the $k$-th head of the ensemble will only be trained if $m^{i}_{k}$, the $k$-th mask associated with tuple $i$, takes value 1. All masks are typically sampled from the same distribution, and a Bernoulli distribution with parameter $p$ is a possible choice.}

Crucially, through the $K$ heads, we can directly use the sample variance of the $Q$-estimates and compute the target \rd{$y^{KDQN}_{k,t}$ for the $k$-th network $\hat{Q}^{*}_{k,t}$:
\vspace{-0.3cm}
\begin{equation}\label{eq:y_KDQN}
\resizebox{1.0 \textwidth}{!} 
{
    $
    y^{KDQN}_{k,t} = r_t + \gamma \left\{\sum_{a\in \mathcal{A}} \kappa\left[T_{\hat{Q}^{*}_{k,t}}(s_{t+1},a)\right]\right\}^{-1} \sum_{a\in \mathcal{A}} \kappa\left[T_{\hat{Q}^{*}_{k,t}}(s_{t+1},a)\right] \hat{Q}^{*}_{k,t}(s_{t+1},a;\theta_{k,t}^{-}),
    $
}
\end{equation}
where
\begin{equation}\label{eq:FA_KQ}
    T_{\hat{Q}^{*}_{k,t}}(s_{t+1},a) = \frac{\hat{Q}^{*}_{k,t}(s_{t+1},a;\theta^{-}_{k,t}) - \max_{a' \in \mathcal{A}}\hat{Q}^{*}_{k,t}(s_{t+1},a';\theta_{k,t}^{-})}{\sqrt{\widehat{\operatorname{Var}}_t\left[\hat{Q}^{*}_{k,t}(s_{t+1},a;\theta_{k,t}^{-})\right] + \widehat{\operatorname{Var}}_t\left[\hat{Q}^{*}_{k,t}(s_{t+1}, a^*;\theta_{k,t}^{-}) \right]}},
\end{equation}
for an action ${a^* \in \{a \in \mathcal{A} \mid \hat{Q}^{*}_{k,t}(s_{t+1},a; \theta_{k,t}^-) = \max_{a' \in \mathcal{A}} \hat{Q}^{*}_{k,t}(s_{t+1}, a'; \theta_{k,t}^{-})\}}$.
\smallskip
The resulting gradient $g_{k,t}^{i}$ for the $i$-th tuple from the replay buffer is:
\begin{equation}\label{eq:BootDQN_KE_TE_grad}
    g_{k,t}^{i} = m_{k}^{i}\left[y_{k,t}^{KDQN,i} - \hat{Q}^{*}_{k,t} (s_i, a_i; \theta_{k,t}) \right] \nabla_{\theta_{k,t}}\hat{Q}^{*}_{k,t}(s_i, a_i;\theta_{k,t}).
\end{equation}
The full procedure is termed TE-BDQN and KE-BDQN, respectively, being detailed in Algorithm \ref{algo:TE_KE_based_BootDQN}. Similar to the tabular case, there is the possibility to include a covariance correction in the denominator of (\ref{eq:FA_KQ}), thereby combating possible cross-sample dependencies. Although the implementation is straightforward due to the bootstrap ensemble, we propose a more holistic adaptive bias reduction scheme in the next section, which goes beyond mitigating the consequences of cross-sample dependencies.}

\begin{algorithm}[ht]
\begin{small}
\setstretch{1.10}
\DontPrintSemicolon
\SetAlgoLined
 \textbf{initialize} Action-value estimate networks with $K$ outputs $\left\{\hat{Q}^{*}_k \right\}^{K}_{k=1}$, masking distribution $M$, empty replay buffer $D$\\
\Repeat{}{
    Initialize $s$\\
    Pick a value function to act: $k \sim \text{Uniform}\{1,\ldots,K\}$\\
    \Repeat{$s$ is terminal}{
      Choose action $a$ from state $s$ with greedy policy derived from $\hat{Q}^{*}_k$\\
      Take action $a$, observe reward $r$ and next state $s'$\\
      \rd{Sample bootstrap masks $m = (m_1, \ldots, m_K)$}\\
      Add $(s, a, r, s', m)$ to replay buffer $D$\\
      Sample random minibatch of transitions $\left\{(s_i, a_i, s'_{i}, r_i, \rd{m^i})\right\}_{i=1}^{B}$ from $D$\\
      Perform gradient descent step based on (\ref{eq:BootDQN_KE_TE_grad})\\
      Every $C$ steps reset $\rd{\theta_{k}^{-} = \theta_k}$ for $k = 1,\ldots, K$\\
      $s \leftarrow s'$\\
    }
}
\caption{TE-BDQN/KE-BDQN}
\label{algo:TE_KE_based_BootDQN}
\end{small}
\end{algorithm}

\section{Adaptive Absolute Bias Minimization}\label{subsec:adaptive_bias}
While the TE can interpolate between under- and overestimation by selecting a smaller or larger $\alpha$, it is a priori not known which $\alpha$ is adequate for an unknown environment. The method of choice for practitioners in such situations is a grid search to select a suitable value empirically. Next to the selection issue, a fixed parameter might not be sufficient for controlling the estimation bias, and, e.g., an ascending strategy would be favourable. Leveraging these considerations, we propose an adaptive modification of TE-BDQN to adjust $\alpha$ under the objective of minimizing the absolute estimation bias during training.

\subsection{Bias Estimation}\label{subsec:bias_estimation}
Before introducing the adaptive mechanism, we first outline how to estimate the bias of given action-value estimates at a certain point in training. Following \cite{chen2021randomized}, we consider a current policy $\pi$, which is connected with true action-values $Q^{\pi}(s,a)$ \rd{for $s \in \mathcal{S}, a \in \mathcal{A}$}. The aggregated bias of estimates $\hat{Q}^{\pi}(s,a)$ of $Q^{\pi}(s,a)$ for all $s \in \mathcal{S}, a \in \mathcal{A}$ is defined as: $$\operatorname{Bias}(\hat{Q}^{\pi},\pi) = \operatorname{E}_{s\sim \rho^{\pi}, a\sim\pi}[\hat{Q}^{\pi}(s,a) - Q^{\pi}(s,a)],$$
where $\rho^{\pi}$ is the state-visitation distribution of $\pi$. \cite{chen2021randomized} proposed to repeatedly run analysis episodes from random initial states while following $\pi$. The observed Monte Carlo return of an encountered state-action pair serves as an unbiased estimate of its true $Q$-value. Averaging over all encountered state-action pairs yields the estimate of the estimation bias:
\begin{equation}\label{eq:bias_estimate}
\widehat{\operatorname{Bias}}(\hat{Q}^{\pi},\pi) = \frac{1}{\lvert \mathcal{T} \rvert} \sum_{(s,a,R) \in \mathcal{T}} [\hat{Q}^{\pi}(s,a) - R],    
\end{equation}
where $\mathcal{T}$ is a set of encountered $(s,a,R)$-tuples, where $s$ is the state, $a$ the executed action\rd{,} $R$ the (later) observed Monte Carlo return\rd{, and} $\lvert \mathcal{T} \rvert$ is the cardinality of $\mathcal{T}$. \cite{chen2021randomized} applied this procedure to the Soft-Actor Critic algorithm \citep{haarnoja2018soft}, which uses an actor dictating the policy $\pi$ and a critic providing the estimates $\hat{Q}^{\pi}$. In standard $Q$-Learning, no explicit actor is providing the policy, and the algorithm directly approximates the optimal action-values $Q^*(s, a)$, leading to estimates $\hat{Q}^{*}(s,a)$ \citep{sutton2018reinforcement}. To still generate insights into the action-value estimation accuracy of the algorithm, \cite{van2016deep} compare the $\hat{Q}^{*}(s,a)$ generated \emph{during} training with Monte Carlo returns of the final greedy policy \emph{after} training. Although this approach is certainly instructive, it does not enable an assessment without having a converged baseline.

\rd{We instead propose to use (\ref{eq:bias_estimate}) with the policy being defined as $\pi(s) = \argmax_a \hat{Q}^{*}(s,a)$ for all $s \in \mathcal{S}$, being briefly summarized in Algorithm \ref{algo:Bias_estimation}. Crucially, the algorithm can be run already \emph{during} training, say always after a fixed number of steps, allowing to get an intuition about the development of the value estimation bias over the training time.} The method is justified since $Q$-Learning uses a greedy target policy in its update. Thus, the algorithm \emph{evaluates} the greedy policy with respect to its own action-value estimates, and we can assess whether the $Q$-estimates are too optimistic or pessimistic. \rdd{Crucially, we emphasize that the greedy bias estimation roll-outs of Algorithm \ref{algo:Bias_estimation} are independent of the learning episodes, which rely on an explorative policy.} Transferred to the BDQN and its modifications, the procedure can be similarly applied by assessing each head separately since they constitute different estimates $\hat{Q}^{*}_k$ for $k = 1, \ldots, K$. Thus, we run Algorithm \ref{algo:Bias_estimation} for each head and average the output to receive an aggregated bias estimate for the bootstrap-ensemble.

\begin{algorithm}[ht]
\begin{small}
\setstretch{1.10}
\DontPrintSemicolon
\SetAlgoLined
 \textbf{Input} Estimates $\hat{Q}^{*}(s,a)$ for all $s \in \mathcal{S}, a \in \mathcal{A}$\\
 Set $\mathcal{T} = \emptyset$ and $\pi(s) = \argmax_a \hat{Q}^{*}(s,a)$ for all $s \in \mathcal{S}$\\
\For{number of episodes}{
    Randomly initialize $s$\\
    Play episode following $\pi$ and append encountered state-action-return tuples to $\mathcal{T}$
}
\Return{$\frac{1}{\lvert \mathcal{T} \rvert} \sum_{(s,a,R) \in \mathcal{T}} [\hat{Q}^{*}(s,a) - R]$}
\caption{Bias estimation for $Q$-Learning like algorithms}
\label{algo:Bias_estimation}
\end{small}
\end{algorithm}

\subsection{Adaptive TE-BDQN}\label{sec:Adaptive_TE}
Algorithm \ref{algo:Bias_estimation} will be used to assess the algorithms in the experiments of Section \ref{sec:Experiments}. Furthermore, the approach serves as a basis for dynamically adjusting the $\alpha$ of the TE-BDQN. Intuitively, since larger $\alpha$ leads to larger $Q$-estimates, $\alpha$ is reduced if the $Q$-estimates are too high. Vice versa, we increase $\alpha$ if the $Q$-estimates are too small. Precisely, we perform the following update:
\rd{\begin{equation}\label{eq:AdaTE_update}
    \alpha_{t+1} \leftarrow \alpha_t + \frac{\tau_{\rm Ada}}{K} \sum_{k=1}^{K} \sum_{\Tilde{t} = 1}^{T_{\rm Ada}} \left[R_k(s_{\Tilde{t},k}, a_{\Tilde{t},k}) - \hat{Q}^{*}_{k,t} (s_{\Tilde{t},k}, a_{\Tilde{t},k}; \theta_k) \right],
\end{equation}
with step size $\tau_{\rm Ada}$ and roll-out length $T_{\rm Ada}$. Importantly, we use $n$-steps returns \citep{sutton2018reinforcement}:
\begin{align*}
    R_k(s_{\Tilde{t},k},a_{\Tilde{t},k}) = r_{\Tilde{t},k} + \gamma r_{\Tilde{t}+1,k} + \gamma^2 r_{\Tilde{t}+2,k} + \ldots + \gamma^{T_{\rm Ada} -\Tilde{t}} r_{T_{\rm Ada},k} \\
    + \gamma^{T_{\rm Ada} - \Tilde{t} +1} \max_{a} \hat{Q}^{*}_{k,t}(s_{T_{\rm Ada} +1,k}, a; \theta_k),
\end{align*}
where $\Tilde{t} = 1, \ldots, T_{\rm Ada}$. Consistent with the notation above, we use the index $t$ to refer to step size of the overall experiment, and $\Tilde{t}$ to the step size inside a Monte Carlo roll-out.} We denote state and action at time \rd{$\Tilde{t}$} under head $k$ as \rd{$s_{\Tilde{t},k}$} and \rd{$a_{\Tilde{t},k}$}, respectively. The resulting immediate reward is \rd{$r_{\Tilde{t},k}$} and the initial states $s_{1,k}$ for $k = 1, \ldots, K$ are randomly sampled from the replay buffer. As motivated in Section \ref{subsec:bias_estimation}, the actions \rd{$a_{\Tilde{t},k}$} under head $k$ are selected by acting \rd{greedily} with relation to \rd{$\hat{Q}^{*}_{k,t}$}. Although $n$-step returns are generally not \rd{unbiased estimates} of the expected return of a policy like a complete Monte Carlo roll-out, we found them empirically much more practicable since they do not require running full episodes while still allowing us to judge the accuracy of the current value estimates. Consequently, this approach can also be applied to non-episodic problems.

We update $\alpha$ immediately after the target networks, avoiding instabilities in the learning process. A similar proposal to (\ref{eq:AdaTE_update}) in an episodic context with continuous action spaces based on full Monte Carlo roll-outs has recently been made by \cite{dorka2021adaptively}. Note that it is possible to maintain a separate $\alpha$ for each bootstrap head, enabling a tailored parametrization for the bias of each approximator. However, we only consider one parameter for the whole ensemble for simplicity in the following. The resulting algorithm is called Ada-TE-BDQN and is shown in \ref{appendix:Ada_TE_BDQN}.

\section{Experiments}\label{sec:Experiments}
We analyze the proposed estimators of the MEV on a statistically motivated real-world example before considering two tabular environments that serve as a proof-of-concept for TE/KE-$Q$-Learning. The experiments with function approximation are carried out in the MinAtar environments of \cite{young2019minatar}, which allow for a thorough algorithmic comparison.

\subsection{Internet Ads}
\begin{figure}[htp!]
    \centering
    \includegraphics[width=0.9\linewidth]{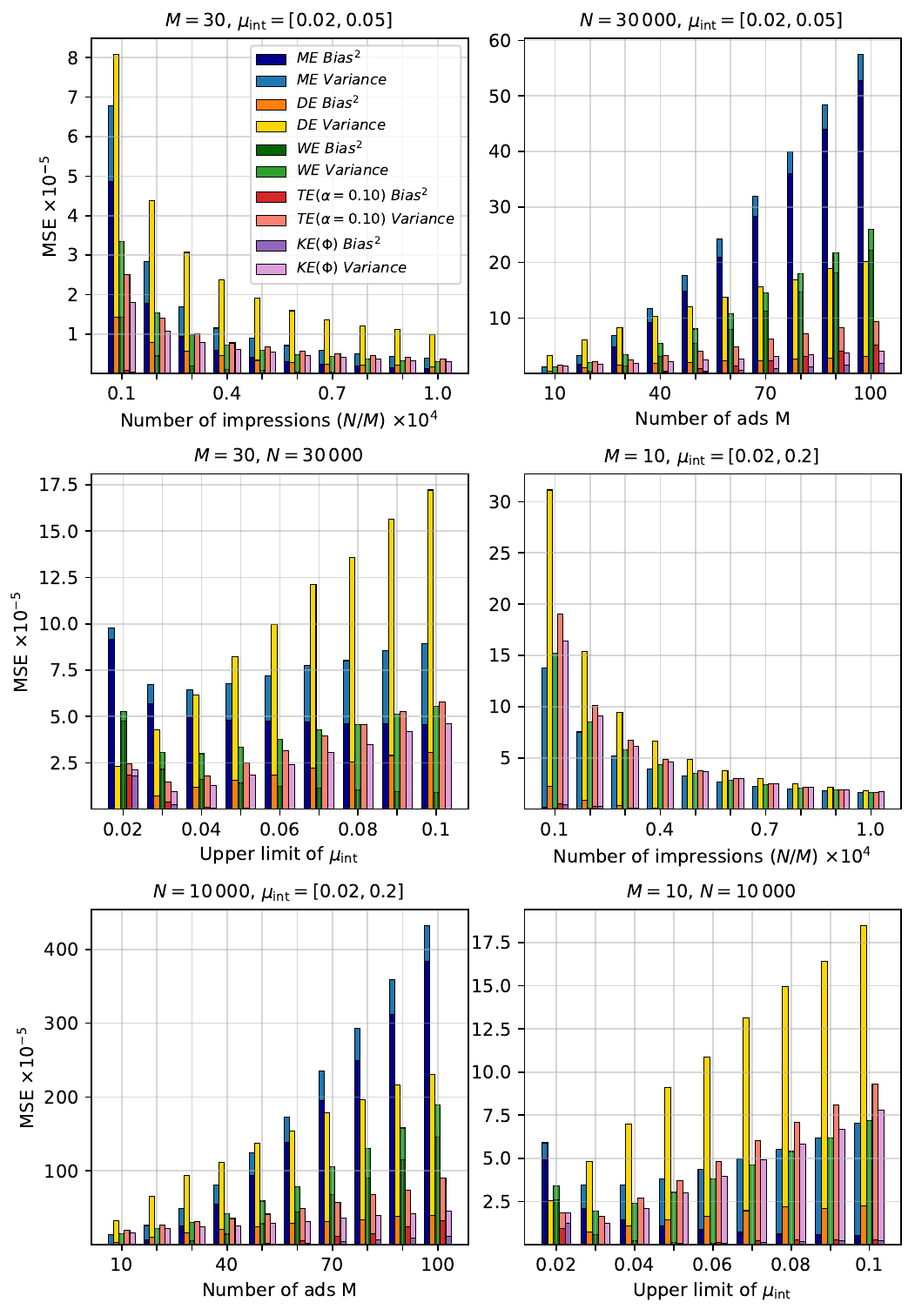}
    \caption{Comparison of ME, DE, WE, TE, and KE on the internet ad problem. Results are averaged over 10\,000 runs.}
    \label{fig:Internet_Ad}
\end{figure}

We consider the internet ad problem previously studied by \cite{van2013estimating}, \cite{d2021gaussian}, and \cite{jiang2021action}. There are $M$ different ads, and each has the same return per click. Consequently, the click rate is the only quantity of interest and is modeled as a Bernoulli variable (click or no click). The true expectations $\mu_1, \ldots, \mu_M$ of these $M$ variables equal the respective click probability and are equally spaced in a specific interval $\mu_{int}$. We consider $N$ customers and assume that the ads are presented equally often to have $N/M$ samples per ad. \rdd{Thus, the sample size for each of the $M$ Bernoulli variables is $N/M \gg 1$, and we refer to this ratio as the number of impressions. In addition, we assume $N/M$ is integer-valued.} The objective is to estimate the maximum true mean accurately, thus finding the best ad based on the given samples. We compare the TE ($\alpha = 0.1$) and the KE (standard Gaussian kernel) with the ME, DE, and WE based on bias, variance, and MSE. Six configurations of the problem are considered by varying the number of customers $N$, the number of ads $M$, or the upper limit of the sampling interval $\mu_{\rm int}$, while the lower limit is always fixed at 0.02. Figure \ref{fig:Internet_Ad} displays the results.

Both TE and KE yield lower MSE than their competitors in most scenarios. We emphasize that $\alpha = 0.1$ was not cherry-picked for this problem, and the TE's performance could thus be further increased by tailoring $\alpha$ for each experiment. In general, all estimators' MSEs decrease with an increasing number of ads, while they are more accurate with a higher number of customers $N$, constituting reasonable observations. The DE often yields large variances, while the ME provides biased estimates. Despite producing a higher MSE than TE and KE in most cases, the WE outperforms the conventional competitors ME and DE.

\subsection{Maximization Bias Example}\label{subsec:Simple_MDP}
We consider the example in Figure 6.5 of \cite{sutton2018reinforcement}. A simple MDP with two non-terminal states A and B is given\rd{, which is visualized in Figure \ref{fig:SimpleMDP_env}.} The agent starts in A. If it goes \rd{\emph{right}} from A, the episode ends, and zero reward is received. If action \rd{\emph{left}} is selected in A, the agent deterministically gets to state B and receives zero reward. There are eight actions to choose from B, but all lead to a terminal state and yield a reward sampled from $\mathcal{N}(-0.1, 1)$. The parameters are $\gamma = 1$, $\epsilon = 0.1$, and \rdd{learning rate} $\tau = 0.1$. Since this is an undiscounted task, the expected return starting with action \rd{\emph{left}} is $-0.1$, and the agent should always prefer going \rd{\emph{right}}. However, due to the random selection of the $\epsilon$-greedy strategy, the \rd{\emph{left}} action will always be picked at least $5\%$ in expectation.

\begin{figure}[ht]
    \centering
    \includegraphics[width=0.5\linewidth]{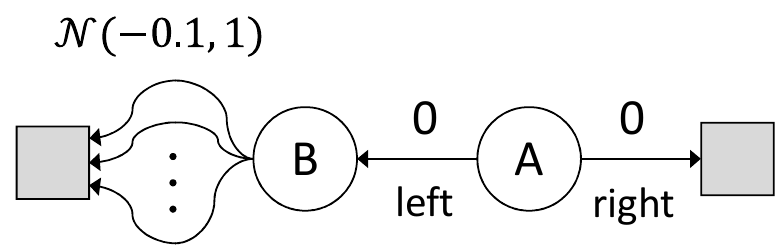}
    \caption{\rd{\footnotesize{The} environment of the Maximization Bias Example; compare Figure 6.5 of \cite{sutton2018reinforcement}.}}
    \label{fig:SimpleMDP_env}
\end{figure}

\begin{figure}[ht]
    \centering
    \includegraphics[width=\linewidth]{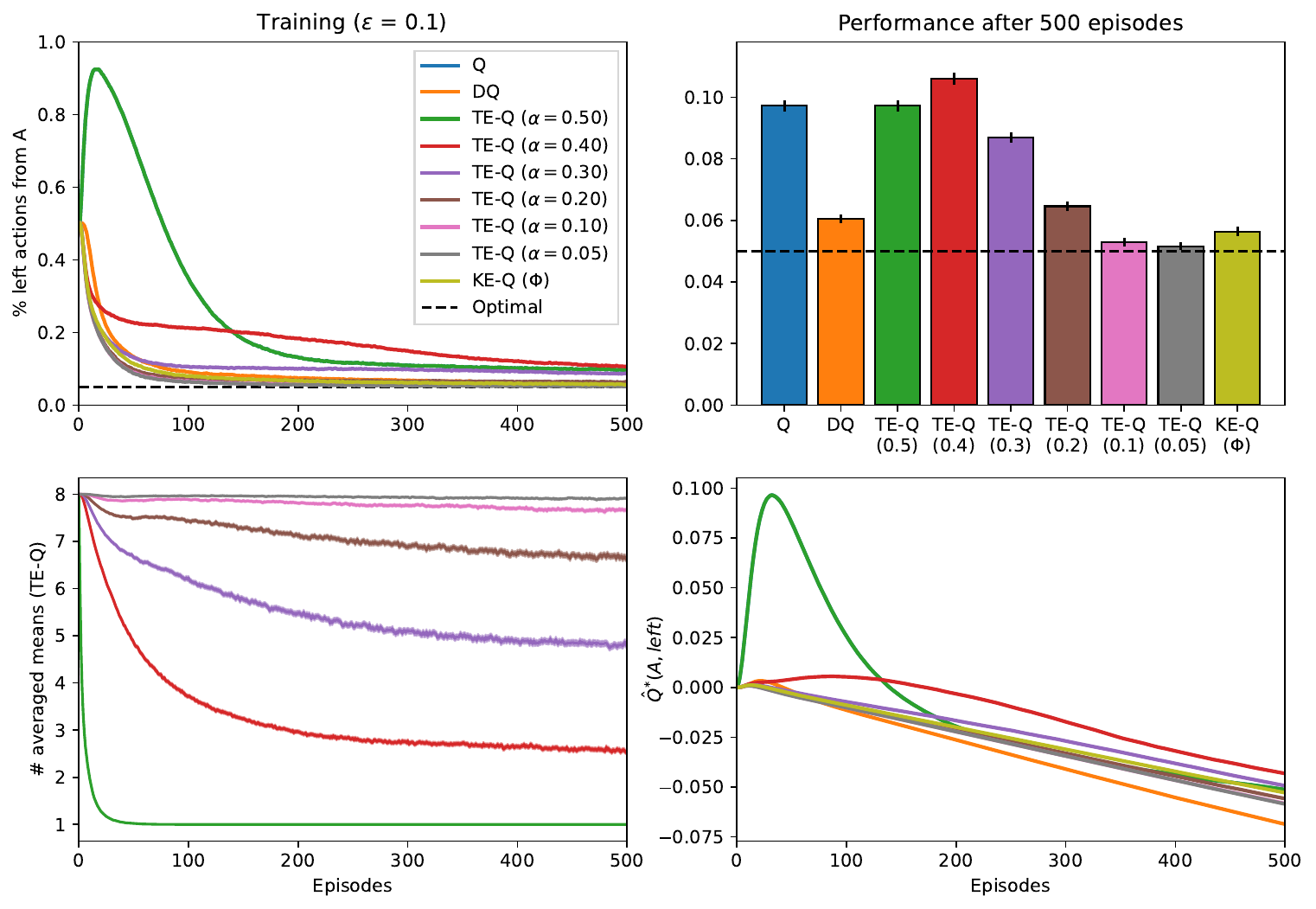}
    \caption{Maximization Bias Example from \cite{sutton2018reinforcement} with parameters $\gamma = 1$, $\epsilon = 0.1$, and $\tau = 0.1$. Q and DQ refer to $Q$-Learning and Double $Q$-Learning, respectively. Results are averaged over $100\,000$ runs and 95\% confidence intervals are included. The action-value estimate in the lower-right graph for DQ is generated by averaging over both $Q$-tables.}
    \label{fig:SimpleMDP}
\end{figure}

Figure \ref{fig:SimpleMDP} depicts the results. The upper-left part displays the percentage of selecting action \rd{\emph{left}} in A; the upper-right plot contains the same percentage after 500 training episodes. The lower-right graph shows the estimate of $Q^*$(A, left) over training. Finally, the lower-left plot displays the number of \rdd{non-rejected hypotheses of TE-$Q$-Learning when updating $\hat{Q}^*$(A, left) according to (\ref{eq:KQ_target}), which is equal to the number of $Q$-estimates of the follow-up state B that are averaged during the target computation. Emphasizing the relationship to (\ref{eq:T-Estimator}), we refer to this quantity as the number of averaged means.} $Q$-Learning is the same algorithm as TE-$Q$-Learning with $\alpha = 0.5$ and only included for comparison.

$Q$-Learning initially overestimates the value of the \rd{\emph{left}} action in state A and selects it nearly \rdd{10\% of the cases after 500 episodes, which is twice as optimal considering that $\epsilon = 0.1$}. Double $Q$-Learning performs better and achieves a final rate of roughly $6\%$. On the other hand, TE-$Q$-Learning can modulate the overestimation bias through its significance level $\alpha$ and reaches a near-optimal selection percentage for $\alpha \leq 0.10$. Interestingly, TE-$Q$ ($\alpha = 0.4$) performs worse after 500 episodes than $Q$-Learning. Although the initial overestimation is not as large as for $\alpha = 0.5$, the effect is more persistent, and more interactions are needed to reduce the estimate. For additional insights, we display the number of means averaged by the TE. TE-$Q$-Learning with $\alpha = 0.5$ naturally considers only the maximum sample mean, which is non-unique in the first several episodes since all action-value estimates are initialized with zero. The lower the significance level of TE-$Q$ gets, the more means are averaged until nearly all sample means are considered for $\alpha \leq 0.1$. Furthermore, KE-$Q$-Learning with the standard Gaussian kernel performs reasonably well and achieves a final selection rate below Double $Q$-Learning.

\subsection{Cliff Walking}\label{subsec:CliffWalking}
We examine the Cliff Walking task from Example 6.6 in \cite{sutton2018reinforcement}, which is an undiscounted, episodic task with start and goal states. Our environment is a grid of width 10 and height 5\rd{, being depicted in Figure \ref{fig:CW_env}.} Start state S is the lower-left grid point; goal state G is the lower-right grid point. All transitions are rewarded with $-1$, except those which lead to the grid points directly between S and G. Those are referred to as \rd{\emph{The Cliff}}, yield reward $-100$, and send the agent back to S. Actions are the four movement directions up, down, right, and left. Performance is measured via the return during an episode.

\begin{figure}[ht]
    \centering
    \includegraphics[width=0.5\linewidth]{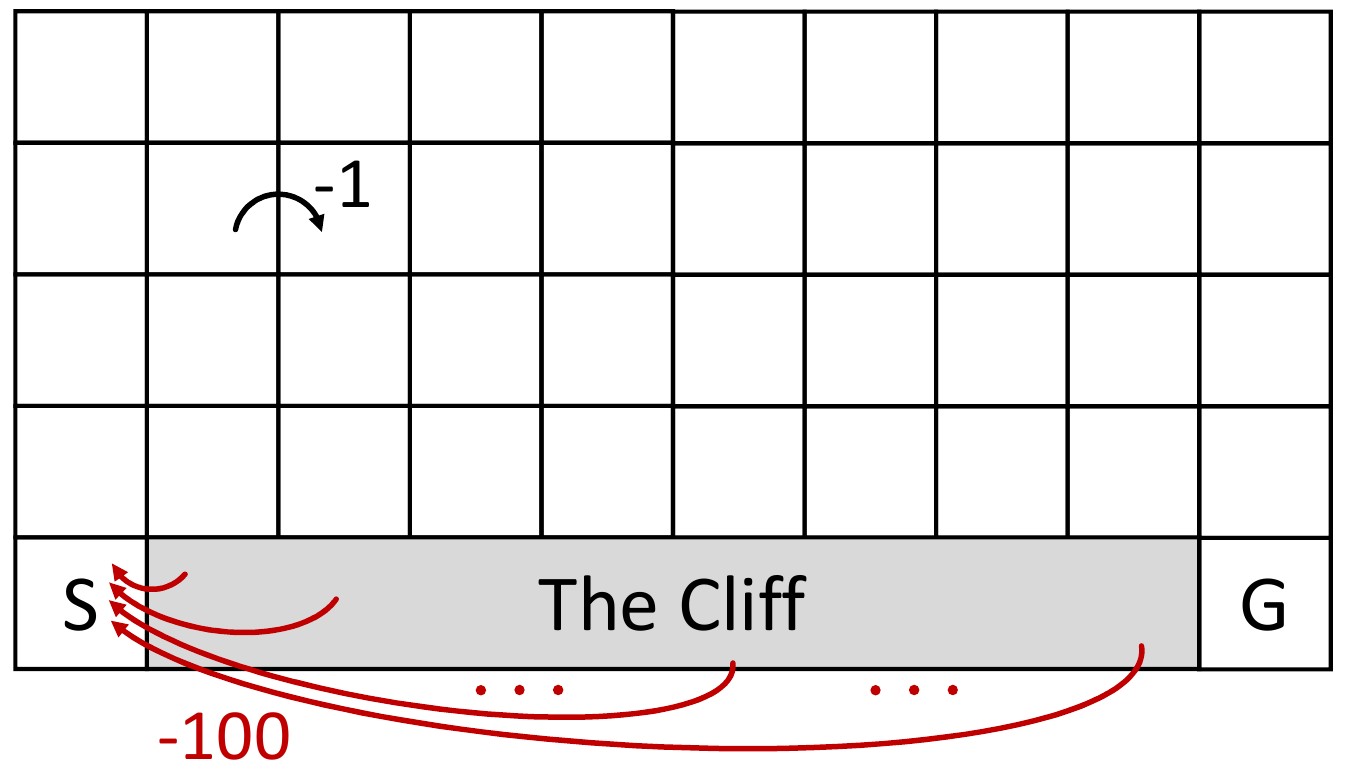}
    \caption{\rd{\footnotesize{The} Cliff Walking environment; compare Example 6.6 of \cite{sutton2018reinforcement}.}}
    \label{fig:CW_env}
\end{figure}

\begin{figure}[ht]
    \centering
    \includegraphics[width=\linewidth]{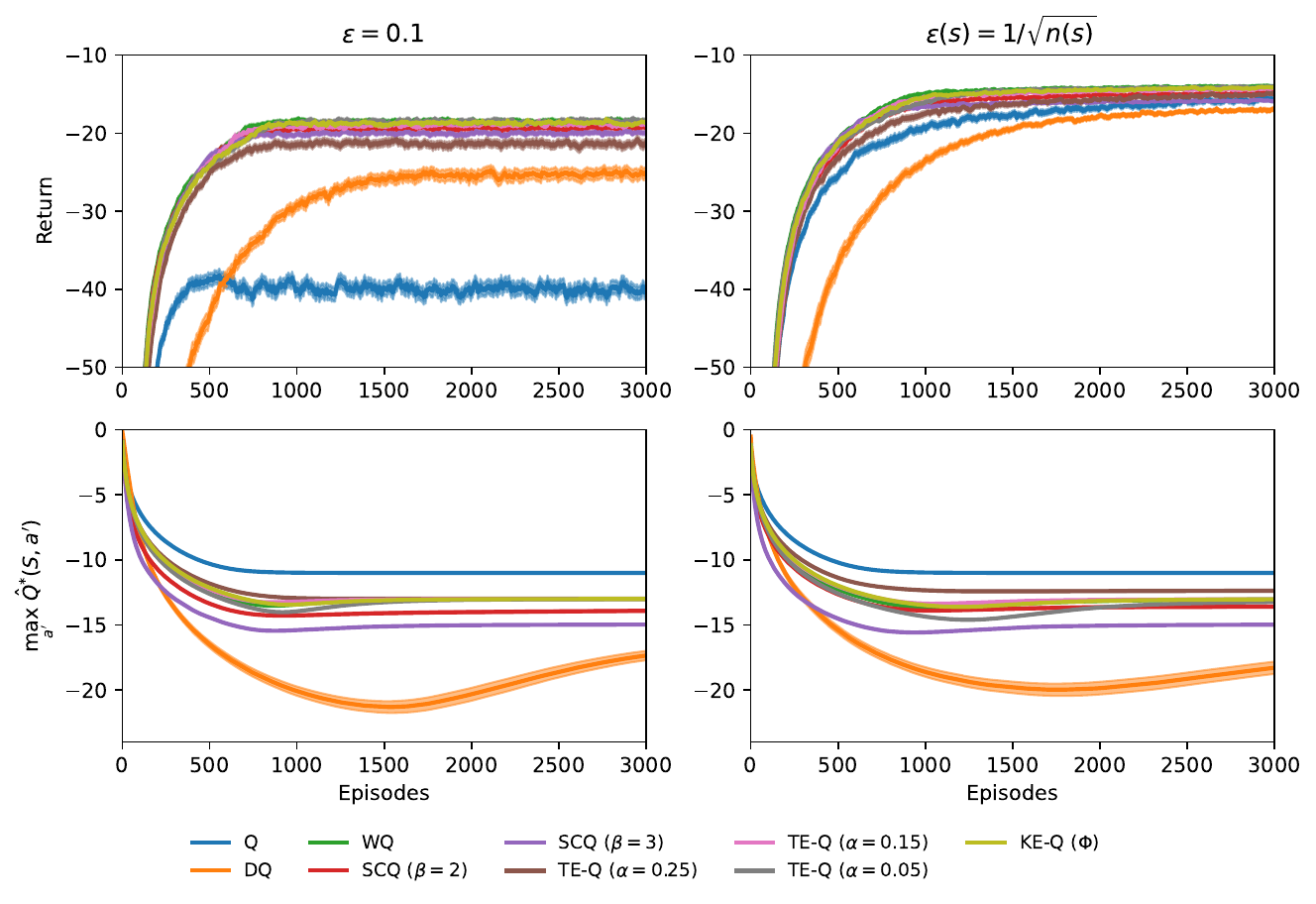}
    \caption{Cliff Walking example from \cite{sutton2018reinforcement} with parameters ${\gamma = 1}$, $\tau\rdd{(s,a)} = 0.1(100 + 1) / \left[100 + n(s,a)\right]$, and two different $\epsilon$-greedy strategies. Results are averaged over 500 runs, exponentially smoothed for visualization purposes, and 95\% confidence intervals are included. The maximum action-value estimate of the start state for DQ is computed by averaging this quantity over both $Q$-tables.}
    \label{fig:CW_I_II}
\end{figure}

Figure \ref{fig:CW_I_II} follows the setup of \cite{zhu2021self} and contains results for constant $\epsilon = 0.1$ and annealing $\epsilon\rdd{(s)} = 1 / \sqrt{n(s)}$, where $n(s)$ is the number of times state $s$ has been visited. The learning rate is $\tau\rdd{(s,a)} = 0.1(100 + 1) / \left[100 + n(s,a)\right]$, with $n(s,a)$ being the number of updates for the state-action pair. Next to $Q$- and Double $Q$-Learning, we consider Weighted $Q$-Learning (WQ, \citealt{d2016estimating}) and Self-correcting $Q$-Learning (SCQ, \citealt{zhu2021self}) with $\beta \in \{2,3\}$, following the recommendation of the authors. \rdd{Crucially, the SCQ is a recent modification of $Q$-Learning that leverages a self-correcting mechanism. The idea is to combine subsequent, correlated action-value estimates into a single self-correcting estimator without increasing the computational burden of the procedure. Due to its solid theoretical foundation and competitive empirical results, we include the SCQ as a benchmark algorithm. Its deep version will also serve as a baseline for the experiments with function approximation in Section \ref{subsec:MinAtar}.}

\rdd{In the Cliff Walking task, we} run each algorithm for 3000 episodes and average the results over 500 independent runs. Additionally, we display the maximum action-value estimate of the start state S over training. For comparison, since at least eleven steps are necessary to walk across our map, it holds for the optimal policy: $\max_{a'} Q^{*}(S,a') = -11$.

We see the strong performance of the newly proposed algorithms for both exploration strategies. Similar to the example in Section \ref{subsec:Simple_MDP}, especially TE-$Q$ with $\alpha = 0.05$ and KE-$Q$ are appropriate for this task and achieve the highest returns together with WQ. Furthermore, the higher action-value estimates for $Q$-Learning are apparent, while Double $Q$-Learning leads to severe underestimation. Finally, the returns are higher for all algorithms with $\epsilon\rdd{(s)} = 1 / \sqrt{n(s)}$ than with a constant exploration rate, which is reasonable since action selection yields a higher probability of selecting greedy in the long-term.

\begin{figure}[ht]
    \centering
    \includegraphics[width=\linewidth]{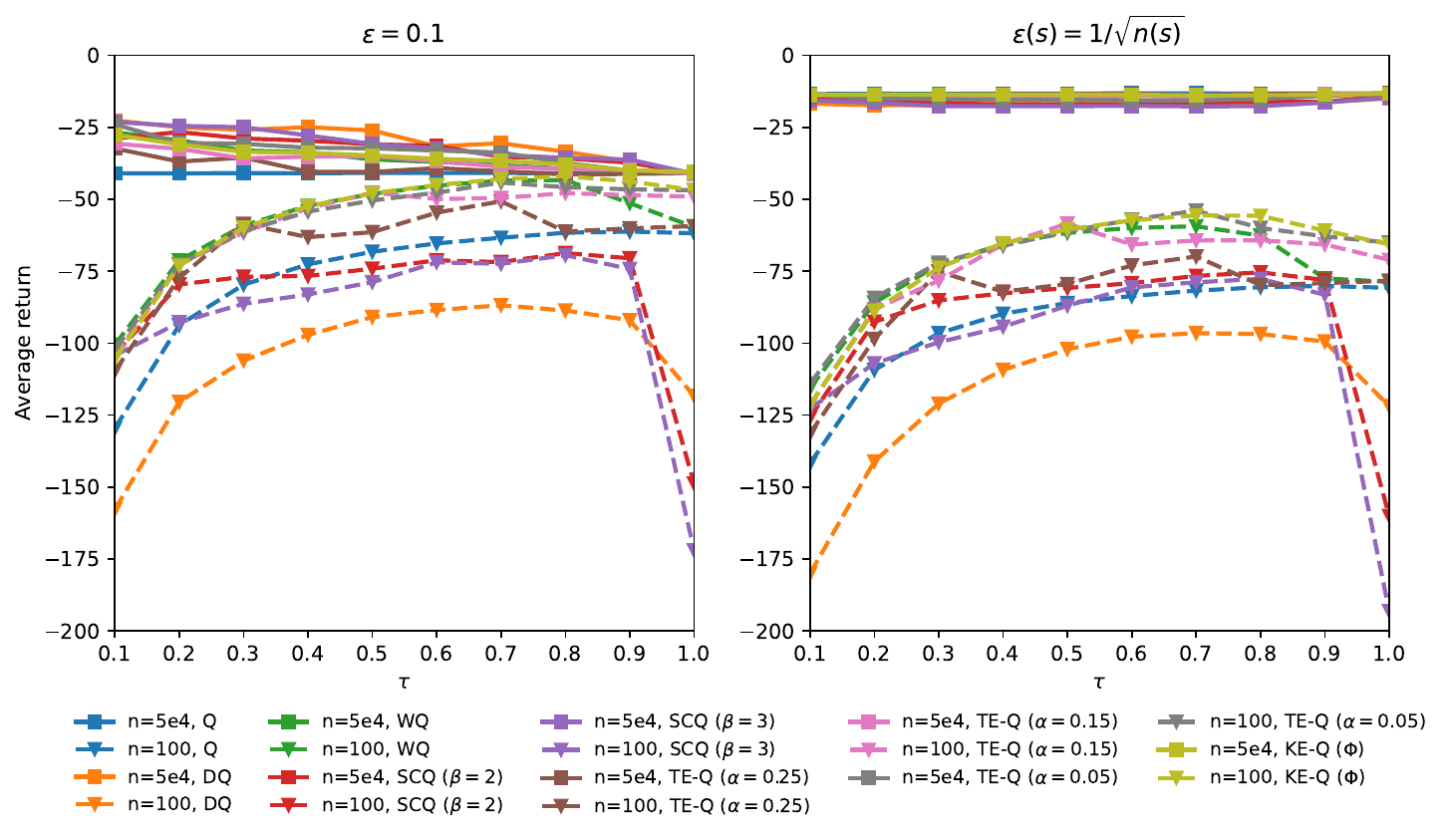}
    \caption{Cliff Walking example adapted from \cite{van2009theoretical}. The algorithms' interim (dotted lines) and asymptotic (solid lines) return averages are analyzed for different learning rates. The number of episodes is $n$.}
    \label{fig:CW_III_IV}
\end{figure}

To investigate the interim and asymptotic behavior of the algorithms for different learning rates, we run an analysis similar to \cite{van2009theoretical}\rd{, and depict the results in Figure \ref{fig:CW_III_IV}.} Precisely, we consider learning rates in $\left\{0.1, 0.2, \ldots, 0.9, 1.0\right\}$ and employ the two exploration strategies ${\epsilon = 0.1}$ and $\epsilon\rdd{(s)} = 1 / \sqrt{n(s)}$ again. We analyze the average return over the first 100 episodes and average the results over 5000 runs for the interim performance. \rd{For the asymptotic scenario, we run each algorithm for $50\,000$ episodes, compute the average return of the episodes, and average the results over 5 runs.}

The TE-$Q$ and KE-$Q$ algorithms offer the most robust interim progress across learning rates for both exploration strategies, while the DQ and SCQ expose a severe performance drop when $\tau = 1$. This might be because both algorithms rely on two different $Q$-tables, and complete replacement of the entries yields instabilities in this case. Regarding the asymptotic analysis, the SCQ and DQ algorithms improve on $Q$-Learning and are marginally above WQ, TE-$Q$, and KE-$Q$ for $\epsilon = 0.1$, while the return differences are close to zero for the annealing exploration strategy due to long-term greedy action selection.

\subsection{MinAtar}\label{subsec:MinAtar}
We select the MinAtar \citep{young2019minatar} environments to test the proposed Deep RL algorithms. MinAtar is a testbed incorporating several Atari games from the Arcade Learning Environment \citep{bellemare2013arcade}, which is considered a challenging benchmark for modern AI algorithms for sequential decision making. MinAtar is based on a reduced state-representation, incorporates sticky actions \citep{machado2018revisiting}, and is designed to enable thorough algorithmic comparisons due to reduced computation times. Following \cite{young2019minatar}, the network structure consists of a convolutional and a fully-connected layer. The remaining hyperparameters match \cite{young2019minatar}, except that we use the Adam \citep{kingma2014adam} optimizer, which led to much more stable results \rd{than using RMSprop \citep{hinton2012neural}} during our experiments. \ref{appendix:exp_hyperparams} contains the full list of specifications. 

\rd{The compared algorithms are the DQN \citep{mnih2015human}, DDQN \citep{van2016deep}, Self-Correcting DQN (SCDQN, \citealt{zhu2021self}), MaxMin DQN \citep{lan2020maxmin}, BDQN \citep{osband2016deep}, TE-BDQN, KE-BDQN (with standard Gaussian cdf), and Ada-TE-BDQN.} For the parametrization of the BDQN and its modifications, we follow \cite{osband2016deep} by using $K = 10$ bootstrap heads, each corresponding to one fully-connected layer, and setting $p = 1$ for the masking Bernoulli distribution. The BDQN uses the target computation of the DDQN, which we apply consequently. Furthermore, we scale the gradients of the convolutional core part for the bootstrap-based algorithms by $1 / K$, which was also recommended by \cite{osband2016deep}. We consider $\beta \in \{2,3,4\}$ for the SCDQN, \rd{the number of networks $N \in \{2,3\}$ for the MaxMin DQN,} and $\alpha \in \{0.1, 0.2, 0.3, 0.4\}$ for the TE-BDQN. The bias parameter of the Ada-TE-BDQN is initialized with $\alpha = 0.25$ and we consider two step sizes $\tau_{\rm Ada} \in \{10^{-4}, 10^{-5}\}$ with horizon $T_{\rm Ada} = 32$.

\begin{figure}[htbp]
    \centering
    \includegraphics[width=\textwidth]{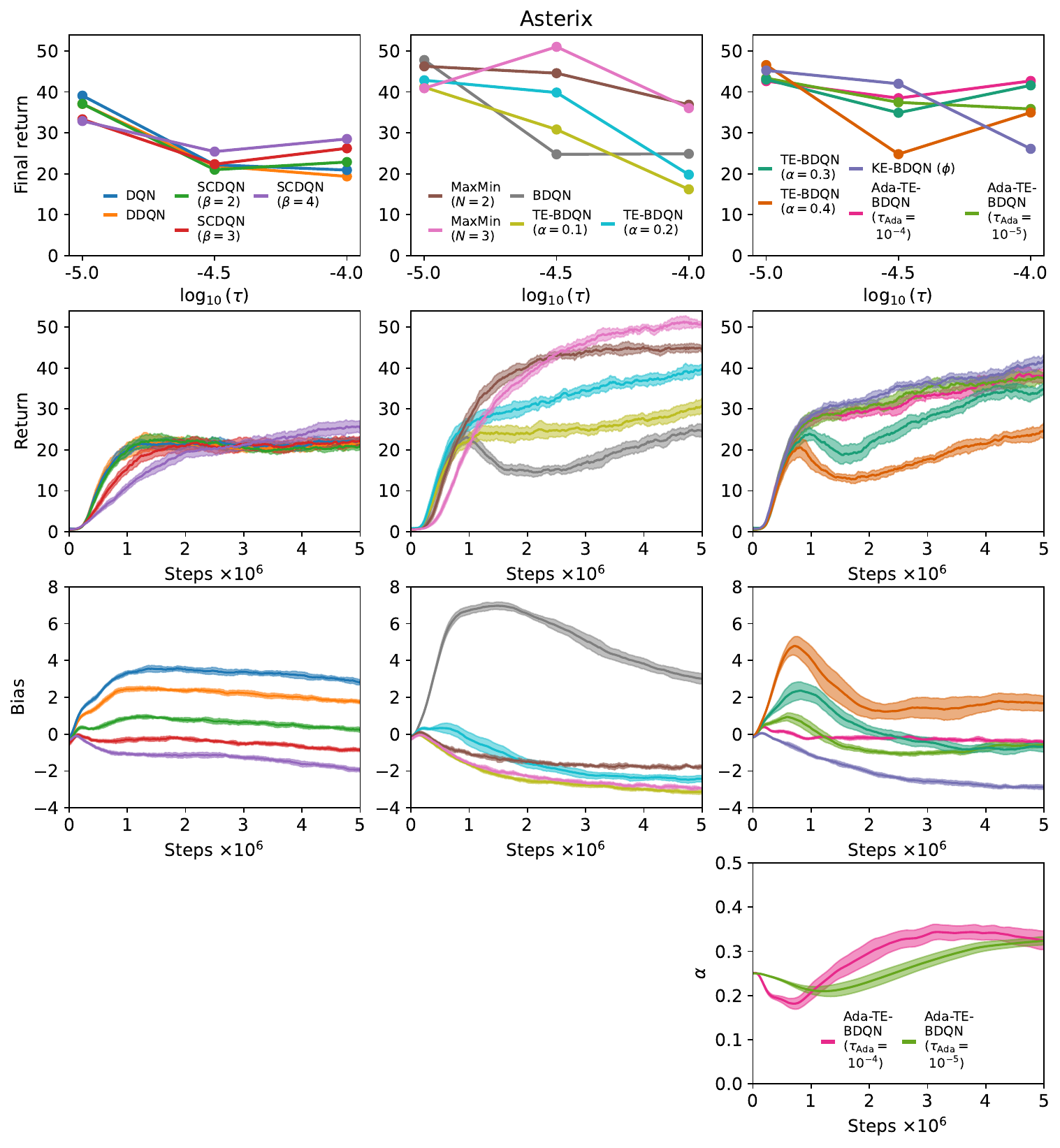}
    \caption{Algorithm comparison on Asterix. \rd{The first row shows the final return of different learning rates. The second and third row show the return and bias over time for $\tau = 10^{-4.5}$. The dynamic behavior of the $\alpha$ for the Ada-TE-BDQN is displayed in row four. Regarding algorithms, the left column includes the DQN, DDQN, and SCDQN; the middle column displays the MaxMin DQN, BDQN, and two TE-BDQNs; and the right column contains the remaining TE-BDQNs, the KE-BDQN, and the Ada-TE-BDQN results.}}
    \label{fig:MinAtar_Asterix}
\end{figure}

\begin{figure}[htbp]
    \centering
    \includegraphics[width=\textwidth]{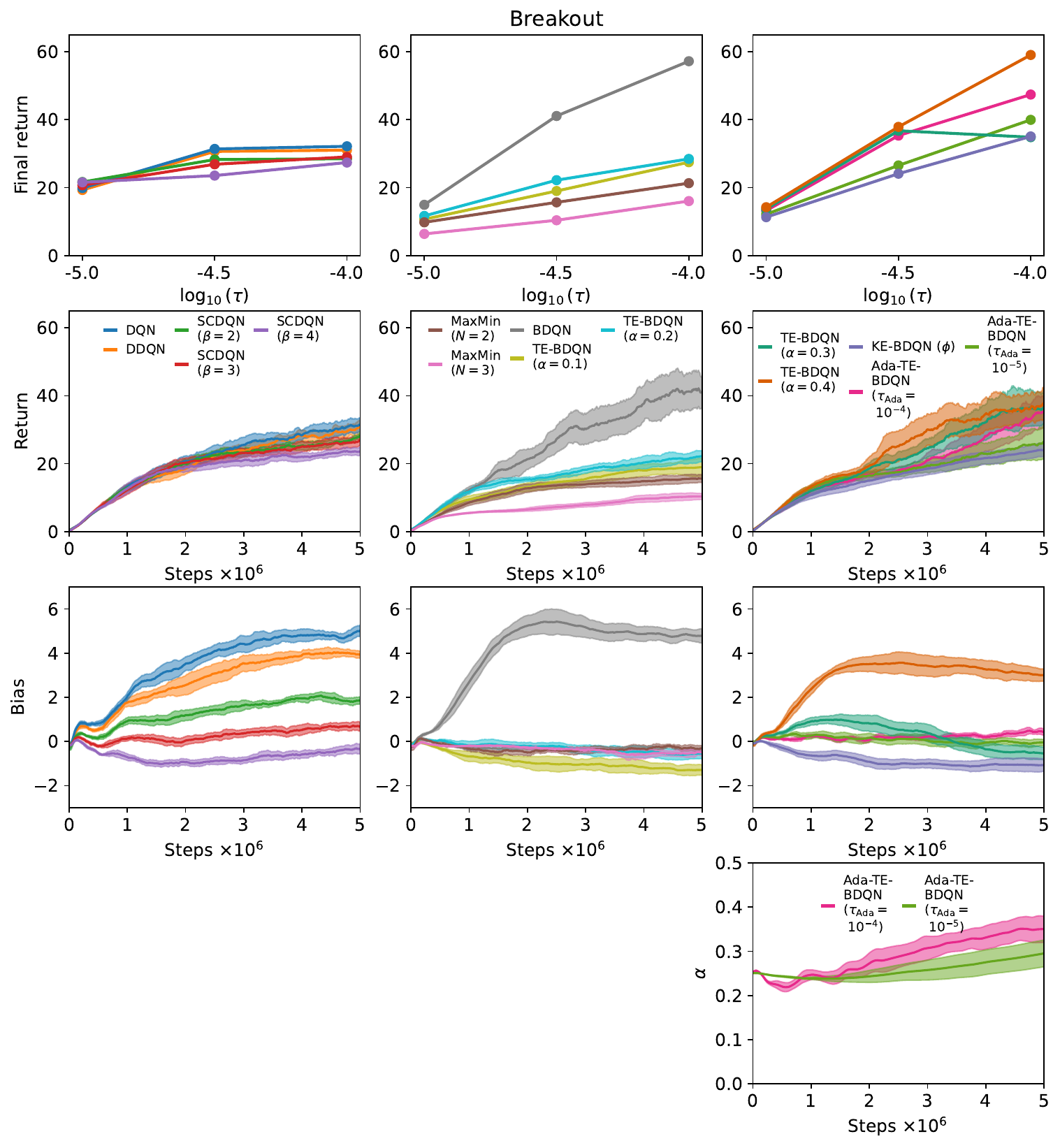}
    \caption{Algorithm comparison on Breakout. \rd{The first row shows the final return of different learning rates. The second and third row show the return and bias over time for $\tau = 10^{-4.5}$. The dynamic behavior of the $\alpha$ for the Ada-TE-BDQN is displayed in row four. Regarding algorithms, the left column includes the DQN, DDQN, and SCDQN; the middle column displays the MaxMin DQN, BDQN, and two TE-BDQNs; and the right column contains the remaining TE-BDQNs, the KE-BDQN, and the Ada-TE-BDQN results.}}
    \label{fig:MinAtar_Breakout}
\end{figure}

\begin{figure}[htbp]
    \centering
    \includegraphics[width=\textwidth]{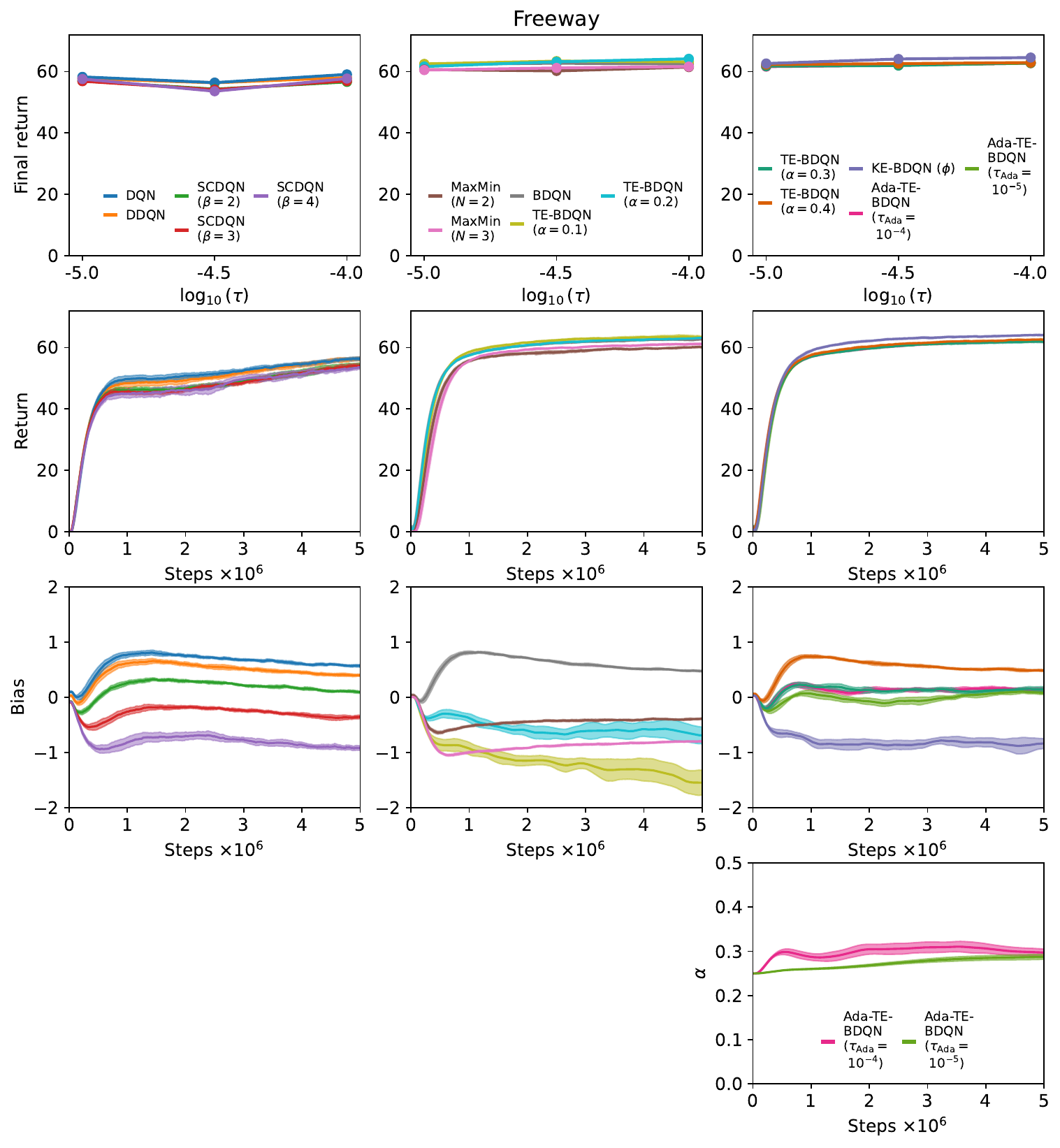}
    \caption{Algorithm comparison on Freeway. \rd{The first row shows the final return of different learning rates. The second and third row show the return and bias over time for $\tau = 10^{-4.5}$. The dynamic behavior of the $\alpha$ for the Ada-TE-BDQN is displayed in row four. Regarding algorithms, the left column includes the DQN, DDQN, and SCDQN; the middle column displays the MaxMin DQN, BDQN, and two TE-BDQNs; and the right column contains the remaining TE-BDQNs, the KE-BDQN, and the Ada-TE-BDQN results.}}
    \label{fig:MinAtar_Freeway}
\end{figure}

\begin{figure}[htbp]
    \centering
    \includegraphics[width=\textwidth]{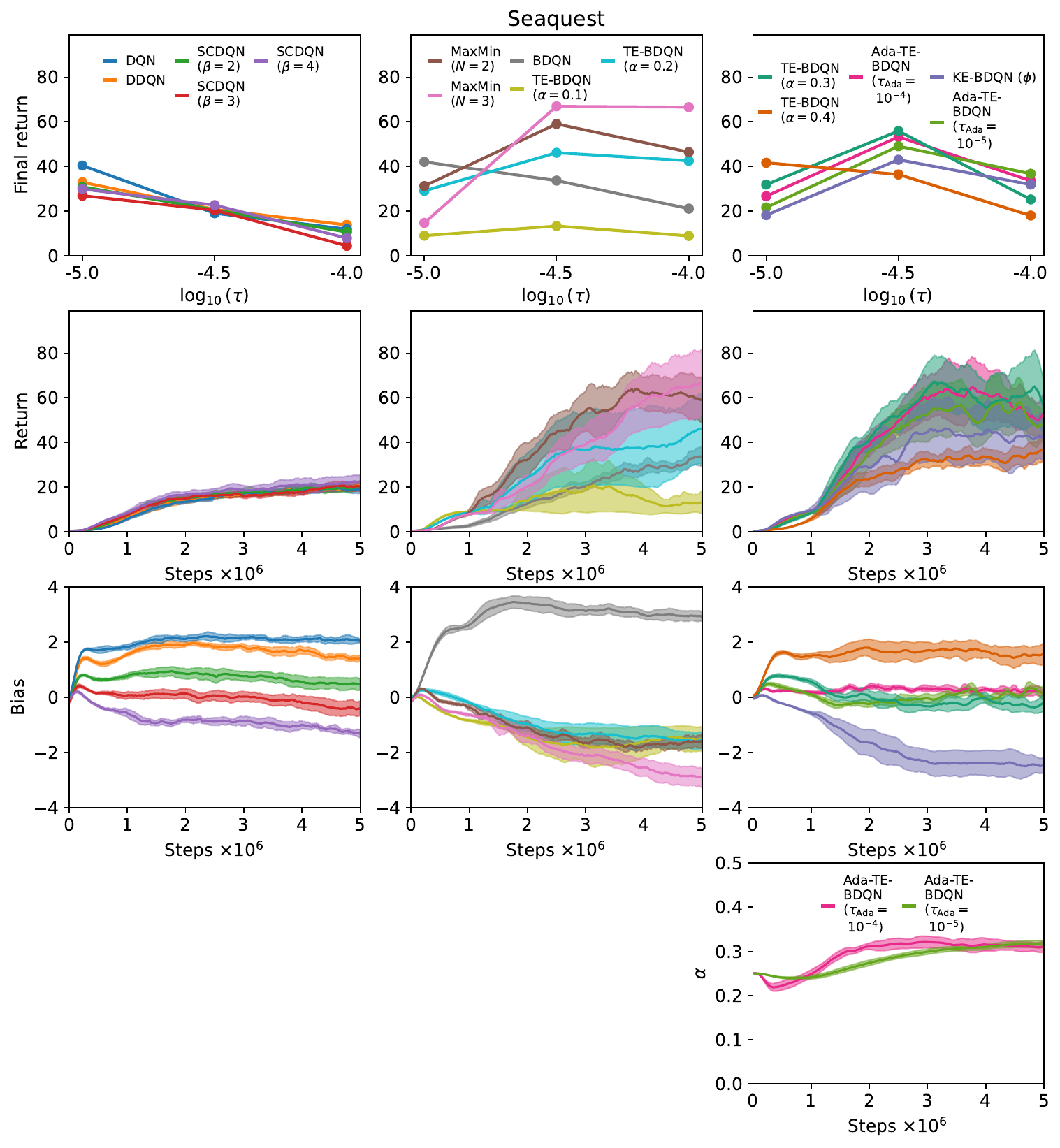}
    \caption{Algorithm comparison on Seaquest. \rd{The first row shows the final return of different learning rates. The second and third row show the return and bias over time for $\tau = 10^{-4.5}$. The dynamic behavior of the $\alpha$ for the Ada-TE-BDQN is displayed in row four. Regarding algorithms, the left column includes the DQN, DDQN, and SCDQN; the middle column displays the MaxMin DQN, BDQN, and two TE-BDQNs; and the right column contains the remaining TE-BDQNs, the KE-BDQN, and the Ada-TE-BDQN results.}}
    \label{fig:MinAtar_Seaquest}
\end{figure}

\begin{figure}[htbp]
    \centering
    \includegraphics[width=\textwidth]{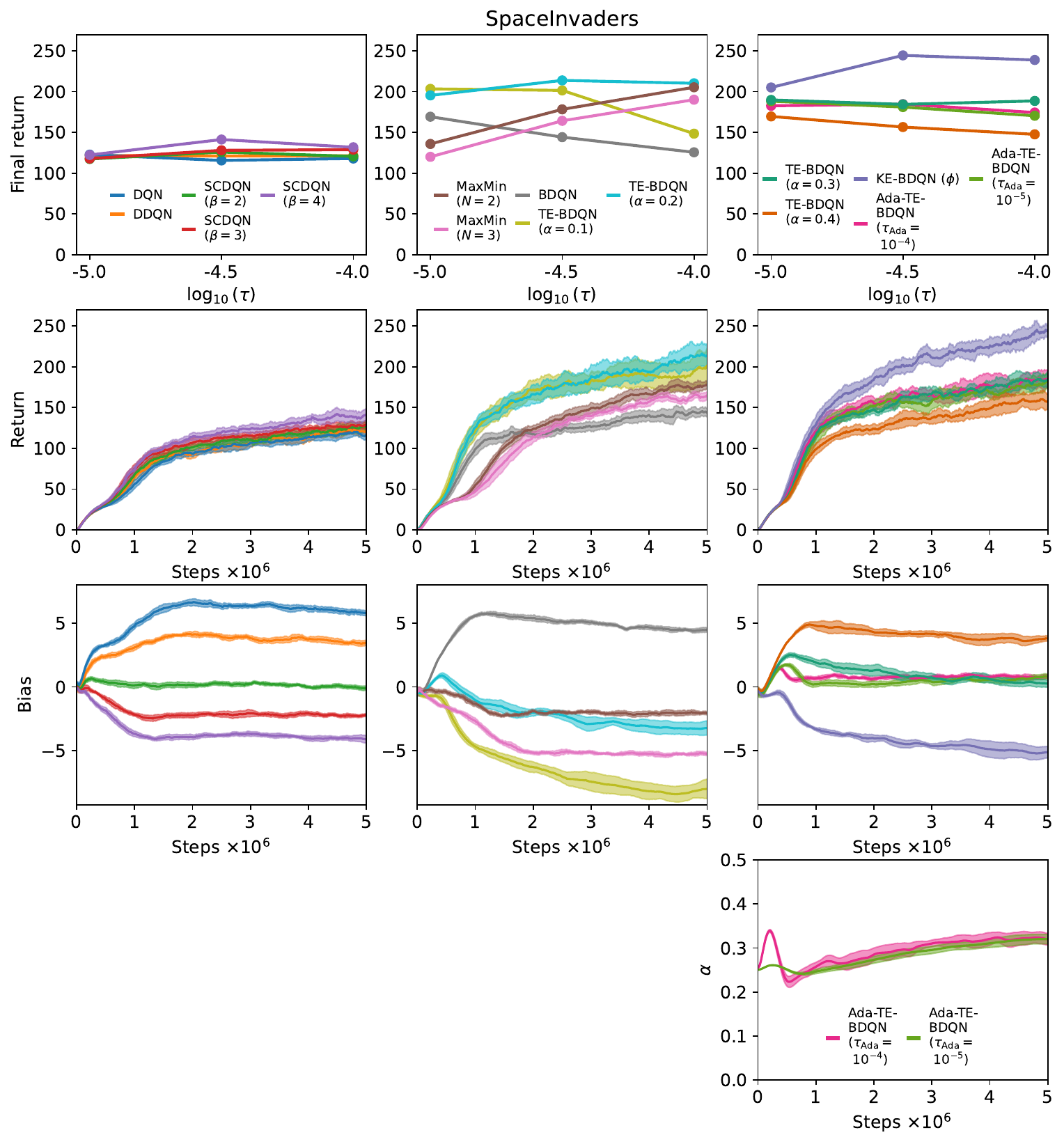}
    \caption{Algorithm comparison on SpaceInvaders. \rd{The first row shows the final return of different learning rates. The second and third row show the return and bias over time for $\tau = 10^{-4.5}$. The dynamic behavior of the $\alpha$ for the Ada-TE-BDQN is displayed in row four. Regarding algorithms, the left column includes the DQN, DDQN, and SCDQN; the middle column displays the MaxMin DQN, BDQN, and two TE-BDQNs; and the right column contains the remaining TE-BDQNs, the KE-BDQN, and the Ada-TE-BDQN results.}}
    \label{fig:MinAtar_SpaceInvaders}
\end{figure}

To check the robustness of the algorithms, we analyze three different learning rates for each environment and algorithm: $\tau \in \{10^{-5}, 10^{-4.5}, 10^{-4}\}$. Every $10\,000$ steps during an experiment, we average the return of 10 test episodes. For the BDQN and its variants, the majority vote of the ensemble is applied. Additionally, we run for all algorithms bias estimation episodes from random initial states sampled from the replay buffer, following Algorithm \ref{algo:Bias_estimation}. The number of those episodes are 10 for the DQN, DDQN, and SCDQN, while we run only 3 episodes for each head of the BDQN-based algorithms due to computation time. We repeat all experiments for ten independent runs, exponentially smooth the results for clarity, and include $95\%$ point-wise confidence intervals over the runs.

Figures \ref{fig:MinAtar_Asterix} - \ref{fig:MinAtar_SpaceInvaders} depict the results. \rd{We show the final return across learning rates in the first row of each plot, while the second and third rows contain the return and bias plots, respectively, for $\tau = 10^{-4.5}$. Lastly, the fourth row displays the dynamic behavior of the $\alpha$ for the Ada-TE-BDQN. \rdd{Please note that the bias and return curves for $\tau \in \{10^{-5}, 10^{-4}\}$ are provided in \ref{appendix:Minatar_further}.} The BDQN-based algorithms are generally comparable to the MaxMin DQN and outperform their competitors DQN, DDQN, and SCDQN. Especially the KE-BDQN and Ada-TE-BDQN show a robust performance across environments, although the algorithms' variances are relatively high in Seaquest. As expected, the DQN is affected by massive overestimations, while the DDQN can reduce the $Q$-estimates in comparison. Although the DE theoretically underestimates the MEV, the DDQN still offers a positive bias in the experiments. This observation is in line with \cite{van2016deep} and might be explained by the fact that the DE, as introduced in Section \ref{subsec:DE}, would require two independent networks. However, the DDQN, as commonly implemented, uses the main and target networks for action evaluation and selection, respectively, and these two networks are time-delayed copies of each other and thus are not independent.}

As theoretically discussed in prior sections, a larger $\alpha$ in the TE-BDQN yields larger $Q$-estimates and, consequently, a larger estimation bias. The adaptive mechanism of Ada-TE-BDQN, especially for $\tau_{\rm Ada} = 10^{-4}$, results in approximately unbiased action-value estimates. Throughout environments, the $\alpha$ of the Ada-TE-BDQN mostly increases during training but stays on moderate values in approximately $[0.2, 0.4]$. Large values with $\alpha \approx 0.5$ are not achieved even in later stages of training, indicating the criticality of the ME combined with function approximation.

\subsection{Discussion}
The experimental results confirm that we can embed statistical thinking in form of the TE/KE into value-based RL methods to control their estimation bias. In addition, the experiments support the finding of \cite{lan2020maxmin} that unbiased $Q$-estimation does not necessarily translate into the best return performance. For example, in SpaceInvaders\rd{, depicted in Figure \ref{fig:MinAtar_SpaceInvaders}}, the KE-BDQN is return-wise the strongest algorithm despite its severe negative bias. However, the TE-BDQN ($\alpha = 0.1$) offers an even lower bias but cannot match the return of the KE-BDQN. There seems to be a critical level - or path over time - of estimation bias for a given MDP, which yields maximum performance. Careful selection of a bias control parameter like $\alpha$ for the TE or the kernel function for the KE thus constitutes a crucial component in designing temporal-difference algorithms.

Analyzing the behavior of the algorithms in more detail, we see that the estimated bias \emph{changes over time}. With random initialization of the networks and a couple of zero-return episodes, all algorithms' bias is shortly approximately zero. As soon as some non-zero rewards are observed, the different target specifications affect the update routine and result in severely different bias plots over time. Besides the Ada-TE-BDQN, each algorithm reveals its tendency towards over- or underestimation, although exceptions are possible. For example, the TE-BDQN ($\alpha = 0.3$) offers slight overestimations during the first three million steps in Breakout\rd{, as shown in Figure \ref{fig:MinAtar_Breakout},} before shifting towards underestimation. Importantly, \rdd{\textbf{none}} of the non-adaptive algorithms shows reliable convergence to zero-bias as training proceeds, which agrees with the observations of \cite{van2016deep}. Finally, we summarize the core findings of our investigation:

\begin{enumerate}
    \item \emph{Absolute bias minimization does not equal return maximization.} In order to maximize performance in a real application, different bias control configurations should be considered, for which the TE/KE-BDQN build a flexible framework.
    \item \emph{Approximately unbiased estimation offers a robust baseline across tasks.} Although it is not always the return-maximizing choice, using a scheme for approximately unbiased value-estimation appears more robust across tasks than fixing a particular bias control parameter. The Ada-TE-BDQN is a powerful candidate for such a scheme \rd{since it almost achieves zero bias during the MinAtar experiments.}
    \item \emph{The compatibility between bias control algorithms and exploration schemes requires systematic analysis.} The impact on exploration most likely constitutes an essential factor in the occasional return-improvement of a biased procedure over an unbiased one \citep{liang2021temporal}. Further study needs to generate insights on how these components interact and, crucially, whether the algorithmic approaches are compatible. Can we use the Ada-TE-BDQN with a modified exploration scheme to boost return performance? Do we maintain unbiased action-value estimation in this process? Can we achieve or even improve on the return peaks of a fine-tuned bias control configuration in this manner? Exploration and action-value estimation are even in off-policy RL not necessarily orthogonal and thus constitute a crucial path for future research. \rdd{We emphasize that similar considerations are prevalent in the bandit literature \citep{lattimore2020bandit, slivkins2019introduction}, whose insights might be leveraged to further analyze the interplay of exploration and action-value estimation.}
\end{enumerate}

\section{Related Works}\label{sec:related_work}
Next to \cite{hasselt2010double}, \cite{d2016estimating}, \rd{\cite{lan2020maxmin}}, and \cite{zhu2021self}, several proposals have been made to further tackle the issue of estimation bias in temporal-difference algorithms. \cite{zhang2017weighted} proposed a hybrid between the ME and the DE called Weighted Double Estimator. It relies on a hyperparameter on the positive real axis, for which the authors propose a heuristic based on empirical experiences. \cite{lee2013bias} proposed Bias-corrected $Q$-Learning, which incorporates a correction term depending on the reward variance. \rd{The Randomized Ensembled Double $Q$-Learning (REDQ, \citealt{chen2021randomized}) is an extension of MaxMin $Q$-Learning \citep{lan2020maxmin} and applies a minimization operator over a subset of the ensemble.} The Action-Candidate based Clipped Double Estimator \citep{jiang2021action} extends the DE by creating a so-called candidate set of indices of which the maximizing one will be picked. Furthermore, the Clipped Double Estimator of \cite{fujimoto2018addressing} and the Truncated Quantile Critic (TQC, \citealt{kuznetsov2020controlling}) algorithm are relevant contributions to addressing the overestimation issue in actor-critic frameworks. Finally, \cite{lee2021sunrise} \rd{pursued} a re-weighting strategy of sampled transitions based on uncertainty estimates from an ensemble. Apart from methodological extensions, \cite{chen2021investigation} recently reported that a lower learning rate or an adequate schedule could also avoid the massive overestimations of $Q$-Learning. However, lowering the learning rate can come at the expense of impractically slow learning, as seen in our Breakout experiments \rd{in Figure \ref{fig:MinAtar_Breakout}}, and constitutes thus not a practical option to address the issue of action-value estimation.

Recently, some proposals have been made to minimize the estimation bias of temporal-difference algorithms through online parameter adjustments in the spirit of the Ada-TE-BDQN. \cite{liang2021temporal} expand the work of \cite{fox2016taming}, \cite{fox2019toward} by using an ensemble to adjust the temperature parameter in a maximum entropy framework. \cite{kuznetsov2021automating} and \cite{dorka2021adaptively} introduce adaptive variants of the TQC by adjusting the number of quantiles to drop based on recent near on-policy trajectories. Finally, \cite{wang2021adaptive} generalize MaxMin $Q$-Learning and REDQ by changing the size of the subset of the ensemble on which the minimization operator is performed. The metric driving the adjustment is the ensemble's function approximation error since it is argued that high approximation error is connected with the overestimation of action-values.

\section{Conclusion}\label{sec:conclusion}
Reinforcement learning as a domain of artificial intelligence has made significant breakthroughs in a diverse set of real-world applications, particularly in the last decade. A key issue of frequently applied temporal-difference algorithms is the propagation of biased action-value estimates. We address this topic by proposing the $T$-Estimator and the $K$-Estimator for the underlying problem of estimating the maximum expected value of random variables. Both estimators are easy to compute and allow to flexibly interpolate between over- and underestimation bias, leading to promising modifications of $Q$-Learning and the Bootstrapped DQN algorithm. Coupled with the dynamic selection procedure of the significance level of TE, our work constitutes an important step towards unbiased estimation of action-values with function approximation.

In future research, we will analyze the discussed interplay of action-value estimation and exploration. As methodological extensions, we will investigate possibilities to extend the two-sample testing procedures into continuous \rd{action} spaces to modify policy gradient methods because the latter constitute an elementary class of methods in several application domains. Furthermore, next to the considered procedure, there are alternative approaches for uncertainty quantification in the neural network scenario. For example, the regularization technique dropout \citep{srivastava2014dropout} \rd{can} be applied similar to \cite{d2021gaussian} to obtain the required variance estimate for the newly proposed algorithms, and the Bootstrapped DQN \rd{can} be enhanced by adding random prior functions \citep{osband2018randomized}. Finally, our work estimates the \rd{bias for given $Q$-estimates via Algorithm \ref{algo:Bias_estimation}, leading to one scalar bias estimate for the whole state-action distribution. While we analyze how this scalar changes throughout training, we do not delve into a detailed differentiation of how this bias is distributed across the state-action space at a specific point during training. We acknowledge that assessing complex MDPs in this fashion might result in an over-simplification. In particular, this line of investigation might lead to more tailored solutions, and one could consider, \rdd{for example}, an individual significance level of the TE for different regions of the state-action space.}

\section*{Acknowledgments}
We would like to thank Niklas Paulig for fruitful discussions in the early stages of this work. Furthermore, we are grateful to the Center for Information Services and High Performance Computing at TU Dresden for providing its facilities for high throughput calculations. \rd{Moreover, we appreciate the valuable feedback of two anonymous reviewers, which helped to improve this paper thoroughly. Finally, we would like to thank the participants of the Conference on Reinforcement Learning and Decision Making (RLDM) 2022, the German Probability and Statistics Days (GPSD) 2023, and the Conference on Computational Statistics (COMPSTAT) 2023 for their fruitful feedback on this work.}


\bibliographystyle{elsarticle-harv} 
\bibliography{bib}

\newpage
\appendix

\gdef\thesection{Appendix \Alph{section}}
\section{Analytic forms for Section \ref{sec:Analysis_Bias_Var_MSE}}
\label{appendix:analytic_forms}

\gdef\thesection{\Alph{section}}
\subsection{Maximum Estimator}
Consider the following setup: $M=2$, $X_1 \sim \mathcal{N}(\mu_1, \sigma^2)$, $X_2 \sim \mathcal{N}(\mu_2, \sigma^2)$, and given sample sizes $|S_1|$, $|S_2|$, from which follows: $\hat{\mu}_1 \sim \mathcal{N}(\mu_1, \frac{\sigma^2}{|S_1|})$ and $\hat{\mu}_2 \sim \mathcal{N}(\mu_2, \frac{\sigma^2}{|S_2|})$. Regarding the ME, we use the expectation in (\ref{eq:ME_exp}) to compute the bias. However, we can alternatively use the following closed-form solutions \citep{nadarajah2008exact}:
\begin{small}
\begin{align*}
    \operatorname{E}\left[\max(\hat{\mu}_1, \hat{\mu}_2)\right] &= \mu_1 \Phi\left(\frac{\mu_1 -\mu_2}{\theta}\right) + \mu_2 \Phi\left(\frac{\mu_2 -\mu_1}{\theta}\right) + \theta \phi\left(\frac{\mu_1 -\mu_2}{\theta}\right),\\
    \operatorname{E}\left\{\left[\max(\hat{\mu}_1, \hat{\mu}_2)\right]^2\right\} &= \left(\frac{\sigma^2}{|S_1|} + \mu_{1}^{2}\right) \Phi\left(\frac{\mu_1 -\mu_2}{\theta}\right) + \left(\frac{\sigma^2}{|S_2|} + \mu_{2}^{2}\right) \Phi\left(\frac{\mu_2 -\mu_1}{\theta}\right)\\
    &+ (\mu_1 + \mu_2) \theta\phi\left(\frac{\mu_1 -\mu_2}{\theta}\right),
\end{align*}
\end{small}
where $\phi$ is the standard Gaussian pdf and $\theta = \sqrt{\frac{\sigma^2}{|S_1|} + \frac{\sigma^2}{|S_2|}}$. The expectation of the squared ME can be used to compute the variance:
\begin{equation*}
    \operatorname{Var}\left[\max(\hat{\mu}_1, \hat{\mu}_2)\right] = \operatorname{E}\left\{\left[\max(\hat{\mu}_1, \hat{\mu}_2)\right]^2\right\} - \operatorname{E}\left[\max(\hat{\mu}_1, \hat{\mu}_2)\right]^2.
\end{equation*}

\subsection{Double Estimator}
The expectation of the DE is given in (\ref{eq:DE_exp}), which directly yields the bias. As already mentioned, what we refer to as the DE throughout the paper is actually the CVE whenever possible, thus we compute the variance of the latter for this example. For notation, we use \rd{$\hat{\mu}_{i}^{A} = \hat{\mu}_{i}(S_{i}^{A})$} and $\hat{f}_{i}^{A}$, $\hat{F}_{i}^{A}$ for the pdf and cdf of $\hat{\mu}_{i}^{A}$, respectively, and similar\rdd{ly} for $S^B$. We assume that the sample $S$ is split evenly between \rd{$S^A$ and $S^B$}, so that the theoretical mean distribution $\hat{f}_{i}^{A}$ equals $\hat{f}_{i}^{B}$. The DE estimate when index selection is performed on subsample $S^A$ is denoted with $\hat{\mu}^{DE, A}_{*}$, and when selecting based on $S^B$ with $\hat{\mu}^{DE, B}_{*}$. It follows:
\begin{align}\label{eq:Var_DE}
    \operatorname{Var}\left(\hat{\mu}^{\rdd{\textrm{CVE}}}_{*}\right) &= \operatorname{Var}\left(\frac{\hat{\mu}^{DE, A}_{*} + \hat{\mu}^{DE, B}_{*}}{2} \right) \nonumber\\
    &= \frac{1}{4} \operatorname{Var}\left(\hat{\mu}^{DE, A}_{*}\right) + \frac{1}{4} \operatorname{Var}\left(\hat{\mu}^{DE, B}_{*}\right) + \frac{1}{2} \operatorname{Cov}\left(\hat{\mu}^{DE, A}_{*}, \hat{\mu}^{DE, B}_{*}\right) \nonumber\\
    &=  \frac{1}{2} \operatorname{Var}\left(\hat{\mu}^{DE, A}_{*}\right) + \frac{1}{2} \operatorname{Cov}\left(\hat{\mu}^{DE, A}_{*}, \hat{\mu}^{DE, B}_{*}\right),
\end{align}
because $\operatorname{Var}\left(\hat{\mu}^{DE, A}_{*}\right) = \operatorname{Var}\left(\hat{\mu}^{DE, B}_{*}\right)$. Using definition:
\begin{equation}\label{eq:Var_DEA}
    \operatorname{Var}\left(\hat{\mu}^{DE, A}_{*}\right) = \operatorname{E}\left[ \left(\hat{\mu}^{DE, A}_{*}\right)^2 \right] - \operatorname{E}\left[\hat{\mu}^{DE, A}_{*}\right]^2,
\end{equation}
in which:
\begin{small}
\begin{equation*}
     \operatorname{E}\left[ \left(\hat{\mu}^{DE, A}_{*}\right)^2 \right] =  \operatorname{E}\left[ \left(\hat{\mu}^{B}_{1}\right)^2 \right] \int_{-\infty}^{\infty} \hat{f}_{1}^{A}(x) \hat{F}_{2}^{A}(x)dx + \operatorname{E}\left[ \left(\hat{\mu}^{B}_{2}\right)^2 \right] \int_{-\infty}^{\infty} \hat{f}_{2}^{A}(x) \hat{F}_{1}^{A}(x)dx,
\end{equation*}
\end{small}
where we compute: $\operatorname{E}\left[ \left(\hat{\mu}^{B}_{1}\right)^2 \right] = \operatorname{Var}\left(\hat{\mu}^{B}_{1}\right) + \operatorname{E}\left[\hat{\mu}^{B}_{1}\right]^2$; and $\operatorname{E}\left[ \left(\hat{\mu}^{B}_{2}\right)^2 \right]$ analogously, so that (\ref{eq:Var_DEA}) is complete. To compute the covariance in (\ref{eq:Var_DE}), we have:
\begin{equation*}
    \operatorname{Cov}\left(\hat{\mu}^{DE, A}_{*}, \hat{\mu}^{DE, B}_{*}\right) = \operatorname{E}\left[\hat{\mu}^{DE, A}_{*} \hat{\mu}^{DE, B}_{*} \right] - \operatorname{E}\left[\hat{\mu}^{DE, A}_{*} \right] \operatorname{E}\left[\hat{\mu}^{DE, B}_{*} \right],
\end{equation*}
with the expectation of the product being\\

\resizebox{0.97\linewidth}{!}{
\begin{math}
\begin{aligned}
     \hspace{-0.91cm}\operatorname{E}\left[\hat{\mu}^{DE, A}_{*} \hat{\mu}^{DE, B}_{*} \right] &= \operatorname{E}\left\{ \left[\mathcal{I}\left( \hat{\mu}_{1}^{A} > \hat{\mu}_{2}^{A}\right) \hat{\mu}_{1}^{B} +
     \mathcal{I}\left( \hat{\mu}_{1}^{A} \leq \hat{\mu}_{2}^{A}\right) \hat{\mu}_{2}^{B}\right] 
     \left[\mathcal{I}\left( \hat{\mu}_{1}^{B} > \hat{\mu}_{2}^{B}\right) \hat{\mu}_{1}^{A} +
     \mathcal{I}\left( \hat{\mu}_{1}^{B} \leq \hat{\mu}_{2}^{B}\right) \hat{\mu}_{2}^{A}\right]
     \right\}\\
     &= \operatorname{E}\left[\mathcal{I}\left(\hat{\mu}_{1}^{A} > \hat{\mu}_{2}^{A}\right)\hat{\mu}_{1}^{A} \right]
     \operatorname{E}\left[\mathcal{I}\left(\hat{\mu}_{1}^{B} > \hat{\mu}_{2}^{B}\right)\hat{\mu}_{1}^{B} \right] +
     \operatorname{E}\left[\mathcal{I}\left(\hat{\mu}_{1}^{A} > \hat{\mu}_{2}^{A}\right)\hat{\mu}_{2}^{A} \right]
     \operatorname{E}\left[\mathcal{I}\left(\hat{\mu}_{1}^{B} \leq \hat{\mu}_{2}^{B}\right)\hat{\mu}_{1}^{B} \right]\\
     &+ \operatorname{E}\left[\mathcal{I}\left(\hat{\mu}_{1}^{A} \leq \hat{\mu}_{2}^{A}\right)\hat{\mu}_{1}^{A} \right]
     \operatorname{E}\left[\mathcal{I}\left(\hat{\mu}_{1}^{B} > \hat{\mu}_{2}^{B}\right)\hat{\mu}_{2}^{B} \right] +
     \operatorname{E}\left[\mathcal{I}\left(\hat{\mu}_{1}^{A} \leq \hat{\mu}_{2}^{A}\right)\hat{\mu}_{2}^{A} \right]
     \operatorname{E}\left[\mathcal{I}\left(\hat{\mu}_{1}^{B} \leq \hat{\mu}_{2}^{B}\right)\hat{\mu}_{2}^{B} \right].
\end{aligned}
\end{math}
}
\bigskip

This expression is simplified using $\mathcal{I}\left( \hat{\mu}_{1}^{A} > \hat{\mu}_{2}^{A}\right) = 1 - \mathcal{I}\left( \hat{\mu}_{1}^{A} \leq \hat{\mu}_{2}^{A}\right)$ to get
\begin{equation*}
    \operatorname{E}\left[\hat{\mu}^{DE, A}_{*} \hat{\mu}^{DE, B}_{*} \right] = \mu_{1}^{2} + 2 I_1 (\mu_2 - \mu_1) + (I_1 - I_2)^2,
\end{equation*}
where
\vspace{-0.1cm}
\begin{align*}
    I_1 &= \operatorname{E}\left[\mathcal{I}(\hat{\mu}_{1}^{A} \leq \hat{\mu}_{2}^{A}) \hat{\mu}_{1}^{A} \right] = \mu_1 - \int_{-\infty}^{\infty} x \hat{f}_{1}^{A}(x)\hat{F}_{2}^{A}(x) dx,\\
    I_2 &=  \operatorname{E}\left[\mathcal{I}(\hat{\mu}_{1}^{A} \leq \hat{\mu}_{2}^{A}) \hat{\mu}_{2}^{A} \right] = \int_{-\infty}^{\infty} x \hat{f}_{2}^{A}(x)\hat{F}_{1}^{A}(x) dx.
\end{align*}

\subsection{T-Estimator and K-Estimator}\label{app:TE_KE_subapp}
Regarding the expectation of the KE, we have:
\smallskip
\begin{small}
\begin{align*}
    \operatorname{E}&\left[\hat{\mu}^{\rdd{\textrm{KE}}}_*\right] = \operatorname{E}\left[\hat{\mu}^{\rdd{\textrm{KE}}}_* \mathcal{I}(\hat{\mu}_1 > \hat{\mu}_2)\right] + \operatorname{E}\left[\hat{\mu}^{\rdd{\textrm{KE}}}_* \mathcal{I}(\hat{\mu}_1 \leq \hat{\mu}_2)\right]\\
    &=\operatorname{E}\left[ \left\{\left[\sum_{j=1}^{2} \kappa\left( \frac{\hat{\mu}_j - \hat{\mu}_1}{\theta_{1j}} \right)\right]^{-1}\sum_{j=1}^{2} \kappa\left( \frac{\hat{\mu}_j - \hat{\mu}_1}{\theta_{1j}} \right)\hat{\mu}_j\right\} \mathcal{I}\left(\hat{\mu}_1 > \hat{\mu}_2\right) \right]\\
    &+ \operatorname{E}\left[ \left\{\left[\sum_{j=1}^{2} \kappa\left( \frac{\hat{\mu}_j - \hat{\mu}_2}{\theta_{2j}} \right)\right]^{-1}\sum_{j=1}^{2} \kappa\left( \frac{\hat{\mu}_j - \hat{\mu}_2}{\theta_{2j}} \right)\hat{\mu}_j\right\} \mathcal{I}\left(\hat{\mu}_1 \leq \hat{\mu}_2\right) \right]\\
    &= \operatorname{E}\left\{\frac{1}{\kappa(0) + \kappa\left(\frac{\hat{\mu}_2 - \hat{\mu}_1}{\theta_{12}}\right)} \left[ \kappa(0) \hat{\mu}_1 + \kappa\left(\frac{\hat{\mu}_2 - \hat{\mu}_1}{\theta_{12}}\right)\hat{\mu}_2 \right] \mathcal{I}(\hat{\mu}_1 > \hat{\mu}_2)\right\}\\
    &+ \operatorname{E}\left\{\frac{1}{\kappa\left(\frac{\hat{\mu}_1 - \hat{\mu}_2}{\theta_{21}}\right) + \kappa(0)} \left[ \kappa\left(\frac{\hat{\mu}_1 - \hat{\mu}_2}{\theta_{21}}\right)\hat{\mu}_1 + \kappa(0) \hat{\mu}_2 \right] \mathcal{I}(\hat{\mu}_1 \leq \hat{\mu}_2)\right\}\\
    &= \int_{-\infty}^{\infty} \int_{-\infty}^{x_1} \frac{1}{\kappa(0) + \kappa\left( \frac{x_2-x_1}{\theta_{12}} \right)} \left[\kappa(0) x_1 + \kappa\left( \frac{x_2-x_1}{\theta_{12}} \right) x_2 \right] \hat{f}_1(x_1) \hat{f}_2(x_2) dx_2 dx_1\\
    &+ \int_{-\infty}^{\infty} \int_{-\infty}^{x_2} \frac{1}{\kappa\left( \frac{x_1-x_2}{\theta_{21}} \right) + \kappa(0)} \left[ \kappa\left( \frac{x_1-x_2}{\theta_{21}} \right) x_1 + \kappa(0) x_2 \right] \hat{f}_1(x_1) \hat{f}_2(x_2) dx_1 dx_2,
\end{align*}
\end{small}
where $\theta_{ij} = \sqrt{\frac{\sigma^2}{|S_i|} + \frac{\sigma^2}{|S_j|}}$ and $\hat{f}_i$ is the pdf of $\hat{\mu}_i$. For the variance, we can compute $\operatorname{E}\left\{[\hat{\mu}^{\rdd{\textrm{KE}}}_*]^2\right\}$ analogously. Since the TE is a special case of the KE with $\kappa(T) = \mathcal{I}(T \geq z_{\alpha})$, the above formula is also applicable for TE. \rdd{We emphasize that we numerically approximate all integrals when using the formulas of \ref{appendix:analytic_forms}.}

\setcounter{figure}{0}
\gdef\thesection{Appendix \Alph{section}}
\section{\rd{Expectation of the KE for known variances}}
\label{appendix:Exp_TEKE}
\rd{In the following, we detail an expression for the expectation of the KE in case the variances of the underlying random variables are known. Since the KE generalizes the TE, the expression similarly holds for the TE.}

\begin{Corollary}
    \rd{The expectation of the KE is:
\begin{small}
\begin{align*}
    \operatorname{E}[\hat{\mu}^{\rdd{\textnormal{KE}}}_*] &= \sum_{i=1}^{M} \int_{-\infty}^{\infty} \int_{-\infty}^{x_i} \cdots \int_{-\infty}^{x_i} \\
    &\left[ \sum_{j=1}^{M} \kappa\left(\frac{x_j - x_i}{\theta_{ij}}\right) \right]^{-1} \left[ \sum_{j=1}^{M} \kappa\left(\frac{x_j - x_i}{\theta_{ij}}\right) x_j  \right] \left[\prod_{j=1}^{M} \hat{f}_j(x_j) \right] \left[\prod\limits_{\substack{j=1 \\ j\neq i}}^M dx_j \right] dx_i,
\end{align*}
\end{small}
where $\theta_{ij} = \sqrt{\frac{\sigma_i^2}{|S_i|} + \frac{\sigma_j^2}{|S_j|}}$, $\hat{f_j}$ is the pdf of $\hat{\mu}_j$, being asymptotically normal, and $\sigma_i^2$ are known.}
\end{Corollary}

\begin{proof}
    \rd{Follows immediately from generalizing the derivation in Appendix \ref{app:TE_KE_subapp} from the case $M = 2$ to higher dimensions.}
\end{proof}

\setcounter{figure}{0}
\gdef\thesection{Appendix \Alph{section}}
\section{\rd{Optimizing the Example from Section \ref{sec:Analysis_Bias_Var_MSE}}}
\label{appendix:opt_Gaussian}
\rd{The chosen specification of the beta distribution in the example with two Gaussian random variables from \rdd{Figure \ref{fig:MSE_KE} of} Section \ref{sec:Analysis_Bias_Var_MSE} results in a rather strong underestimation for large $\mu_1 - \mu_2$. To find a better fitting parametrization, we numerically solved the optimization problem of minimizing the squared bias for the range of $\mu_1 \in [0,5]$ over the parameters $\mathfrak{a}$, $\mathfrak{b}$ of the $\mathcal{B}_{\mathfrak{a}, \mathfrak{b}}$ kernel. To enable comparability between the estimators, we run the identical optimization for the parameter $\lambda$ of the Gaussian kernel $\Phi_{\lambda}$ (deviating from the unit variance specification) and the level of significance of the TE. The optimized kernel functions alongside specifications from Figures \ref{fig:MSE_TE} and \ref{fig:MSE_KE} are depicted in Figure \ref{fig:Kernel_funcs}, while Figure \ref{fig:MSE_KE_TE_opt} displays the performance of the optimized estimators.}

\rd{The functions in Figure \ref{fig:Kernel_funcs} are normalized to $[0,1]$ by division through $\kappa(0)$ of the respective kernel. Optimizing the standard deviation of the Gaussian cdf yields $\lambda \approx 0.84$, which is close to the unit variance specification. On the other hand, the bias-optimal value for the significance level of the TE is $\approx 0.14$. Regarding the $\mathcal{B}_{\mathfrak{a}, \mathfrak{b}}$ specification, one needs to recall that the beta kernel is capable of approximating both the optimized TE and the optimized KE with the (non-standard) Gaussian cdf kernel. Following Figures \ref{fig:Kernel_funcs} and \ref{fig:MSE_KE_TE_opt}, the optimized TE is favorable in this scenario since the optimized beta cdf is in line with the optimized TE. \rdd{We emphasize that the numerical optimization of the beta CDF yields different solutions depending on the starting values. However, all solutions result in a distribution with zero variance.}}

\begin{figure}[!htb]
    \centering
    \includegraphics[width=0.5\linewidth]{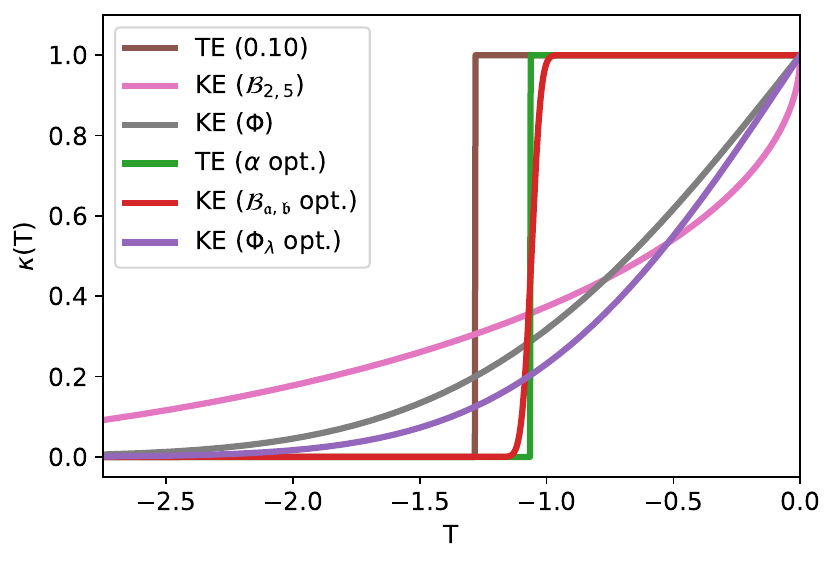}
    \caption{\rd{Original kernel functions and optimized specification for minimizing the squared bias in Figures \ref{fig:MSE_TE} and \ref{fig:MSE_KE}.}}
    \label{fig:Kernel_funcs}
\end{figure}

\begin{figure}[!htb]
\centering
\includegraphics[width=\textwidth]{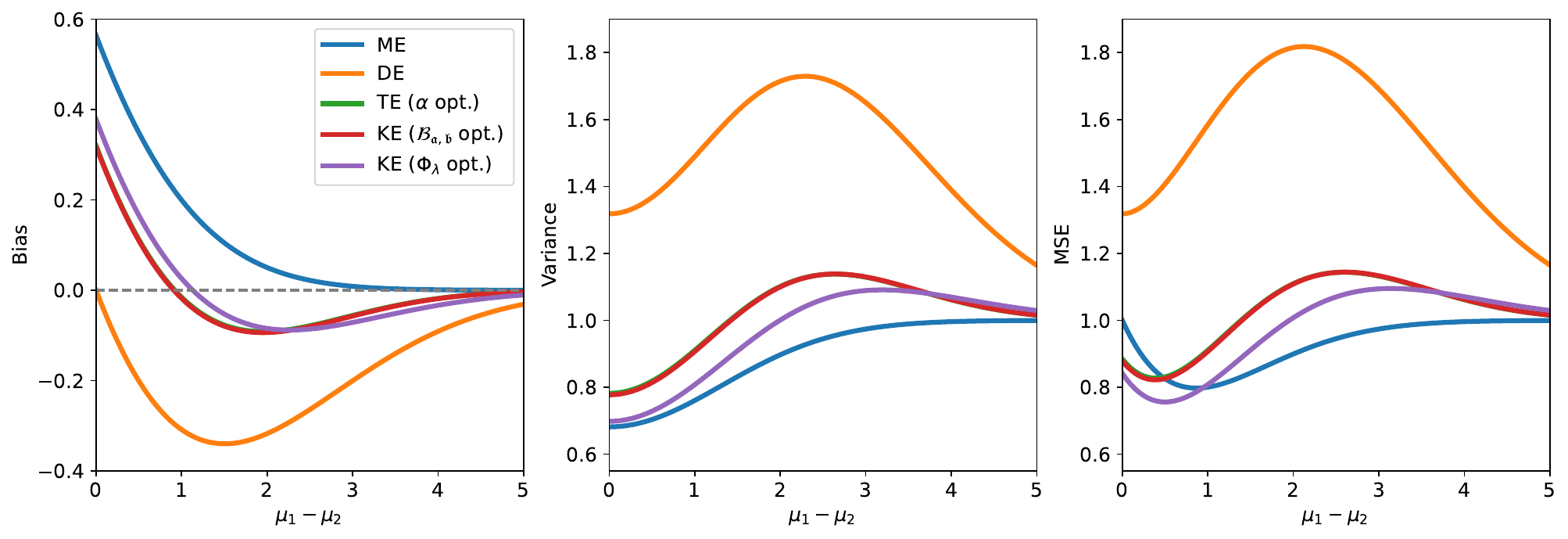}
\caption{\rd{Bias, variance, and MSE for different estimators of the MEV in the case of two Gaussian random variables using the optimized TE and KE's shown in Figure \ref{fig:Kernel_funcs}. The optimized beta kernel opts for a similar solution as the optimized TE.}}
\label{fig:MSE_KE_TE_opt}
\end{figure}

\gdef\thesection{Appendix \Alph{section}}
\section{\rd{Proof of Convergence of TE/KE-$Q$-Learning}}\label{appendix:convergence_proof}
\rd{We sketch the proof of \rdd{Theorem} \ref{pro:KE_Q_convergence} for the discounted case by building on the works of \cite{singh2000convergence}, \cite{hasselt2010double}, and \cite{fujimoto2018addressing}. In particular, we need the following Lemma of \cite{singh2000convergence}:}

\begin{Lemma}\label{lemma:conv_stoch_approx}
    \rd{Consider a stochastic process $(\zeta_t, \Delta_t, F_t)$, $t \geq 0$ where $\zeta_t, \Delta_t, F_t: X \rightarrow \mathbb{R}$ satisfy the equation:
    \begin{equation*}
        \Delta_{t+1}(x_t) = [1- \zeta_t(x_t)]\Delta_t(x_t) + \zeta_t(x_t) F_t(x_t),
    \end{equation*}
    where $x_t \in X$\rdd{, with} $t = 0,1,2,\ldots$\rdd{, and $\bigcup_{t=0}^{\infty} x_t = X$.} Let $P_t$ be a sequence of increasing $\sigma$-fields such that $\zeta_0$ and $\Delta_0$ are $P_0$-measurable and $\zeta_t, \Delta_t,$ and $F_{t-1}$ are $P_t$-measurable, $t=1,2,\ldots$. Assume that the following hold:
    \begin{enumerate}
        \item The set $X$ is finite.
        \item $\zeta_t(x_t) \in [0,1]$, $\sum_t \zeta_t(x_t) = \infty$, $\sum_t[\zeta_t(x_t)]^2 < \infty$ with probability 1 and $\rdd{\forall x \in X \setminus \{x_t\}}: \zeta\rdd{_t}(x) = 0$.
        \item $\rdd{\|} \operatorname{E}[F_t\vert P_t]\rdd{\|_\infty} \leq k \rdd{\|}\Delta_t\rdd{\|_\infty} + c_t$ where $k \in [0,1)$ and $c_t$ converges to 0 with probability 1. \rdd{Here $\rdd{\|}y\rdd{\|_\infty} \rdd{= \max_{i}(\vert y_i\vert)}$ \rdd{for some $y \in \mathbb{R}^n$} denotes the maximum norm.}
        \item $\operatorname{Var}[F_t(x\rdd{_t})\vert P_t] \leq K(1+ k\rdd{\|} \Delta_t \rdd{\|_\infty})^2,$ where $K$ is some constant.
    \end{enumerate}
    Then $\Delta_t$ converges to 0 with probability 1.}
\end{Lemma}

\rd{Moreover, we use the following additional lemmata:}

\begin{Lemma}\label{lemma:n_eff_num}
    \rd{Let $\tau_t$, $t \geq 0$, be a random sequence with $\mathbb{P}(\tau_t \in [0,1]) = 1$ $\forall t$. Define the random process $x_{t+1} = x_t (1- \tau_t) + \tau_{t}$ for $t \geq 1$ and some $x_0 \in [\varepsilon,1]$, where $0 < \varepsilon \leq 1$. Then $\mathbb{P}(x_t \in [\varepsilon, 1]) = 1, \rdd{\forall t \ge 0}$ \rdd{holds.}}
\end{Lemma}
\begin{proof}
    \rd{We prove the lemma via induction. The statement is true per assumption for $x_0$. Now assume the statement holds for some $t \geq 0$: ${\rdd{\mathbb{P}(}x_t \in [\varepsilon, 1]\rdd{) = 1}}$. Then follows: $x_{t+1} = (1-\tau_t) x_t + \tau_t \geq (1-\tau_t) \cdot \varepsilon + \tau_t = \varepsilon + \tau_t(1-\varepsilon) \geq \varepsilon$ \rdd{with probability 1}. Further, we have $x_{t+1} = (1-\tau_t) x_t + \tau_t \leq (1-\tau_t) \cdot 1 + \tau_t = 1$  \rdd{with probability 1}. Combining the bounds, we find that $\rdd{\mathbb{P}(}x_{t+1} \in [\varepsilon, 1]\rdd{) = 1}$, proving the induction step and the lemma.}
\end{proof}

\begin{Lemma}\label{lemma:n_eff_den}
   \rd{Let $\tau_t$, $t \geq 0$, be a random sequence with $\mathbb{P}(\tau_t \in [0,1]) = 1$ $\forall t$\rdd{,} $\mathbb{P}(\sum_t \tau_t = \infty) = 1$\rdd{, and $\mathbb{P}(\sum_t \tau_{t}^2 < \infty) = 1$}. Define the random process $x_{t+1} = x_t (1- \tau_t)^2 + \tau_{t}^2$ for $t \geq 1$ and some $x_0 > 0$. Then holds ${\mathbb{P}(\lim_{t \rightarrow \infty} x_t = 0)}=1$.}
\end{Lemma}
\begin{proof}
\rd{We begin by unfolding the recursively defined process using simple algebra as follows; compare \citet[Chapter 3]{jerri1996}:
\begin{equation}\label{eq:x_t_unfolded}
    x_t = x_0 \prod_{i=0}^{t-1}(1-\tau_i)^2 + \sum_{i=0}^{t-1}\tau_{i}^2 \prod_{j=i+1}^{t-1}(1-\tau_j)^2.
\end{equation}
Furthermore, we notice that $\mathbb{P}(x_t > 0) = 1$ for all $t \geq 0$ due to the conditions on $x_0$ and $\tau_t$. Thus, to prove that $\rdd{\mathbb{P}(}\lim_{t \rightarrow \infty} x_t = 0\rdd{) = 1}$, we need to show that each summand in (\ref{eq:x_t_unfolded}) converges to zero \rdd{with probability 1}. \rdd{At first}, we have to prove:
\begin{equation}\label{eq:lim_prod}
    \rdd{\mathbb{P}\Bigg(}\lim_{t \rightarrow \infty} \left[ \prod_{i=i_0}^{t} (1-\tau_i)^2 \right] = 0\rdd{\Bigg) = 1},
\end{equation}
for any $i_0 \in \mathbb{N}_0$, where $\mathbb{N}_0$ is the set of natural numbers including zero. To see this, we notice \rdd{that with probability 1}:
\begin{equation}
    0 \leq \prod_{i=i_0}^{t} (1-\tau_i)^2 \leq \prod_{i=i_0}^{t} (1-\tau_i) \leq \prod_{i=i_0}^{t} \exp(-\tau_i) = \exp\left(-\sum_{i=i_0}^{t}\tau_i \right),
\end{equation}
because $\exp(x) \geq 1+x$ for all $x \in \mathbb{R}$. It follows \rdd{with probability 1}:
\begin{equation}\label{eq:lim_last}
    0 \leq  \lim_{t \rightarrow \infty} \left[ \prod_{i=i_0}^{t} (1-\tau_i)^2 \right] \leq \lim_{t \rightarrow \infty} \left[\exp\left(-\sum_{i=i_0}^{t}\tau_i \right)\right] = 0,
\end{equation}
which proves (\ref{eq:lim_prod}). Note that the last equality in (\ref{eq:lim_last}) follows from the condition that $\mathbb{P}(\sum_t \tau_t = \infty) = 1$. \rdd{From (\ref{eq:lim_prod}) and the assumption that $\mathbb{P}(\sum_t \tau_{t}^2 < \infty) = 1$, it immediately follows that both summands in (\ref{eq:x_t_unfolded}) converge to zero with probability 1. Consequently,} we deduce ${\mathbb{P}(\lim_{t \rightarrow \infty} x_t = 0) = 1}$.}
\end{proof}

\rd{On this basis, we can prove \rdd{Theorem} \ref{pro:KE_Q_convergence}, which we restate in the following.}
\begin{Theorem}
    \rd{Let the following regularity conditions be fulfilled:
    \begin{enumerate}
    \item The MDP is finite.
    \item $\gamma \in [0,1).$
    \item The learning rates satisfy $\tau_t(s,a) \in [0,1]$, $\sum_t \tau_t(s,a) = \infty$, $\sum_t \tau_t^{2}(s,a) < \infty$ all with probability 1 for all $s \in \mathcal{S}, a \in \mathcal{A}$.
    \item The reward function is bounded.
    \item \rdd{Each state-action pair is visited infinitely often.}
    \end{enumerate}
    Then \rdd{the following} holds for the random sequence of action-value estimates $\hat{Q}^{*}_{t}$ generated by TE/KE-$Q$-Learning:
    $$\mathbb{P}\left[\lim_{t\rightarrow \infty}\hat{Q}^{*}_{t}(s,a)=Q^{*}(s,a)\right]=1 \quad \forall s \in \mathcal{S}, a \in \mathcal{A}.$$}
\end{Theorem}
\begin{Sketch}
\normalfont 
    \rd{The idea is to apply Lemma \ref{lemma:conv_stoch_approx} with $P_t = \{\hat{Q}^{*}_0, s_0, a_0, \tau_0, r_1, s_1, \ldots, s_t, a_t \}$, $X = \mathcal{S} \times \mathcal{A}$, $\Delta_t = \hat{Q}^{*}_t - Q^*$, $\rdd{\zeta_t = \tau_t}$\rdd{, and $F_t(s_t, a_t) = r_t + \gamma \operatorname{KE_a} \hat{Q}^{*}_{t}(s_{t+1}, a) - Q^*(s_t,a_t)$}. We first note that conditions 1 and 2 of Lemma \ref{lemma:conv_stoch_approx} hold by regularity conditions 1 and 3 of KE/TE-$Q$-Learning, respectively. Furthermore, condition 4 of Lemma \ref{lemma:conv_stoch_approx} is guaranteed by assuming a bounded reward function (regularity condition 4); see also the derivations in \cite{barber2023smoothed} for an explanation of this point. Finally, condition 3 of Lemma \ref{lemma:conv_stoch_approx} needs to be shown.}

    \rd{Following a similar route as \cite{fujimoto2018addressing}, we observe:
    \begin{align*}
        \Delta_{t+1}(s_t, a_t) &= \hat{Q}^{*}_{t+1}(s_t, a_t) - Q^*(s_t, a_t)\\
    &= \rdd{\hat{Q}^{*}_{t}(s_t, a_t)  +\tau_t(s_t,a_t)\left[r_t + \gamma \operatorname{KE_a} \hat{Q}^{*}_{t}(s_{t+1}, a) - \hat{Q}^{*}_{t}(s_t, a_t)\right]}\\
    &- \rdd{Q^*(s_t,a_t) + \tau_t(s_t,a_t)Q^*(s_t,a_t) - \tau_t(s_t,a_t)Q^*(s_t,a_t)}\\
        &= [1-\tau_t(s_t,a_t)] \left[\hat{Q}^{*}_{t}(s_t, a_t) - Q^*(s_t, a_t)\right]\\
        &+ \tau_t(s_t,a_t)\left[r_t + \gamma \operatorname{KE_a} \hat{Q}^{*}_{t}(s_{t+1}, a) - Q^*(s_t,a_t) \right]\\
        &= [1-\tau_t(s_t,a_t)]\Delta_{t}(s_t, a_t) + \tau_t(s_t,a_t) F_t(s_t,a_t),
    \end{align*}
    where $F_t(s_t, a_t) = r_t + \gamma \operatorname{KE_a} \hat{Q}^{*}_{t}(s_{t+1}, a) - Q^*(s_t,a_t) = F_{t}^{Q}(s_t,a_t) + c_t$ if we set $F_{t}^{Q}(s_t,a_t) = r_t + \gamma \max_a \hat{Q}^{*}_{t}(s_{t+1}, a) - Q^*(s_t,a_t)$ and $c_t = \gamma \operatorname{KE_a} \hat{Q}^{*}_{t}(s_{t+1}, a) - \gamma \max_a \hat{Q}^{*}_{t}(s_{t+1}, a)$. We know from standard $Q$-Learning that $\rdd{\|}\operatorname{E}\left[F_{t}^{Q} \vert P_t \right]\rdd{\|_\infty} \leq \gamma \rdd{\|}\Delta_t \rdd{\|_\infty}$. Hence, to show the remaining condition 3 of Lemma \ref{lemma:conv_stoch_approx} and thus prove \rdd{Theorem} \ref{pro:KE_Q_convergence}, we need to show that
    $$\lim_{t \rightarrow \infty} c_t = \lim_{t\rightarrow \infty} \left[\gamma \operatorname{KE_a} \hat{Q}^{*}_{t}(s_{t+1}, a) - \gamma \max_a \hat{Q}^{*}_{t}(s_{t+1}, a)\right] = 0$$
    with probability 1. For convenience, we recall the definition of the KE-operator:
    \begin{align}\label{eq:KE_a}
    \operatorname{KE_a} \hat{Q}^{*}_{t}(s_{t+1}, a) &= \left\{\sum_{a\in \mathcal{A}} \kappa\left[T_{\hat{Q}^{*}_t}(s_{t+1},a)\right]\right\}^{-1} \sum_{a'\in \mathcal{A}} \kappa\left[T_{\hat{Q}^{*}_t}(s_{t+1},a)\right] \hat{Q}^{*}_t(s_{t+1},a),\\
    T_{\hat{Q}^{*}_t}(s_{t+1},a) &= \frac{\hat{Q}^{*}_t(s_{t+1},a) - \max_{a' \in \mathcal{A}}\hat{Q}^{*}_t(s_{t+1},a')}{\sqrt{\widehat{\operatorname{Var}}_t\left[\hat{Q}^{*}_t(s_{t+1},a)\right] + \widehat{\operatorname{Var}}_t\left[\hat{Q}^{*}_t(s_{t+1}, a^*) \right]}},\label{eq:KE_a2}
\end{align}
for a maximizing action $a^* \in \{a \in \mathcal{A} \mid \hat{Q}^{*}_t(s_{t+1}, a) = \max_{a' \in \mathcal{A}}\hat{Q}^{*}_t(s_{t+1},a')\}$. To ensure that $\operatorname{KE_a} \hat{Q}^{*}_{t}(s_{t+1}, a)$ converges to $\max_a \hat{Q}^{*}_{t}(s_{t+1}, a)$ \rdd{with probability 1}, we require that the weights of all non-maximizing actions during the summation in (\ref{eq:KE_a}) converge to \rdd{zero with probability 1}. Since we imposed the condition on the kernel function $\kappa(\cdot)$ that $\lim_{x\rightarrow -\infty} \kappa(x) = 0$, we have:
$$\rdd{\mathbb{P}\Big(}\lim_{t\rightarrow \infty} \kappa\left[T_{\hat{Q}^{*}_t}(s_{t+1},a)\right] = 0\rdd{\Big) = 1} \iff \rdd{\mathbb{P}\Big(}\lim_{t\rightarrow \infty} T_{\hat{Q}^{*}_t}(s_{t+1},a) = -\infty\rdd{\Big) = 1},$$
for \rdd{all $a \neq a^*$}. To see whether the test statistics for these actions get arbitrarily small, we need to closely investigate (\ref{eq:KE_a2}). We note that the numerator of (\ref{eq:KE_a2}) is bounded from above by \rdd{zero} and bounded from below by some real number $C_1 \leq 0$ \rdd{with probability 1} due to assumption 4. Thus, to ensure $\rdd{\mathbb{P}\Big(}\lim_{t\rightarrow \infty} T_{\hat{Q}^{*}_t}(s_{t+1},a) = -\infty\rdd{\Big) = 1}$, we require $\rdd{\mathbb{P}\Big(}\lim_{t\rightarrow \infty} \widehat{\operatorname{Var}}_t\left[\hat{Q}^{*}_t(s_{t+1},\rdd{a'})\right] = 0\rdd{\Big) = 1}$ \rdd{$\forall a' \in \mathcal{A}$}. In a classical statistical setup, where the variance of a mean estimate converges to 0 with increasing sample size, the proof would already be finished. However, since we also perform incremental updates of the variance, we need to analyze the limit behaviour of:
\begin{equation}\label{eq:var_est}
    \widehat{\operatorname{Var}}_t\left[\hat{Q}^{*}_t(s_{t+1},a)\right] = \frac{ \widehat{\sigma^{2}}_{\text{pro},t}(s_{t+1}, a)}{n_{\text{eff},t}(s_{t+1},a)},
\end{equation}
\rdd{where $a \in \mathcal{A}$.} The numerator of (\ref{eq:var_est}) is lower-bounded by \rdd{zero} and upper-bounded by some real number $C_2 \geq 0$ \rdd{with probability 1}, again due to the boundedness assumption on the rewards. Therefore, we have: 
$$\rdd{\mathbb{P}\Big(}\lim_{t\rightarrow \infty} \widehat{\operatorname{Var}}_t\left[\hat{Q}^{*}_t(s_{t+1},a)\right] = 0\rdd{\Big) = 1} \iff \rdd{\mathbb{P}\Big(}\lim_{t\rightarrow \infty} n_{\text{eff},t}(s_{t+1},a) = \infty\rdd{\Big) = 1}.$$
The effective sample size is computed via $n_{\text{eff},t}(s_{t+1},a) = \frac{\left[\omega_t(s_{t+1},a)\right]^2}{\omega_{t}^2(s_{t+1},a)}$, where numerator and denominator are updated for a visited state-action pair $(s_t, a_t)$ as follows:
\begin{align}
\omega_{t+1}(s_{t},a_t) &\leftarrow [1-\tau_t(s_t,a_t)]\omega_t(s_{t},a_t) + \tau_t(s_t,a_t),\nonumber \\
\omega_{t+1}^2(s_{t},a_t) &\leftarrow [1-\tau_t(s_t,a_t)]^2\omega_{t}^2(s_{t},a_t) + [\tau_{t}(s_t,a_t)]^2.\label{eq:proc_2}
\end{align}
If we set $\omega_0(s,a) \in [\varepsilon,1]$ $\forall s \in \mathcal{S}, a \in \mathcal{A}$ for some $0 < \varepsilon \leq 1$, then \rdd{it} follows $\omega_t(s,a) \in [\varepsilon, 1]$ $\forall t$ and $\forall s \in \mathcal{S}, a \in \mathcal{A}$ with probability 1 from Lemma \ref{lemma:n_eff_num}. Furthermore, from Lemma \ref{lemma:n_eff_den} \rdd{it} follows that $\lim_{t\rightarrow \infty} \omega_{t}^2(s,a) = 0$ $\forall s \in \mathcal{S}, a \in \mathcal{A}$ with probability 1 if we set, for example, $\omega_{0}^2(s,a) \in (0,1]$ $\forall s \in \mathcal{S}, a \in \mathcal{A}$. Combining these findings yields $\lim_{t\rightarrow \infty} n_{\text{eff},t}(s,a) = \infty$ $\forall s \in \mathcal{S}, a \in \mathcal{A}$ with probability 1, completing the proof.\qed}
\end{Sketch}

\gdef\thesection{Appendix \Alph{section}}
\section{Adaptive TE-BDQN Algorithm}\label{appendix:Ada_TE_BDQN}

\begin{algorithm}[H]
\begin{small}
\DontPrintSemicolon
\setstretch{1.10}
\SetAlgoLined
 \textbf{initialize} Action-value estimate networks with $K$ outputs $\left\{\hat{Q}^{*}_k \right\}^{K}_{k=1}$, masking distribution $M$, empty replay buffer $D$\\
\Repeat{}{
    Initialize $s$\\
    Pick a value function to act: $k \sim \text{Uniform}\{1,\ldots,K\}$\\
    \Repeat{$s$ is terminal}{
      Choose action $a$ from state $s$ with greedy policy derived from $\hat{Q}^{*}_k$\\
      Take action $a$, observe reward $r$ and next state $s'$\\
      \rd{Sample bootstrap masks $m = (m_1, \ldots, m_K)$}\\
      Add $(s, a, r, s', m)$ to replay buffer $D$\\
      Sample random minibatch of transitions $\left\{(s_i, a_i, s'_{i}, r_i, \rd{m^i})\right\}_{i=1}^{B}$ from $D$\\
      Perform gradient descent step based on (\ref{eq:BootDQN_KE_TE_grad})\\
      Every $C$ steps: \\
      \quad Reset $\rd{\theta_k^{-} = \theta_k}$ for $k = 1,\ldots, K$\\
      \quad Run partial episodes to update $\alpha$ via:
    $$
    \rd{\alpha \leftarrow \alpha + \frac{\tau_{\rm Ada}}{K} \sum_{k=1}^{K} \sum_{\Tilde{t} = 1}^{T_{\rm Ada}} \left[R_k(s_{\Tilde{t},k}, a_{\Tilde{t},k}) - \hat{Q}^{*}_k (s_{\Tilde{t},k}, a_{\Tilde{t},k}; \theta_k) \right]}$$\\
    $s \leftarrow s'$\\
    }
}
\caption{Ada-TE-BDQN}
\end{small}
\end{algorithm}

\setcounter{table}{0}
\gdef\thesection{Appendix \Alph{section}}
\section{Hyperparameters in MinAtar}\label{appendix:exp_hyperparams}
Table \ref{tab:hyperparams} details the settings for the experiments in MinAtar \citep{young2019minatar}. All algorithms were implemented using PyTorch \citep{paszke2019pytorch} and the computation was performed on Intel(R) Xeon(R) CPUs E5-2680 v3 (12 cores) @ 2.50GHz. The source code is available at: \url{https://github.com/MarWaltz/TUD_RL}. Note that we replaced some extreme outlier seeds for the bootstrap-based algorithms in Breakout and Seaquest environments. For example, the algorithm TE-BDQN led to an astonishing peak performance with a test return of over 200 in a Breakout run, while it got stuck in a rare occasion on Seaquest. Including those exceptions would paint an unrealistic picture of the actual capabilities of the algorithm. \rdd{A similar phenomenon} was observed for all bootstrap-based algorithms, and we argue that those rare instabilities are due to the algorithm's dependence on the initialization of the bootstrap heads.
\begin{table}[H]
    \def\arraystretch{1}
    \centering
    \begin{tabular}{l|l}
    Hyperparameter & Value\\
    \toprule
    
    Batch size ($B$) & 32\\
    Discount factor ($\gamma$) &  0.99 \\
    Loss function & MSE \\
    Min. replay buffer size & $5\,000$ \\
    Max. replay buffer size & $100\,000$ \\
    Optimizer & Adam \\
    Target network update frequency ($C$) & $1\,000$\\
    Initial exploration rate* ($\epsilon_{\rm initial}$) & 1.0 \\
    Final exploration rate* ($\epsilon_{\rm final}$) & 0.1 \\
    Test exploration rate* ($\epsilon_{\rm test}$) & 0.0 \\
    Exploration steps* & $100\,000$ \\
    Bernoulli mask probability\textsuperscript{\textdagger} ($p$) & 1.0 \\
    Number of bootstrap heads\textsuperscript{\textdagger} ($K$) & 10 \\
    Initial bias parameter\textsuperscript{\textdaggerdbl} ($\alpha$) & 0.25\\
    Time horizon\textsuperscript{\textdaggerdbl} ($T_{\rm Ada})$ & 32\\
    \end{tabular}
    \caption{List of hyperparameters used in the MinAtar experiments. Parameters with a * are used by DQN, DDQN, SCDQN, \rd{and MaxMin DQN,} while the ones with a \textsuperscript{\textdagger} are relevant for BDQN, TE-BDQN, KE-BDQN, and Ada-TE-BDQN. An \textsuperscript{\textdaggerdbl} exclusively refers to Ada-TE-BDQN.}
    \label{tab:hyperparams}
\end{table}

\setcounter{figure}{0}
\gdef\thesection{Appendix \Alph{section}}
\section{\rdd{Further results for MinAtar}}\label{appendix:Minatar_further}
\rdd{For completeness, we present the return and bias development during training in the MinAtar environments with learning rates $\tau \in \{10^{-5}, 10^{-4}\}$ in Figures \ref{fig:MinAtar_Asterix_LRcomp} - \ref{fig:MinAtar_SpaceInvaders_LRcomp}.}

\begin{figure}[htbp]
    \centering
    \includegraphics[width=\textwidth]{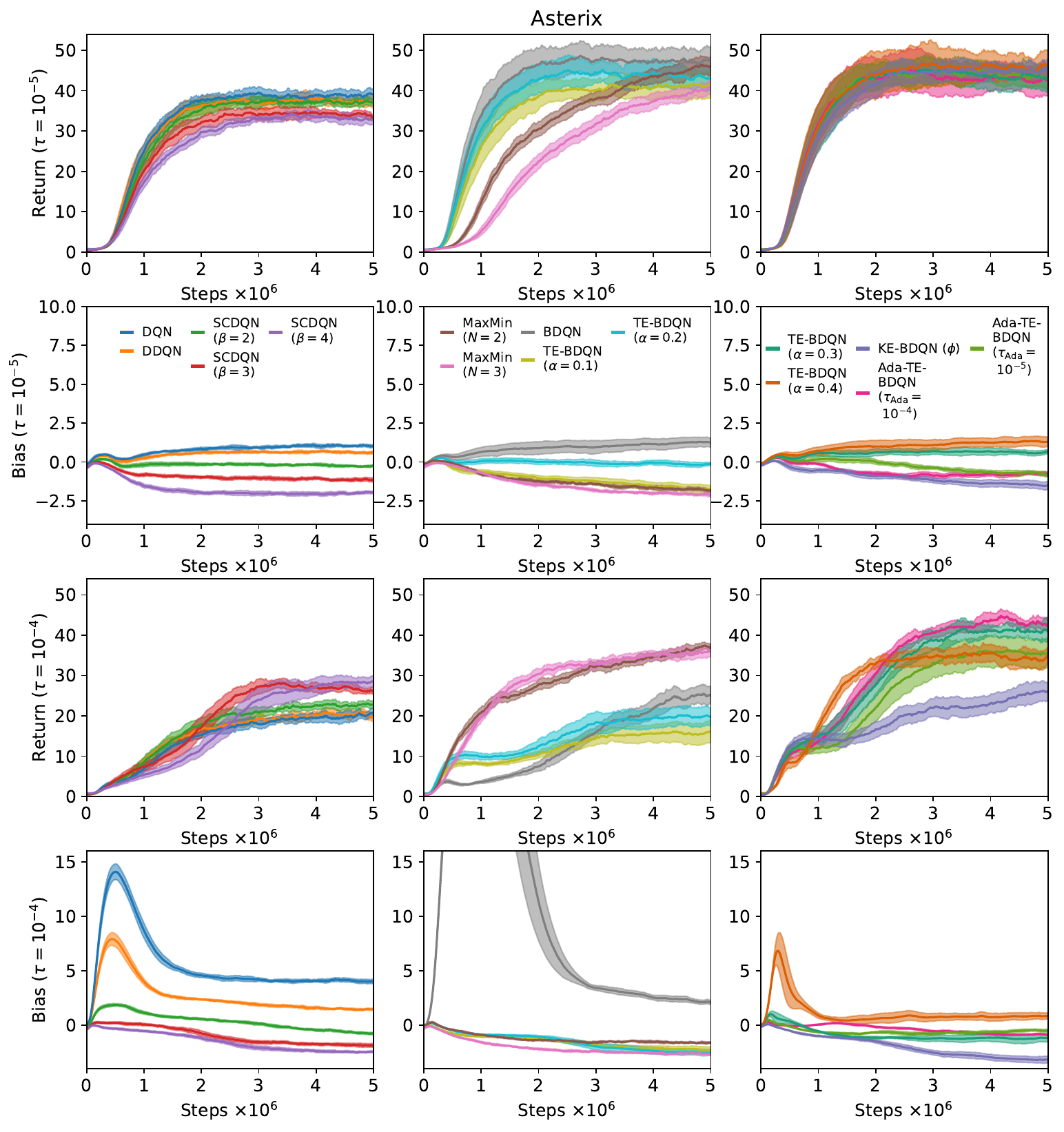}
    \caption{\rdd{Algorithm comparison on Asterix for learning rates $\tau \in \{10^{-5}, 10^{-4}\}$. The first two rows show the return and bias over time for $\tau = 10^{-5}$, while the results for $\tau=10^{-4}$ are displayed in rows three and four. Regarding algorithms, the left column includes the DQN, DDQN, and SCDQN; the middle column displays the MaxMin DQN, BDQN, and two TE-BDQNs; and the right column contains the remaining TE-BDQNs, the KE-BDQN, and the Ada-TE-BDQN results. The peak of the bias curve of the BDQN in row four of column two is at approximately 50, which we do not display to ensure the readability of the other curves.}}
    \label{fig:MinAtar_Asterix_LRcomp}
\end{figure}

\begin{figure}[htbp]
    \centering
    \includegraphics[width=\textwidth]{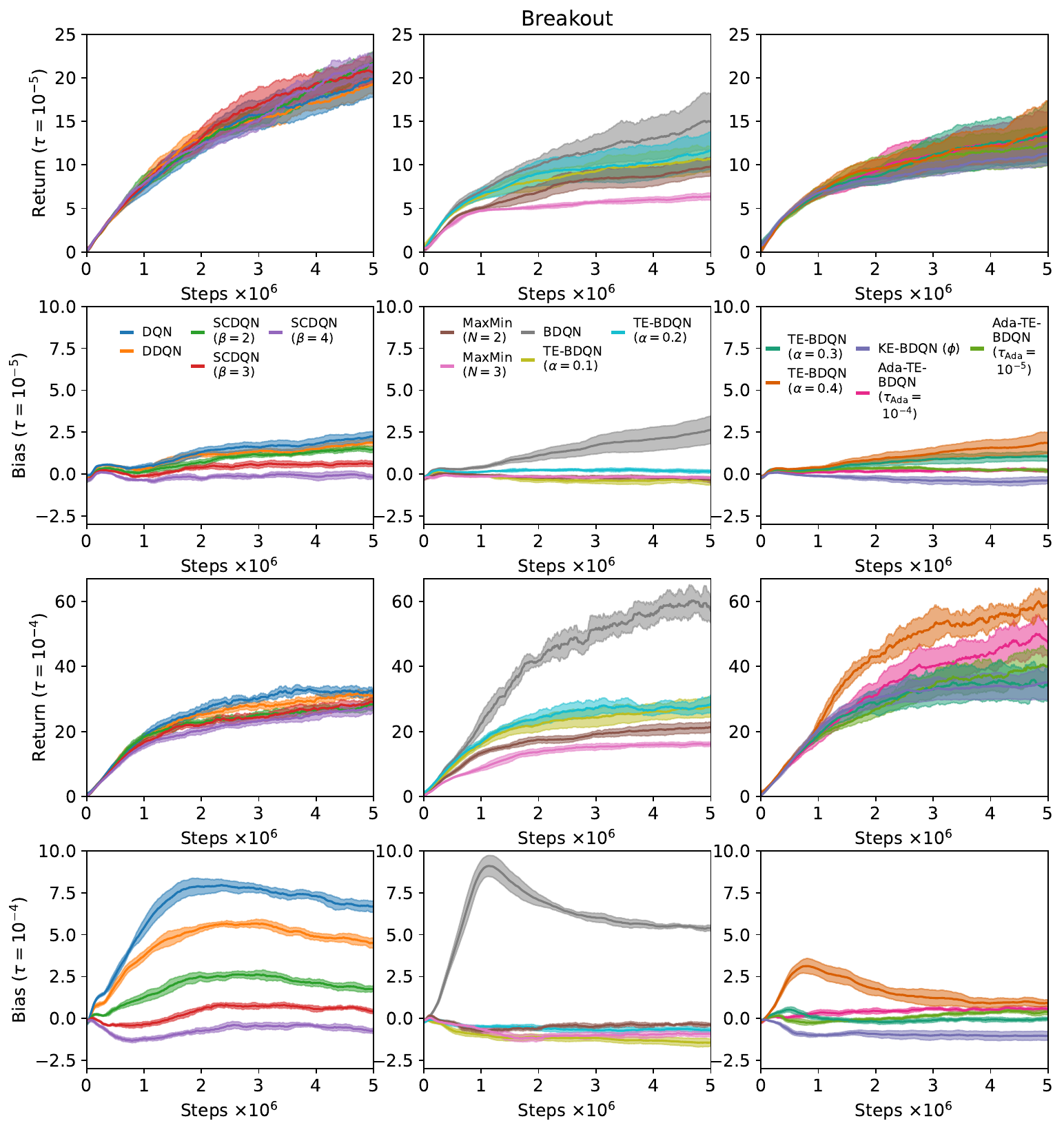}
    \caption{\rdd{Algorithm comparison on Breakout for learning rates $\tau \in \{10^{-5}, 10^{-4}\}$. The first two rows show the return and bias over time for $\tau = 10^{-5}$, while the results for $\tau=10^{-4}$ are displayed in rows three and four. Regarding algorithms, the left column includes the DQN, DDQN, and SCDQN; the middle column displays the MaxMin DQN, BDQN, and two TE-BDQNs; and the right column contains the remaining TE-BDQNs, the KE-BDQN, and the Ada-TE-BDQN results.}}
    \label{fig:MinAtar_Breakout_LRcomp}
\end{figure}

\begin{figure}[htbp]
    \centering
    \includegraphics[width=\textwidth]{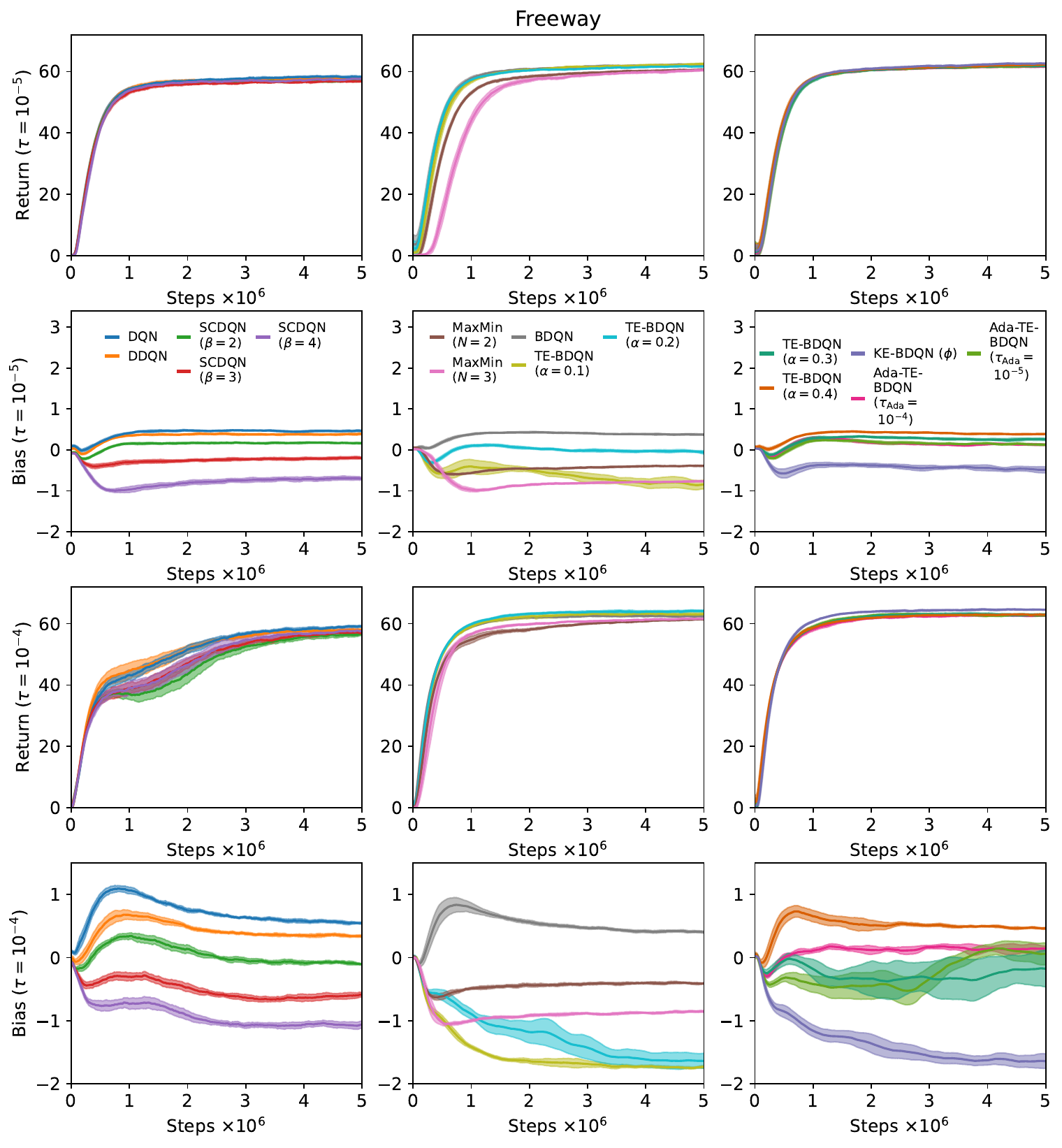}
    \caption{\rdd{Algorithm comparison on Freeway for learning rates $\tau \in \{10^{-5}, 10^{-4}\}$. The first two rows show the return and bias over time for $\tau = 10^{-5}$, while the results for $\tau=10^{-4}$ are displayed in rows three and four. Regarding algorithms, the left column includes the DQN, DDQN, and SCDQN; the middle column displays the MaxMin DQN, BDQN, and two TE-BDQNs; and the right column contains the remaining TE-BDQNs, the KE-BDQN, and the Ada-TE-BDQN results.}}
    \label{fig:MinAtar_Freeway_LRcomp}
\end{figure}

\begin{figure}[htbp]
    \centering
    \includegraphics[width=\textwidth]{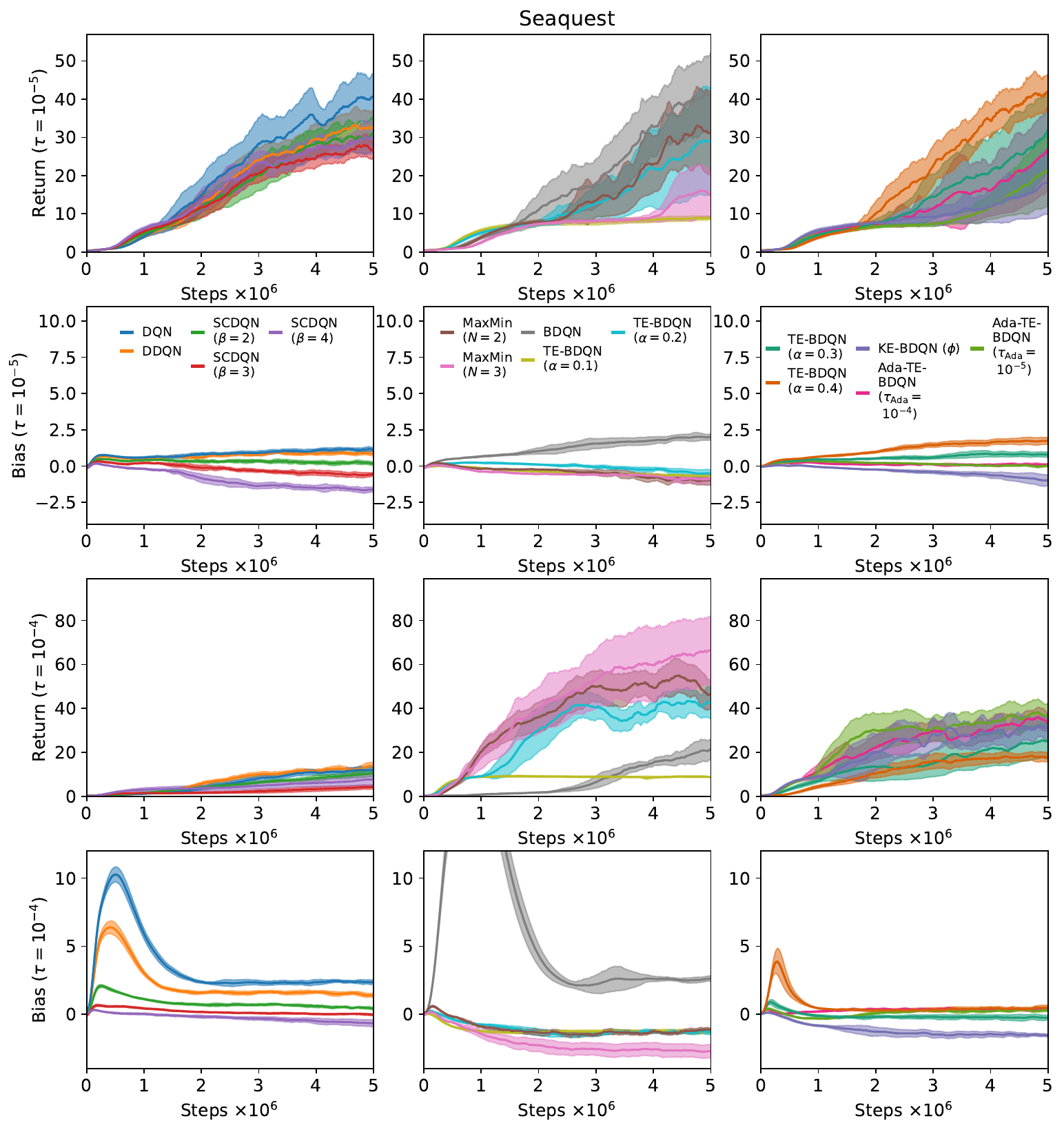}
    \caption{\rdd{Algorithm comparison on Seaquest for learning rates $\tau \in \{10^{-5}, 10^{-4}\}$. The first two rows show the return and bias over time for $\tau = 10^{-5}$, while the results for $\tau=10^{-4}$ are displayed in rows three and four. Regarding algorithms, the left column includes the DQN, DDQN, and SCDQN; the middle column displays the MaxMin DQN, BDQN, and two TE-BDQNs; and the right column contains the remaining TE-BDQNs, the KE-BDQN, and the Ada-TE-BDQN results. The peak of the bias curve of the BDQN in row four of column two is at approximately 20, which we do not display to ensure the readability of the other curves.}}
    \label{fig:MinAtar_Seaquest_LRcomp}
\end{figure}

\begin{figure}[htbp]
    \centering
    \includegraphics[width=\textwidth]{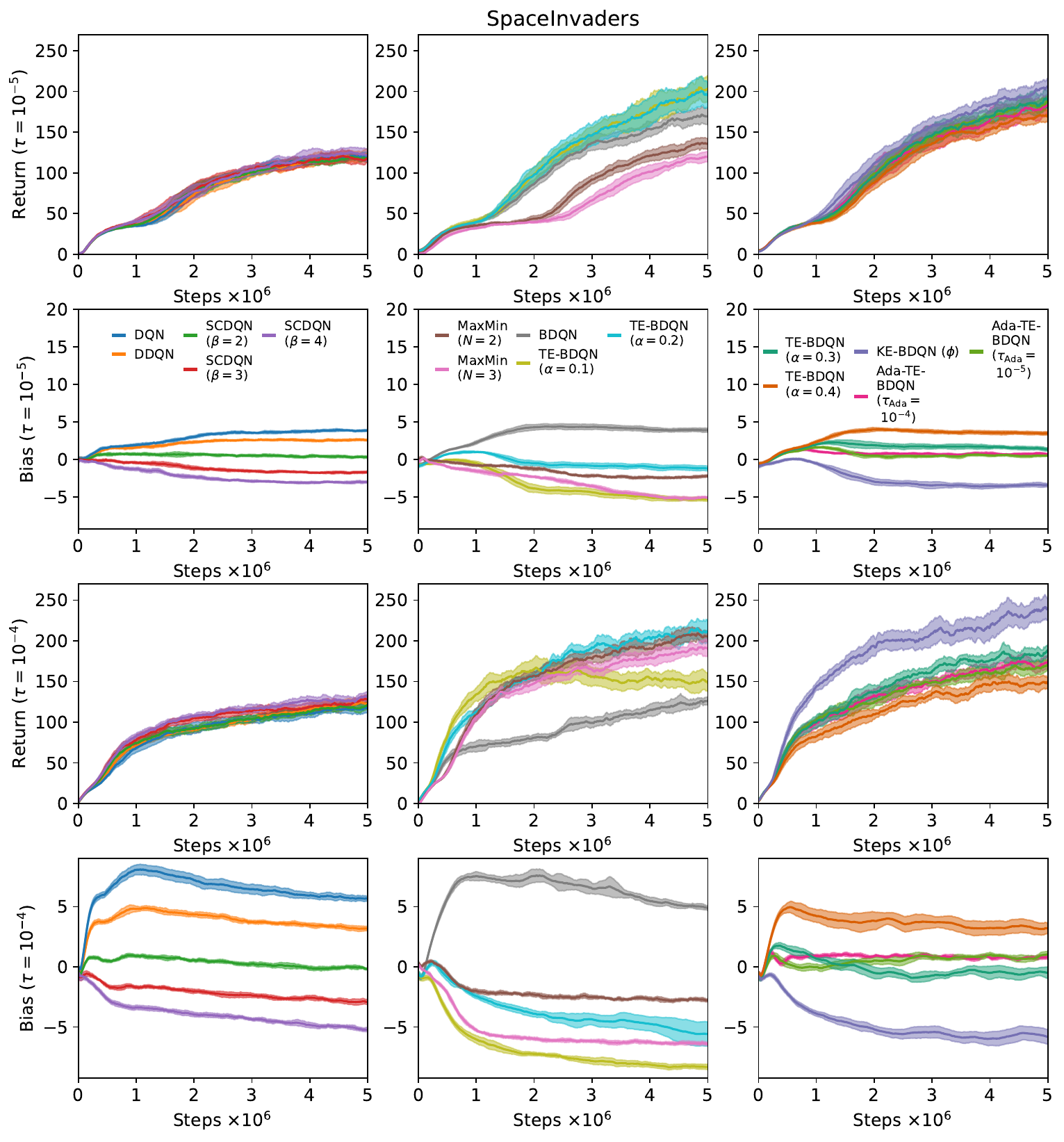}
    \caption{\rdd{Algorithm comparison on SpaceInvaders for learning rates $\tau \in \{10^{-5}, 10^{-4}\}$. The first two rows show the return and bias over time for $\tau = 10^{-5}$, while the results for $\tau=10^{-4}$ are displayed in rows three and four. Regarding algorithms, the left column includes the DQN, DDQN, and SCDQN; the middle column displays the MaxMin DQN, BDQN, and two TE-BDQNs; and the right column contains the remaining TE-BDQNs, the KE-BDQN, and the Ada-TE-BDQN results.}}
    \label{fig:MinAtar_SpaceInvaders_LRcomp}
\end{figure}






\end{document}